\documentclass{article}

\usepackage{microtype}
\usepackage{graphicx}
\usepackage{subfigure}
\usepackage{booktabs} 

\usepackage{hyperref}
\usepackage{algorithmic}
\usepackage[tbtags]{amsmath}
\usepackage{amssymb}
\usepackage{amsthm}
\usepackage{bbm}
\usepackage{xcolor}

\usepackage[OT1]{fontenc} 
\usepackage{xr}
\def\N{\mathcal{N}}
\def\reals{\mathbb{R}}
\newcommand{\W}[1]{ W^{(#1)}}
\def\E{\mathbb{E}}
\def\v{\text{vec}}

\newtheorem{lemma}{Lemma}
\newtheorem{corollary}{Corollary}
\newtheorem{theorem}{Theorem}

\usepackage[accepted]{icml2021}

\icmltitlerunning{A Modular Analysis of Provable Acceleration via Polyak's Momentum}

\begin{document}
\twocolumn[
\icmltitle{A Modular Analysis of Provable Acceleration via Polyak's Momentum: Training a Wide ReLU Network and a Deep Linear Network}




\begin{icmlauthorlist}
\icmlauthor{Jun-Kun Wang}{to}
\icmlauthor{Chi-Heng Lin}{goo}
\icmlauthor{Jacob Abernethy}{to}
\end{icmlauthorlist}

\icmlaffiliation{to}{School of Computer Science, Georgia Institute of Technology}
\icmlaffiliation{goo}{School of Electrical and Computer Engineering, Georgia Institute of Technology}

\icmlcorrespondingauthor{Jun-Kun Wang}{jimwang@gatech.edu}
\icmlcorrespondingauthor{Chi-Heng Lin}{cl3385@gatech.edu}
\icmlcorrespondingauthor{Jacob Abernethy}{prof@gatech.edu}

\icmlkeywords{Machine Learning, ICML}

\vskip 0.3in
]



\printAffiliationsAndNotice{}  

\begin{abstract}
Incorporating a so-called ``momentum'' dynamic in gradient descent methods is widely used in neural net training as it has been broadly observed that, at least empirically, it often leads to significantly faster convergence.
At the same time, there are very few theoretical guarantees
in the literature to explain this apparent acceleration effect.
Even for the classical strongly convex quadratic problems, several existing results only show Polyak's momentum has an accelerated linear rate asymptotically. 
In this paper, we first revisit the quadratic problems and show a non-asymptotic accelerated linear rate of Polyak's momentum. 
Then, we provably show that Polyak's momentum achieves acceleration for training a one-layer wide ReLU network and a deep linear network,
which are perhaps the two most popular canonical models for studying optimization and deep learning in the literature. 
Prior work \citep{DZPS19,WDW19} showed that using vanilla gradient descent, and with an use of over-parameterization, the error decays as $(1- \Theta(\frac{1}{ \kappa'}))^t$ after $t$ iterations, where $\kappa'$ is the condition number of a Gram Matrix.
Our result shows that with the appropriate choice of parameters Polyak's momentum has a rate of $(1-\Theta(\frac{1}{\sqrt{\kappa'}}))^t$.
For the deep linear network, prior work \citep{HXP20} showed that 
vanilla gradient descent has a rate of 
$(1-\Theta(\frac{1}{\kappa}))^t$, where $\kappa$ is the condition number of a data matrix.  
Our result shows an acceleration rate $(1- \Theta(\frac{1}{\sqrt{\kappa}}))^t$ is achievable by Polyak's momentum.
This work establishes that momentum does indeed speed up neural net training.
\end{abstract}

\section{Introduction}

Momentum methods are very popular for training neural networks in various applications (e.g. \citet{Rnet16,attention17,KSH12}).
It has been widely observed that the use of momentum helps faster training in deep learning (e.g. \citet{KH1918,WRSSR17,CO19}).
Among all the momentum methods, the most popular one seems to be Polyak's momentum (a.k.a. Heavy Ball momentum) \citep{P64}, which is the default choice of momentum in PyTorch and Tensorflow.
The success of Polyak's momentum in deep learning is widely appreciated
and almost all of the recently developed adaptive gradient methods like
Adam \cite{KB15}, AMSGrad \cite{RKK18}, and AdaBound \cite{LXLS19}    adopt the use of Polyak's momentum, instead of Nesterov's momentum. 

However, despite its popularity, little is known in theory about why Polyak's momentum helps to accelerate training neural networks.
Even for convex optimization, 
problems 
like strongly convex quadratic problems seem to be one of the few cases that discrete-time Polyak's momentum method provably achieves faster convergence than standard gradient descent (e.g. \citet{LRP16,goh2017why,GFJ15,GLZX19,LR17,LR18,CGZ19,SP20,NB15,WJR21,FSRV20,DJ19,SDJS18,H20}). 
On the other hand, the theoretical guarantees of
Adam, AMSGrad , or AdaBound     
are only worse if the momentum parameter $\beta$ is non-zero and 
the guarantees deteriorate as the momentum parameter increases,
which do not show any advantage of the use of momentum \citep{AMMC20}.
Moreover,
the convergence rates that have been established for Polyak's momentum in several
related works \citep{GPS16,SYLHGJ19,YLL18,LGY20,MJ20} do not improve upon those for vanilla gradient descent or vanilla SGD in the worst case.
\citet{LRP16,GFJ15} even show negative cases in \emph{convex} optimization 
that the use of Polyak's momentum results in divergence. 
Furthermore,
\citet{NKJK18} construct a problem instance for which the momentum method under its optimal tuning is outperformed by other algorithms. 
\citet{WCA20} show that Polyak's momentum helps escape saddle points faster compared with the case without momentum, which is the only provable advantage of Polyak's momentum in non-convex optimization that we are aware of.
A solid understanding of the empirical success of Polyak's momentum in deep learning has eluded researchers for some time.

We begin this paper by first revisiting the use of Polyak's momentum for 
the class of strongly convex quadratic problems,
\begin{equation} \label{obj:strc}
 \min_{w \in \reals^d} \frac{1}{2} w^\top \Gamma w + b^\top w,
\end{equation} 
where $\Gamma \in \reals^{d \times d}$ is a PSD matrix such that $\lambda_{\max}( \Gamma)= \alpha$, $\lambda_{\min}(\Gamma) =  \mu > 0$.
This is one of the few\footnote{In Section~\ref{sec:pre} and
	Appendix~\ref{app:history}, we will provide more discussions about this point.}
known examples that Polyak's momentum has a provable \textit{globally} \textit{accelerated} linear rate in the \textit{discrete-time} setting. 
Yet even for this class of problems existing results only establish an accelerated linear rate in an asymptotic sense and several of them do not have an explicit rate in the non-asymptotic regime (e.g. \citet{P64,LRP16,M19,R18}).
Is it possible to prove a non-asymptotic accelerated linear rate in this case? We will return to this question soon. 

For general $\mu$-strongly convex, $\alpha$-smooth, and twice differentiable
functions (not necessarily quadratic), denoted as $F_{\mu,\alpha}^2$,
Theorem 9 in \citet{P64} shows an asymptotic accelerated linear rate
when the iterate is \textit{sufficiently} close to the minimizer so that the landscape can be well approximated by that of a quadratic function. However, 
the definition of the neighborhood was not very precise in the paper. In this work, we show a locally accelerated linear rate under a quantifiable definition of the neighborhood.

\begin{algorithm}[t]
\begin{algorithmic}[1]
\caption{Gradient descent with Polyak's momentum \citep{P64} (Equivalent Version 1)
} \label{alg:HB1}{}
\STATE Required: Step size $\eta$ and momentum parameter $\beta$.
\STATE Init: $w_{0} \in \reals^d $ and $M_{-1} = 0 \in \reals^d$.
\FOR{$t=0$ to $T$}
\STATE Given current iterate $w_t$, obtain gradient $\nabla \ell(w_t)$.
\STATE Update momentum $M_t = \beta M_{t-1} +  \nabla \ell(w_t)$.
\STATE Update iterate $w_{t+1} = w_t - \eta M_t$.
\ENDFOR
\end{algorithmic}
\end{algorithm}

\begin{algorithm}[t]
\begin{algorithmic}[1]
\caption{Gradient descent with Polyak's momentum \citep{P64} (Equivalent Version 2) } 
\label{alg:HB2}
\STATE Required: step size $\eta$ and momentum parameter $\beta$.
\STATE Init: $w_{0} = w_{-1} \in \reals^d $
\FOR{$t=0$ to $T$}
\STATE Given current iterate $w_t$, obtain gradient $\nabla \ell(w_t)$.
\STATE Update iterate $w_{t+1} = w_t - \eta \nabla \ell(w_t) + \beta ( w_t - w_{t-1} )$.
\ENDFOR
\end{algorithmic}
\end{algorithm}

Furthermore, we provably show that Polyak's momentum helps to achieve a faster convergence for training two neural networks, compared to vanilla GD.
The first is training a one-layer ReLU network.
Over the past few years there have appeared an enormous number of works considering training a one-layer ReLU network, provably showing convergence results for vanilla (stochastic) gradient descent 
(e.g. \citet{LL18,JT20,LY17,DZPS19,DLLWZ16,ZL19_icml,ZY19,ZCZG19,ADHLSW19_icml,JGH18,LXSBSP19,COB19,OS19,BG17,CHS20,
T17,S17,BL20,BSMM19,LMZ20,HN20,Dan17,ZG19,DGM20,D20,WLLM19,YS20,FDZ19,SY19,CCGZ20}),
as well as for other algorithms (e.g. \citet{ZMG19,WDS19,Cetal19,ZSJBD17,GKLW16,BPSW20,LSSWY20,PE20}).
However, we are not aware of any theoretical works that study the momentum method in neural net training except the work \citet{KCH20}.
These authors show that SGD with Polyak's momentum (a.k.a. stochastic Heavy Ball) with infinitesimal step size, i.e. $\eta \rightarrow 0$, for training a one-hidden-layer network with an infinite number of neurons, i.e. $m \rightarrow \infty$, converges to a stationary solution.
However, the theoretical result does not show a faster convergence by momentum. 
In this paper we consider the discrete-time setting
and
nets with finitely many neurons. 
We 
provide a non-asymptotic convergence rate of Polyak's momentum, establishing a concrete improvement relative to the best-known rates for vanilla gradient descent.

Our setting
of training a ReLU network follows the same framework as previous results, including \citet{DZPS19,ADHLSW19_icml,ZY19}. 
Specifically,
we study training a one-hidden-layer ReLU neural net of the form, 
\begin{equation} \label{eq:Network}
\N_{W}^{\text{ReLU}}(x) := \frac{1}{\sqrt{m} } \sum_{r=1}^m a_r \sigma( \langle w^{(r)},  x \rangle ),
\end{equation}
where $\sigma(z):= z \cdot \mathbbm{1}\{ z \geq 0\}$ is the ReLU activation,
$w^{(1)}, \dots, w^{(m)}  \in \reals^d$ are the weights of $m$ neurons on the first layer, $a_1, \dots, a_m \in \reals$ are weights on the second layer, 
 and $\N_{W}^{\text{ReLU}}(x) \in \reals$ is the output predicted on input $x$. 
Assume $n$ number of samples $\{ x_i \in \reals^d \}_{i=1}^n$ is given.
Following \citet{DZPS19,ADHLSW19_icml,ZY19},
we define a Gram matrix $H \in \reals^{n \times n}$ for the weights $W$
and its expectation $\bar{H} \in \reals^{n \times n}$ over the random draws of $w^{(r)} \sim N(0,I_d) \in \reals^d$ 
whose $(i,j)$ entries are defined
as follows,
\begin{equation}
\begin{aligned}
& H(W)_{i,j}   =  \sum_{r=1}^m \frac{x_i^\top x_j}{m} \mathbbm{1}\{ \langle w^{(r)}, x_i \rangle \geq 0 \text{ } \&  \text{ }   \langle w^{(r)}, x_j \rangle \geq 0 \}
\\ & \quad
\bar{H}_{i,j}  := \underset{ w^{(r)}}{\mathsf{E}}
[ x_i^\top x_j \mathbbm{1}\{ \langle w^{(r)}, x_i \rangle \geq 0 \text{ } \&  \text{ }   \langle w^{(r)}, x_j \rangle \geq 0 \}     ] .
\end{aligned}
\end{equation}
The matrix $\bar{H}$ is also called a neural tangent kernel (NTK) matrix in the literature (e.g. \citet{JGH18,Y19,BM19}).
Assume that the smallest eigenvalue $\lambda_{\min}(\bar{H})$ is strictly positive and certain conditions about the step size and the number of neurons are satisfied.
Previous works \citep{DZPS19,ZY19} show a linear rate of vanilla gradient descent, while we show an accelerated linear rate 
\footnote{ 
We borrow the term ``accelerated linear rate'' from the convex optimization literature \citep{N13}, because the result here has a resemblance to those results in convex optimization, even though the neural network training is a non-convex problem.
}
of gradient descent with Polyak's momentum. As far as we are aware, our result is the first acceleration result of training an over-parametrized ReLU network. 

The second result is training a deep linear network.
The deep linear network is a canonical model for studying optimization and deep learning, and in particular for understanding gradient descent (e.g. \citet{BHL18,SMG14,HXP20}), studying the optimization landscape (e.g. \citet{K16,LvB18}), and establishing the effect of implicit regularization (e.g. \citet{MGWLSS20,JT19,LMZ18,RC20,ACHL19,GBL19,GWBNS17,LL20}).
In this paper, following \cite{DH19,HXP20}, we study training a $L$-layer linear network of the form,
\begin{equation} \label{eq:NetworkLinear}
\textstyle
\N_W^{L\text{-linear}}(x) := \frac{1}{\sqrt{m^{L-1} d_{y}}} \W{L} \W{L-1} \cdots \W{1} x,
\end{equation}
where $\W{l} \in \reals^{d_l \times d_{l-1}}$ is the weight matrix of the layer $l \in [L]$, and $d_0 = d$, $d_L = d_y$ and $d_l = m$ for $l \neq 1, L$.
Therefore, except the first layer $\W{1} \in \reals^{ m \times d}$ and the last layer $\W{L} \in \reals^{ d_y \times m}$, all the intermediate layers are $m \times m$ square matrices. 
The scaling $\frac{1}{\sqrt{m^{L-1} d_{y}}} $ is necessary to ensure that the network's output at the initialization 
$\N_{W_0}^{L\text{-linear}}(x)$ has the same size as that of the input $x$, in the sense that $\E[\| \N_{W_0}^{L\text{-linear}}(x) \|^2 ] = \| x \|^2$, where the expectation is taken over some appropriate random initialization of the network (see e.g. \citet{DH19,HXP20}).
\citet{HXP20} show vanilla gradient descent with orthogonal initialization converges linearly and the required width of the network $m$ is independent of the depth $L$, while we show an accelerated linear rate of Polyak's momentum and the width $m$ is also independent of $L$. To our knowledge, this is the first acceleration result 
of training a deep linear network.

A careful reader may be tempted by the following line of reasoning: a deep linear network (without activation) is effectively a simple linear model, and we already know that a linear model with the squared loss gives a quadratic objective for which Polyak's momentum exhibits an accelerated convergence rate. 
But this intuition, while natural, is not quite right: it is indeed nontrivial even to show that vanilla gradient descent provides a linear rate on deep linear networks \cite{HXP20,DH19,BHL18,ACGH19,HM16,WWM19,ZLG20}, as the optimization landscape is non-convex.
Existing works show that under certain assumptions, all the local minimum are global \cite{K16,LvB18,YSJ17,LK17,ZL18,HM16}. These results are not sufficient to explain the linear convergence of momentum, let alone the acceleration; see Section~\ref{app:exp} in the appendix for an empirical result.

Similarly, it is known that under the NTK regime the output of the ReLU network trained by gradient descent can be approximated by a linear model (e.g. \citet{HXAP20}). However, this result alone neither implies a global convergence of any algorithm nor characterizes the optimization landscape. 
While \cite{LZB20a} attempt to derive an algorithm-independent equivalence of a class of linear models and a family of wide networks, their result requires the activation function to be differentiable which does not hold for the most prevalent networks like ReLU. Also, their work heavily depends on the regularity of Hessian, making it hard to generalize beyond differentiable networks. Hence, while there has been some progress understanding training of wide networks through linear models, there remains a significant gap in applying this to the momentum dynamics of a non-differentiable networks.
\citet{LZB20b} establish an interesting connection between solving an over-parametrized non-linear system of equations and solving the classical linear system. They show that for smooth and twice differentiable activation, 
the optimization landscape of an over-parametrized network satisfies a (non-convex) notion called the Polyak-Lokasiewicz (PL) condition \citep{P63}, i.e. $
\frac{1}{2} \| \nabla \ell(w) \|^2 \geq \mu \left( \ell(w) - \ell(w_*) \right)$, where $w_*$ is a global minimizer and $\mu > 0$.
It is not clear whether their result can be extended to ReLU activation, however, and the 
existing result of \citet{DKB18} for the discrete-time Polyak's momentum under the PL condition does not give an accelerated rate nor is it better than that of vanilla GD.
\citet{ADR20} show a \emph{variant} of Polyak's momentum method having an accelerated rate in a \emph{continuous-time} limit for a problem that satisfies PL and has a unique global minimizer. 
It is unclear if their result is applicable to our problem. 
Therefore, showing the advantage of training the ReLU network and the deep linear network by using existing results of Polyak's momentum can be difficult.

To summarize, our contributions 
in the present work include
\begin{itemize}
\item 
In convex optimization, we show
an accelerated linear rate in the non-asymptotic sense for solving the class of the strongly convex quadratic problems via Polyak's momentum
(Theorem~\ref{thm:stcFull}).
We also provide an analysis of the accelerated local convergence  
for the class of functions in $F_{\mu,\alpha}^2$ (Theorem~\ref{thm:STC}). 
We establish a technical result (Theorem~\ref{thm:akv}) that helps to obtain these non-asymptotic rates.
\item 
In non-convex optimization, we show 
accelerated linear rates of the discrete-time Polyak's momentum for training an over-parametrized ReLU network and a deep linear network (Theorems~\ref{thm:acc} and~\ref{thm:LinearNet}).
\end{itemize}
Furthermore, we will develop a modular analysis to show all the results in this paper. We identify conditions and propose a meta theorem of acceleration when the momentum method exhibits a certain dynamic, which can be of independent interest.
We show that when applying Polyak's momentum for these problems, the induced dynamics exhibit a form where we can directly apply our meta theorem.

\section{Preliminaries} \label{sec:pre}

Throughout this paper, $\| \cdot \|_F$ represents the Frobenius norm and $\| \cdot \|_2$ represents the spectral norm of a matrix, while $\| \cdot \|$ represents $l_2$ norm of a vector. We also denote $\otimes$ the Kronecker product, $\sigma_{\max}(\cdot)=\| \cdot \|_2$ and $\sigma_{\min}(\cdot)$ the largest and the smallest singular value of a matrix respectively.

For the case of training neural networks, we will consider minimizing the squared loss 
\begin{equation}  \label{eq:obj}
\textstyle 
\ell(W):= \frac{1}{2} \sum_{i=1}^n \big(  y_i -   \N_{W}(x_i)    \big)^2,
\end{equation}
where $x_i \in \reals^{d}$ is the feature vector, $y_i \in \reals^{d_y}$
is the label of sample $i$, and there are $n$ number of samples.
For training the ReLU network, we have
$\N_{W}(\cdot) := \N_W^{\text{ReLU}}(\cdot)$, $d_y = 1$,
and $W:= \{ w^{(r)}  \}_{r=1}^m$,
while for the deep linear network, we have 
$\N_{W}(\cdot) := \N_W^{L\text{-linear}}(\cdot)$,
and $W$ represents the set of all the weight matrices, i.e. $W:= \{ W^{(l)}  \}_{l=1}^L$.
The notation $A^k$ represents the $k_{th}$ matrix power of $A$.

\subsection{Prior result of Polyak's momentum}

Algorithm~\ref{alg:HB1} and Algorithm~\ref{alg:HB2} show two equivalent presentations of gradient descent with Polyak's momentum. Given the same initialization, one can show that 
Algorithm~\ref{alg:HB1} and Algorithm~\ref{alg:HB2} generate exactly the same iterates during optimization. 

Let us briefly describe a prior acceleration result of Polyak's momentum. 
The recursive dynamics of Poylak's momentum for solving the
strongly convex quadratic problems (\ref{obj:strc})
can be written as
\begin{equation} \label{eq:A}
\textstyle
\begin{bmatrix}
w_{t+1} - w_* \\
w_{t} - w_*
\end{bmatrix}
=
\underbrace{
\begin{bmatrix}
I_d - \eta \Gamma + \beta I_d &  - \beta I_d \\
I_d                      &  0_d
\end{bmatrix} }_{:=A}
\cdot
 \begin{bmatrix}
w_{t} - w_* \\
w_{t-1} - w_* 
\end{bmatrix},
\end{equation}
where $w_*$ is the unique minimizer.
By a recursive expansion, one can get
\begin{equation} \label{eq:B}
\textstyle
\|
\begin{bmatrix}
w_{t} - w_* \\
w_{t-1} - w_*
\end{bmatrix}
\| \leq \| A^{t} \|_2 \| 
\begin{bmatrix}
w_{0} - w_* \\
w_{-1} - w_*
\end{bmatrix}
\|.
\end{equation}
Hence, it suffices to control the spectral norm of the matrix power
$\| A^{t} \|_2$ for obtaining a convergence rate. In the literature, this is achieved by using 
Gelfand's formula.
\begin{theorem} (\citet{G41}; see also \citet{F18}) (Gelfand's formula) \label{thm:Gelfand}
Let $A$ be a $d \times d$ matrix. Define the spectral radius $\rho(A) := \max_{i \in [d]} | \lambda_i(A)|$, where $\lambda_i(\cdot)$ is the $i_{th}$ eigenvalue.
Then, there exists a non-negative sequence $\{ \epsilon_t \}$ such that 
$\| A^t \|_2 = \left( \rho(A) + \epsilon_t \right)^t $
and
$\lim_{t \rightarrow \infty} \epsilon_t = 0$.
\end{theorem}
We remark that there is a lack of the convergence rate of $\epsilon_t$ in
Gelfand's formula in general.

Denote $\kappa:= \alpha / \mu$ the condition number.
One can control the spectral radius $\rho(A)$
as $\rho(A) \leq  1 - \frac{2}{\sqrt{\kappa}+1}$ by choosing $\eta$ and $\beta$ appropriately,  which leads to the following result.
\begin{theorem} (\citet{P64}; see also \citet{LRP16,R18,M19}) \label{thm:polyak}
Gradient descent with Polyak's momentum with the step size $\eta = \frac{4}{(\sqrt{\mu}+\sqrt{\alpha})^2}$ and the momentum parameter $\beta = \left( 1 - \frac{2}{\sqrt{\kappa}+1} \right)^2$ has
\[
\|
\begin{bmatrix}
w_{t+1} - w_* \\
w_{t} - w_*
\end{bmatrix}
\| \leq \left(  1 - \frac{2}{\sqrt{\kappa}+1}  + \epsilon_t \right)^{t+1}
\begin{bmatrix}
w_{0} - w_* \\
w_{-1} - w_*
\end{bmatrix}
\|,
\]
where $\epsilon_t$ is a non-negative sequence that goes to zero.
\end{theorem}
That is, when $t \rightarrow \infty$, Polyak's momentum has the
 $(  1 - \frac{2}{\sqrt{\kappa}+1})$ rate, which has a better dependency on the condition number $\kappa$ than the
$1-\Theta(\frac{1}{\kappa})$ rate of vanilla gradient descent.
A concern is that the bound is not quantifiable for a finite $t$.
On the other hand,
we are aware of a different analysis 
that leverages Chebyshev polynomials
instead of Gelfand's formula (e.g. \citet{LB18}), which manages to obtain a $t (1- \Theta(\frac{1}{\sqrt{\kappa}}) )^t$ convergence rate.
So the accelerated linear rate is still obtained in an asymptotic sense.
Theorem~9 in \citet{CGZ19} shows a rate
$\max \{\bar{C}_1,  t \bar{C}_2 \} (1- \Theta(\frac{1}{\sqrt{\kappa}})^t )$ 
for some constants $\bar{C}_1$ and $\bar{C}_2$ under the same choice of the momentum parameter and the step size as Theorem~\ref{thm:polyak}. 
However, for a large $t$, the dominant term could be $ t (1- \Theta(\frac{1}{\sqrt{\kappa}})^t)$. 
In this paper, we aim at obtaining a bound that (I) holds for a wide range of values of the parameters, (II) has a dependency on the squared root of the condition number $\sqrt{\kappa}$, (III) is quantifiable in each iteration and is better than the rate $t (1- \Theta(\frac{1}{\sqrt{\kappa}}) )^t $.

\subsection{(One-layer ReLU network) Settings and Assumptions}
The ReLU activation is not differentiable at zero. So 
for solving (\ref{eq:obj}), 
we will replace the notion of gradient in Algorithm~\ref{alg:HB1} and \ref{alg:HB2} with subgradient
$\frac{ \partial \ell(W_t)}{ \partial w_t^{(r)} }  := \frac{1}{\sqrt{m}} \sum_{i=1}^n \big( \N_{W_t}(x_i) - y_i \big) a_r \cdot \mathbbm{1}[ \langle w_t^{(r)}, x_i \rangle \geq 0]  x_i $ 
and update the neuron $r$ as 
$
w_{t+1}^{(r)} = w_t^{(r)} - \eta \frac{ \partial \ell(W_t)}{ \partial w_t^{(r)} } + \beta \big(  w_t^{(r)} - w_{t-1}^{(r)}  \big).
$
\begin{figure}[t]
  \centering
    \includegraphics[width=0.4\textwidth]{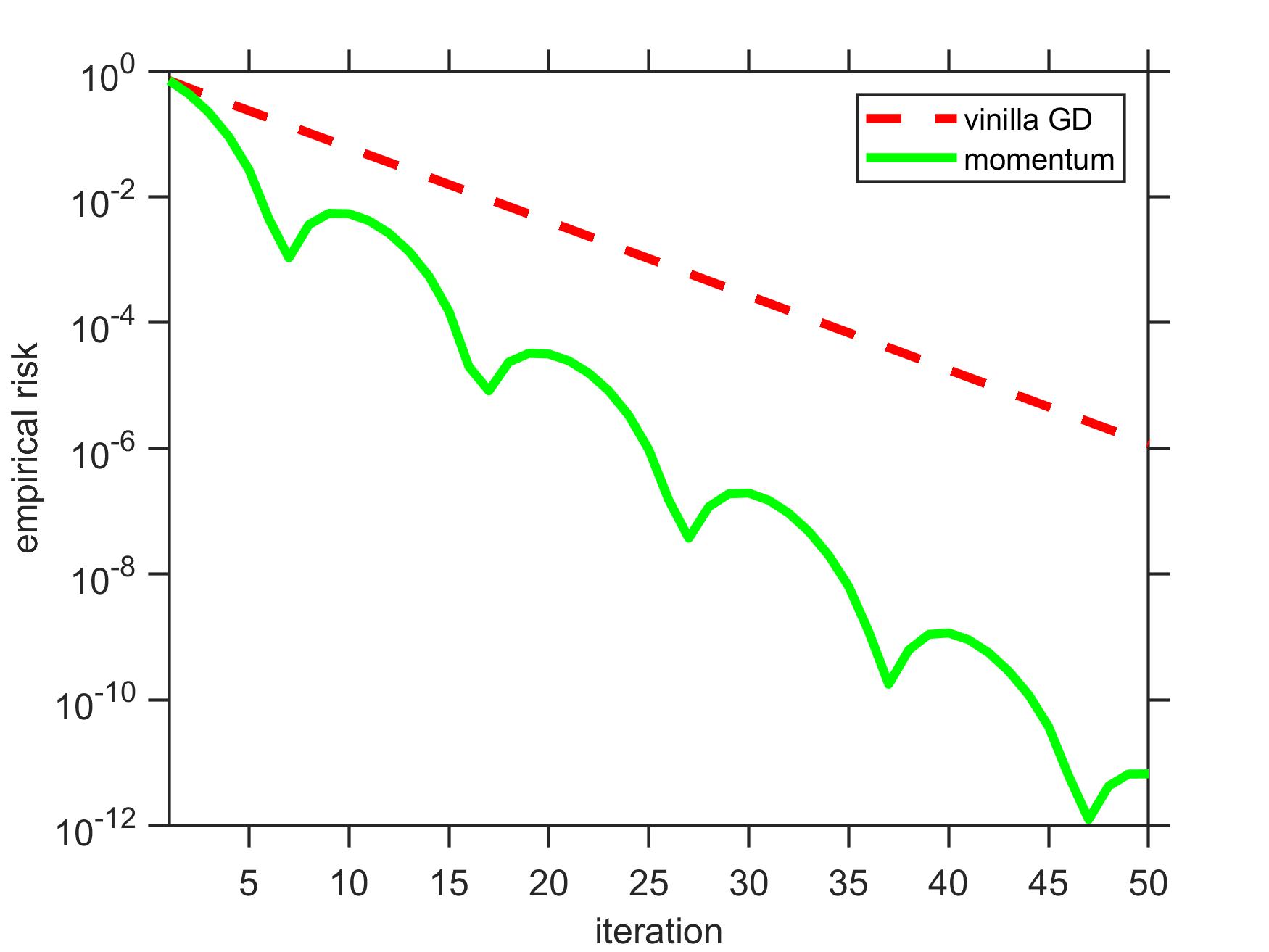}
    \caption{\footnotesize Empirical risk $\ell(W_t)$ vs. iteration $t$. Polyak's momentum accelerates the optimization process of training an over-parametrized one-layer ReLU network. Experimental details are available in Appendix~\ref{app:exp}. }
        \label{fig:exp} 
\end{figure}
As described in the introduction,
we assume that the smallest eigenvalue of the Gram matrix $\bar{H} \in \reals^{n \times n }$ is strictly positive, i.e. $\lambda_{\min}( \bar{H} ) > 0 $.
We will also denote the largest eigenvalue of the Gram matrix $\bar{H}$ as
$\lambda_{\max}( \bar{H} )$ and
denote the condition number of the Gram matrix as $\kappa := \frac{\lambda_{\max}(\bar{H}) }{ \lambda_{\min}(\bar{H}) }$.
\citet{DZPS19} show that the strict positiveness assumption is indeed mild.
Specifically, they show that if no two inputs are parallel, then the least eigenvalue is strictly positive. 
\citet{PSG2020} were able to provide a quantitative lower bound
under certain conditions.
Following the same framework of \citet{DZPS19},
we consider that each weight vector $w^{(r)} \in \reals^d$ is initialized according to the normal distribution, i.e. $w^{(r)} \sim N(0,I_d)$, 
and each $a_r \in R$ is sampled from the Rademacher distribution, 
i.e. $a_r =1$ with probability 0.5; and $a_r=-1$ with probability $0.5$.
We also assume $\| x_i \| \leq 1$ for all samples $i$.
As the previous works (e.g. \citet{LL18,JT20,DZPS19}), we consider only training the first layer $\{w^{(r)} \}$
and the second layer $\{a_r\}$ is fixed throughout the iterations. We will
denote $u_t \in \reals^n$ whose $i_{th}$ entry is the network's prediction for sample $i$, i.e. $u_t[i] := \N_{W_t}^{\text{ReLU}}(x_i)$ in iteration $t$
and denote $y \in \reals^n$ the vector whose $i_{th}$ element is the label of sample $i$. 
The following theorem
is a prior result due to \citet{DZPS19}.

\begin{theorem} (Theorem 4.1 in \citet{DZPS19}) \label{thm:du}
Assume that $\lambda := \lambda_{\min}( \bar{H}) / 2 > 0$ and that $w_0^{(r)} \sim N(0,I_d)$ and $a_r$ uniformly sampled from $\{-1,1\}$. 
Set the number of nodes $m = \Omega( \lambda^{-4} n^6 \delta^{-3})$ and the constant step size $\eta = O(\frac{\lambda}{n^2})$.
Then, with probability at least $1-\delta$ over the random initialization,
vanilla gradient descent, i.e. Algorithm~\ref{alg:HB1}\&~\ref{alg:HB2} with $\beta=0$, has
$\|
u_t - y 
\|^2 \leq \left( 1 - \eta  \lambda \right)^{t}
\cdot
\|
u_0 - y
\|^2.
$
\end{theorem}
Later \citet{ZY19} improve the network size $m$ to
$m =\Omega( \lambda^{-4} n^4 \log^3 ( n / \delta))$.
\citet{WDW19} provide an improved analysis over \citet{DZPS19}, which shows that the step size $\eta$ of vanilla gradient descent can be set as $\eta =  \frac{1}{c_1 \lambda_{\max}(\bar{H})}$ for some quantity $c_1>0$. The result in turn leads to a convergence rate 
$( 1 - \frac{1}{c_2 \kappa}  )$ for some quantity $c_2>0$. 
However, the quantities $c_1$ and $c_2$ are not universal constants and actually depend on the problem parameters $\lambda_{\min}(\bar{H})$, $n$, and $\delta$.
A question that we will answer in this paper is ``\textit{Can Polyak's momentum achieve an accelerated linear rate $\left( 1- \Theta(\frac{1}{\sqrt{\kappa}}) \right)$, where the factor $\Theta(\frac{1}{\sqrt{\kappa}})$ does not depend on any other problem parameter?}''. 

\subsection{(Deep Linear network) Settings and Assumptions}

For the case of deep linear networks, we will denote  
$X := [x_1,\dots, x_n] \in \reals^{d \times n}$
the data matrix
 and $Y:= [y_1, \dots, y_n] \in \reals^{d_y \times n}$  
 the corresponding label matrix. 
We will also
denote $\bar{r} := rank(X)$ and the condition number $\kappa := \frac{\lambda_{\max}( X^\top X)}{ \lambda_{\bar{r}}(X^\top X)}$.
Following \citet{HXP20},
we will assume that the linear network
is initialized by
the orthogonal initialization,
which is conducted by sampling uniformly from (scaled) orthogonal matrices such that
 $(\W{1}_0)^\top \W{1}_0 = m I_{d}$,
 $\W{L}_0 (\W{L}_0)^\top  = m I_{d_y}$, and $(\W{l}_0)^\top \W{l}_0= \W{l}_0 (\W{l}_0)^\top  = m I_{m}$
for layer $2 \leq l \leq L-1$. 
We will denote $\W{j:i} := W_j W_{j-1} \cdots W_i = \Pi_{l=i}^j W_l$, where $1 \leq i \leq j \leq L$ and $\W{i-1:i} = I$. We also denote the network's output 
$\textstyle
U := \frac{1}{\sqrt{m^{L-1} d_y} } \W{L:1}X \in \reals^{d_y \times n}.
$

In our analysis, following \citet{DH19,HXP20},
we will further assume that (A1) there exists a $W^*$ such that $Y = W^* X$, $X \in \reals^{d \times \bar{r}}$, and $\bar{r}=rank(X)$, which is actually without loss of generality (see e.g. the discussion in Appendix B of \citet{DH19}).
\begin{theorem} (Theorem 4.1 in \citet{HXP20}) \label{thm:hu}
Assume (A1) and the use of the orthogonal initialization.
Suppose the width of the deep linear network satisfies $m 
\geq C \frac{\|X \|^2_F}{\sigma^2_{\max}(X)} \kappa^2 \left( d_y (1 + \| W_* \|^2_2) + \log (\bar{r} /\delta)  \right)$ and $m \geq \max \{d_x,d_y \}$ 
for some $\delta \in (0,1)$ and a sufficiently large constant $C> 0$.
Set the constant step size $\eta = \frac{d_y}{2 L \sigma^2_{\max}(X) }$.
Then, with probability at least $1-\delta$ over the random initialization,
vanilla gradient descent, i.e. Algorithm~\ref{alg:HB1}\&~\ref{alg:HB2} with $\beta=0$, has
$\|
U_t - Y 
\|^2_F \leq \left( 1 - \Theta(\frac{1}{\kappa}) \right)^{t}
\cdot
\|
U_0 - Y
\|^2_F.$
\end{theorem}

\section{Modular Analysis} \label{sec:meta}

In this section,
we will provide a meta theorem for the following dynamics of the residual vector $\xi_t \in \reals^{n_0} $,
\begin{equation} \label{eq:meta}
\begin{split}
\begin{bmatrix}
\xi_{t+1} \\
\xi_{t} 
\end{bmatrix}
 &
=
\begin{bmatrix}
I_{n_0} - \eta H + \beta I_{n_0} & - \beta  I_{n_0}   \\
I_{n_0} & 0_{n_0} 
\end{bmatrix}
\begin{bmatrix}
\xi_{t} \\
\xi_{t-1} 
\end{bmatrix}
+
\begin{bmatrix}
\varphi_t \\ 0_{n_0}
\end{bmatrix}
,
\end{split}
\end{equation}
where $\eta$ is the step size, $\beta$ is the momentum parameter,
$H \in \reals^{n_0 \times n_0}$ is a PSD matrix, $\varphi_t \in \reals^{n_0}$ is some vector, and $I_{n_0}$ is the $n_0 \times n_0$-dimensional identity matrix. Note that $\xi_t$ and $\varphi_t$ depend on the underlying model learned at iteration $t$, i.e. depend on $W_t$.

We first show that the residual dynamics of Polyak's momentum for solving all the four problems in this paper are in the form of \eqref{eq:meta}.
The proof of the following lemmas (Lemma~\ref{lem:SC-residual},~\ref{lem:ReLU-residual}, and~\ref{lem:DL-residual}) are available in Appendix~\ref{app:sec:meta}.

\subsection{Realization: Strongly convex quadratic problems} \label{inst:stc}

One can easily see that the dynamics of Polyak's momentum (\ref{eq:A}) for solving the strongly convex quadratic problem (\ref{obj:strc}) is in the form of (\ref{eq:meta}). We thus have the following lemma.

\begin{lemma} \label{lem:stc-residual} 
Applying Algorithm~\ref{alg:HB1} or Algorithm~\ref{alg:HB2} to solving the class of strongly convex quadratic problems (\ref{obj:strc})  induces a residual dynamics in the form of 
(\ref{eq:meta}), where 
$\xi_t  = w_t - w_* (\text{and hence } n_0 = d), 
 H    = \Gamma,
\varphi_t = 0_d.
$
\end{lemma}

\subsection{Realization: Solving $F_{\mu,\alpha}^2$} \label{inst:SC}
A similar result holds for optimizing functions in $F_{\mu,\alpha}^2$.
\begin{lemma} \label{lem:SC-residual} 
Applying Algorithm~\ref{alg:HB1} or Algorithm~\ref{alg:HB2} to minimizing a function $f(w) \in F_{\mu,\alpha}^2$ induces a residual dynamics in the form of 
(\ref{eq:meta}), where 
$\xi_t  = w_t - w_* $,
$H = \int_0^1 \nabla^2 f\big( (1-\tau) w_0 + \tau w_* \big) d \tau$,
$\varphi_t  = \eta \big( \int_0^1 \nabla^2 f\big( (1-\tau) w_0 + \tau w_* \big) d \tau -  \int_0^1 \nabla^2 f\big( (1-\tau) w_t + \tau w_* \big) d \tau \big) (w_t - w_*)$, where
$w_*:=\arg\min_w f(w)$.
\end{lemma}

\subsection{Realization: One-layer ReLU network} \label{inst:ReLU}

\noindent
\textbf{More notations:}
For the analysis,
let us define the event 
$A_{ir} := \{  \exists w \in \reals^d: \| w - w_0^{(r)} \| \leq R^{\text{ReLU}}, \mathbbm{1} \{ x_i^\top w_0^{(r)}  \} \neq \mathbbm{1} \{ x_i^\top w \geq 0  \}  \},$
where $R^{\text{ReLU}}>0$ is a number to be determined later. The event $A_{ir}$ means that there exists a $w \in \reals^d$ which is within the $R^{\text{ReLU}}$-ball centered at the initial point $w_0^{(r)}$ such that its activation pattern of sample $i$ is different from that of $w_0^{(r)}$.
We also denote a random set $S_i := \{ r \in [m]: \mathbbm{1}\{ A_{ir} \}  = 0 \}$ 
and its complementary set $S_i^\perp := [m] \setminus S_i$.

Lemma~\ref{lem:ReLU-residual} below shows that training the ReLU network $\N_W^{\text{-ReLU}}(\cdot)$ via momentum induces the residual dynamics in the form of (\ref{eq:meta}).


\begin{lemma} \label{lem:ReLU-residual} (Residual dynamics of training the ReLU network $\N_W^{\text{ReLU}}(\cdot)$)
Denote 
\[
\begin{aligned}
 & \textstyle (H_t)_{i,j}:= H(W_t)_{i,j} = \frac{1}{m} \sum_{r=1}^m x_i^\top x_j \\ & \textstyle \qquad \qquad \qquad \quad  \times \mathbbm{1}\{ \langle w^{(r)}_t, x_i \rangle \geq 0 \text{ } \&  \text{ }   \langle w^{(r)}_t, x_j \rangle \geq 0 \}.
\end{aligned}
\]
Applying Algorithm~\ref{alg:HB1} or Algorithm~\ref{alg:HB2} to (\ref{eq:obj}) for training the ReLU network 
$\N_W^{\text{ReLU}}(x)$ induces a residual dynamics in the form of 
(\ref{eq:meta}) such that
$\xi_t[i] = \N_{W_t}^{\text{ReLU}}(x_i) - y_i (\text{and hence } n_0 = d),
 H      = H_0, \text{ and }
 \varphi_t  = \phi_t + \iota_t,$
where 
each element $i$ of $\xi_t \in \reals^n$ is the residual error of the sample $i$, and
the $i_{th}$-element of $\phi_t\in \reals^n$ satisfies \[ \textstyle |\phi_t[i]| \leq \frac{ 2 \eta \sqrt{n} |S_i^\perp|}{ m } \big(  \| u_t - y \| + \beta \sum_{s=0}^{t-1} \beta^{t-1-s}  \| u_s - y \| \big),\]
and $\iota_t = \eta  \left( H_0 - H_t \right) \xi_t \in \reals^n $.
\end{lemma}


\subsection{Realization: Deep Linear network} \label{inst:linear}

Lemma~\ref{lem:DL-residual} below shows that the residual dynamics due to Polyak's momentum
for training the deep linear network 
is indeed in the form of (\ref{eq:meta}). 
In the lemma, ``$\v$'' stands for the vectorization of the underlying matrix in column-first order. 

\begin{lemma} \label{lem:DL-residual} (Residual dynamics of training $\N_W^{L\text{-linear}}(\cdot)$)
Denote 
$M_{t,l}$ the momentum term of layer $l$ at iteration $t$, which is recursively defined as
$M_{t,l} = \beta M_{t,l-1} + \frac{ \partial \ell(\W{L:1}_t)}{ \partial \W{l}_t } $. Denote
\[
\begin{aligned}
& \textstyle H_t \textstyle := \frac{1}{ m^{L-1} d_y } \sum_{l=1}^L [ (\W{l-1:1}_t X)^\top (\W{l-1:1}_t X ) 
\\ & \textstyle \qquad \qquad \qquad \quad \textstyle
\otimes
  \W{L:l+1}_t (\W{L:l+1}_t)^\top ]   \in \reals^{d_y n \times d_y n}.
\end{aligned}
\]
Applying Algorithm~\ref{alg:HB1} or Algorithm~\ref{alg:HB2} to (\ref{eq:obj}) for training the deep linear network 
$\N_W^{L\text{-linear}}(x)$ induces a residual dynamics in the form of 
(\ref{eq:meta}) such that
$\xi_t     = \text{vec}(U_t - Y) \in \reals^{d_y n} (\text{and hence } n_0 = d_y n), 
 H  =  H_0, \text{ and }  
\varphi_t  = \phi_t + \psi_t  + \iota_t \in \reals^{d_y n},
$
where the vector 
$\phi_t   = \frac{1}{\sqrt{m^{L-1} d_y} } \v\left( \Phi_t X\right)$ \text{ with } 
\[
\begin{aligned}
& \textstyle  \Phi_t    = \Pi_l \left( \W{l}_t - \eta M_{t,l} \right)
- \W{L:1}_t  \\ & \textstyle \qquad \qquad \qquad \qquad \quad \textstyle  +  \eta \sum_{l=1}^L \W{L:l+1}_t M_{t,l} \W{l-1:1}_t,
\end{aligned}
\]
and the vector $\psi_t$ is 
\[
\begin{aligned}
& 
\textstyle \psi_t = 
\frac{1}{\sqrt{m^{L-1} d_y} } 
\v\big( (L-1) \beta \W{L:1}_{t}  X + \beta  \W{L:1}_{t-1} X 
\\ & \textstyle \qquad \qquad \qquad \quad \textstyle 
- \beta \sum_{l=1}^L \W{L:l+1}_t \W{l}_{t-1} \W{l-1:1}_{t} X \big),
 \end{aligned}
\]
and $ \iota_t = \eta (H_0 - H_t) \xi_t$.
\end{lemma}

\subsection{A key theorem of bounding a matrix-vector product}

Our meta theorem of acceleration will be based on 
Theorem~\ref{thm:akv} in the following, which upper-bounds the size of the matrix-vector product of a matrix power $A^k$ and a vector $v_0$.
Compared to Gelfand's formula (Theorem~\ref{thm:Gelfand}),
Theorem~\ref{thm:akv} below provides a better control of the size of the matrix-vector product, 
since it avoids the dependency on the unknown sequence $\{ \epsilon_t\}$.
The result can be of independent interest and might be useful for analyzing Polyak's momentum for other problems in future research. 

\begin{theorem} \label{thm:akv}
Let $A:=
\begin{bmatrix} 
(1 + \beta) I_n - \eta  H  & - \beta I_n \\
I_n & 0
\end{bmatrix}
\in \reals^{2 n \times 2 n}$. 
Suppose that $H \in \reals^{n \times n}$ is a positive semidefinite matrix. 
Fix a vector $v_0 \in \reals^n$.
If $\beta$ is chosen to satisfy 
$1 \geq \beta >  \max \{ \left( 1 - \sqrt{\eta \lambda_{\min}(H) } \right)^2, \left( 1 - \sqrt{\eta \lambda_{\max}(H) } \right)^2 \}$, 
then
\begin{equation}
\| A^k v_0 \| \leq \big( \sqrt{ \beta } \big)^k C_0 \| v_0 \|, 
\end{equation}
where the constant 
\begin{equation} \label{C_0}
C_0:= \frac{\sqrt{2} (\beta+1)}{
\sqrt{ \min\{  
h(\beta,\eta \lambda_{\min}(H)) , h(\beta,\eta \lambda_{\max}(H)) \} } }\geq 1,
\end{equation} 
and the function $h(\beta,z)$ is defined as
$  h(\beta,z):=-\left(\beta-\left(1-\sqrt{z}\right)^2\right)\left(\beta-\left(1+\sqrt{z}\right)^2\right).  $
\end{theorem}

Note that the constant $C_0$ in Theorem~\ref{thm:akv} depends on $\beta$ and $\eta H$. It should be written as $C_0(\beta,\eta H)$ to be precise. 
However, for the brevity, we will simply denote it as $C_0$ when the underlying choice of $\beta$ and $\eta H$ is clear from the context. 
The proof of Theorem~\ref{thm:akv} is available in Appendix~\ref{app:thm:akv}.
Theorem~\ref{thm:akv} allows us to derive a concrete upper bound of the residual errors
in each iteration of momentum, and consequently allows us to show an accelerated linear rate in the non-asymptotic sense. 
The favorable property of the bound will also help to analyze Polyak's momentum for training the neural networks. As shown later in this paper, we will need to guarantee the progress of Polyak's momentum in each iteration, which is not possible if we only have a quantifiable bound in the limit. 
Based on Theorem~\ref{thm:akv}, we have the following corollary. The proof is in Appendix~\ref{app:corr}.
\begin{corollary}  \label{corr:1}
Assume that $\lambda_{\min}(H) > 0$.
Denote $\kappa:= \lambda_{\max}(H) / \lambda_{\min}(H)$.
Set $\eta = 1 / \lambda_{\max}(H)$ 
and set $\beta = \left( 1 - \frac{1}{2} \sqrt{\eta \lambda_{\min}(H)} \right)^2 = \left( 1 - \frac{1}{2 \sqrt{\kappa}} \right)^2$.
Then, $C_0 \leq  4 \sqrt{\kappa}$.
\end{corollary}


\subsection{Meta theorem}

Let $\lambda >0$ be the smallest eigenvalue of the matrix $H$ that appears on the residual dynamics (\ref{eq:meta}).
Our goal is to show that the residual errors satisfy
\begin{equation} \label{assump:eq}
\textstyle
\left\|
\begin{bmatrix}
\xi_{s} \\
\xi_{s-1} 
\end{bmatrix}
\right\|
\leq
\left(\sqrt{\beta } + \mathbbm{1}_{\varphi} C_2 \right)^{s} (C_{0} + \mathbbm{1}_{\varphi} C_1) 
\left\|
\begin{bmatrix}
\xi_{0} \\
\xi_{-1} 
\end{bmatrix}
\right\|,
\end{equation}
where $C_0$ is the constant defined on (\ref{C_0}),
and $C_1,C_2 \geq 0$ are some constants,
$\mathbbm{1}_{\varphi}$ is an indicator if any $\varphi_t$ on the residual dynamics (\ref{eq:meta}) is a non-zero vector.
For the case of training the neural networks, we have 
$\mathbbm{1}_{\varphi}= 1$. 

\begin{theorem} \label{thm:meta} (Meta theorem for the residual dynamics (\ref{eq:meta}))
Assume that
the step size $\eta$ and
the momentum parameter $\beta$ satisfying 
$1 \geq \beta >  \max \{ \left( 1 - \sqrt{\eta \lambda_{\min}(H) } \right)^2, \left( 1 - \sqrt{\eta \lambda_{\max}(H) } \right)^2 \}$, 
are set appropriately so that (\ref{assump:eq})
holds at iteration $s=0,1,\dots,t-1$
implies that
\begin{equation} \label{eq:thm1}
\textstyle
\| \sum_{s=0}^{t-1} A^{t-s-1} \begin{bmatrix}
\varphi_s \\
0 
\end{bmatrix}
\|
\leq \left(\sqrt{\beta} + \mathbbm{1}_{\varphi} C_2 \right)^{t}
C_3
\left\|
\begin{bmatrix}
\xi_{0} \\
\xi_{-1} 
\end{bmatrix}
\right\|.
\end{equation}
Then, we have
\begin{equation} \label{eq:thm2}
\textstyle
\left\|
\begin{bmatrix}
\xi_{t} \\
\xi_{t-1} 
\end{bmatrix}
\right\|
\leq
\left( \sqrt{\beta}  + \mathbbm{1}_{\varphi} C_2 \right)^{t} 
(C_0 + \mathbbm{1}_{\varphi} C_1) 
\left\|
\begin{bmatrix}
\xi_{0} \\
\xi_{-1} 
\end{bmatrix}
\right\|,
\end{equation}
holds for all $t$, where $C_0$ is defined on (\ref{C_0}) and $C_1,C_2,C_3 \geq 0$ are some constants satisfying:
\begin{equation} \label{eq:C-meta}
\begin{aligned}
& \textstyle 
\left( \sqrt{\beta} \right)^{t} C_0 +
\left( \sqrt{\beta} + \mathbbm{1}_{\varphi} C_2 \right)^{t} \mathbbm{1}_{\varphi} C_3 \leq 
\\ & 
\textstyle \qquad \qquad \qquad \qquad 
\left( \sqrt{\beta} + \mathbbm{1}_{\varphi} C_2 \right)^{t} (C_0 + \mathbbm{1}_{\varphi} C_1) .
\end{aligned}
\end{equation}
\end{theorem}

\begin{proof}
The proof is by induction. At $s=0$,  (\ref{assump:eq}) holds
since $C_0 \geq 1$ by Theorem~\ref{thm:akv}. Now assume that the inequality holds at $s=0,1,\dots, t-1$.
Consider iteration $t$.
Recursively expanding the dynamics (\ref{eq:meta}), we have 
\begin{equation} \label{eq:1-meta}
\begin{bmatrix}
\xi_{t} \\
\xi_{t-1} 
\end{bmatrix}
=  
A^t
\begin{bmatrix}
\xi_{0} \\
\xi_{-1} 
\end{bmatrix}
+ \sum_{s=0}^{t-1}  A^{t-s-1} \begin{bmatrix}
\varphi_s \\
0 
\end{bmatrix}
.
\end{equation}
By Theorem~\ref{thm:akv},
the first term on the r.h.s. of (\ref{eq:1-meta}) can be bounded by 
\begin{equation} \label{eq:2-meta}
\| A^t
\begin{bmatrix}
\xi_{0} \\
\xi_{-1} 
\end{bmatrix}
\| \leq \left( \sqrt{\beta} \right)^t C_0 \| \begin{bmatrix}
\xi_{0} \\
\xi_{-1} 
\end{bmatrix}
\|
\end{equation}
By assumption, given (\ref{assump:eq}) holds at $s=0,1,\dots, t-1$, we have (\ref{eq:thm1}).
Combining (\ref{eq:thm1}), (\ref{eq:C-meta}), (\ref{eq:1-meta}), and (\ref{eq:2-meta}),
we have (\ref{eq:thm2}) and hence the proof is completed.

\end{proof}

\noindent
\textbf{Remark:} As shown in the proof, we need the residual errors be tightly bounded as (\ref{assump:eq}) in each iteration.
Theorem~\ref{thm:akv} is critical for establishing the desired result.
On the other hand, it would become tricky if instead we use Gelfand's formula or other techniques 
in the related works
that lead to a convergence rate in the form of $O(t \theta^{t})$.

\section{Main results} \label{sec:main}

The important lemmas and theorems in the previous section help to show our main results in the following subsections. 
The high-level idea to obtain the results is by using the meta theorem (i.e. Theorem~\ref{thm:meta}). Specifically, we will need to show that if the underlying residual dynamics satisfy (\ref{assump:eq}) for all the previous iterations, then
the terms $\{ \varphi_s \}$ in the dynamics satisfy
 (\ref{eq:thm1}). 
This condition trivially holds for the case of the quadratic problems, since there is no such term.  
On the other hand,
for solving the other problems, we need to carefully show that the condition holds. 
For example, 
according to Lemma~\ref{lem:ReLU-residual},
showing acceleration for the ReLU network will require bounding terms like $\| (H_0 - H_s) \xi_s \|$ (and other terms as well), where
$H_0-H_s$ corresponds to the difference of the kernel matrix at two different time steps. By controlling the width of the network, we can guarantee that the change is not too much.
A similar result can be obtained for the problem of the deep linear network.
The high-level idea is simple but the analysis
of the problems of the neural networks can be tedious.

\subsection{Non-asymptotic accelerated linear rate for solving strongly convex quadratic problems}

\begin{theorem} \label{thm:stcFull}
Assume the momentum parameter $\beta$ satisfies
$1 \geq \beta >  \max \{ \left( 1 - \sqrt{\eta \mu } \right)^2, \left( 1 - \sqrt{\eta \alpha } \right)^2 \}$.
Gradient descent with Polyak's momentum for solving (\ref{obj:strc}) has
\begin{equation} \label{eq:qqq0}
\|
\begin{bmatrix}
w_{t} - w_* \\
w_{t-1} - w_*
\end{bmatrix}
\| \leq \left(  \sqrt{\beta}  \right)^{t} C_0
\|
\begin{bmatrix}
w_{0} - w_* \\
w_{-1} - w_*
\end{bmatrix}
\|,
\end{equation}
where the constant $C_0$ is defined as
\begin{equation}
\textstyle C_0:=\frac{\sqrt{2} (\beta+1)}{
\sqrt{ \min\{  
h(\beta,\eta \lambda_{\min}(\Gamma)) , h(\beta,\eta \lambda_{\max}(\Gamma)) \} } } \geq 1,
\end{equation} 
and 
$   h(\beta,z)=-\left(\beta-\left(1-\sqrt{z}\right)^2\right)\left(\beta-\left(1+\sqrt{z}\right)^2\right).$  
Consequently, if the step size $\eta = \frac{1}{\alpha}$ and the momentum parameter $\beta = \left(1 - \frac{1}{2\sqrt{\kappa}}\right)^2$, then it has
\begin{equation}\label{eq:qqq1}
\|
\begin{bmatrix}
w_{t} - w_* \\
w_{t-1} - w_*
\end{bmatrix}
\| \leq \left(  1 - \frac{1}{2 \sqrt{\kappa}}   \right)^{t} 4 \sqrt{\kappa}
\|
\begin{bmatrix}
w_{0} - w_* \\
w_{-1} - w_*
\end{bmatrix}
\|.
\end{equation}
Furthermore, if $\eta = \frac{4}{(\sqrt{\mu}+\sqrt{\alpha})^2}$ 
and $\beta$ approaches $\beta \rightarrow \left( 1 - \frac{2}{\sqrt{\kappa}+1} \right)^2$ from above, then it has a convergence rate approximately
$ \left(  1 - \frac{2}{\sqrt{\kappa} + 1}   \right)$
as $t \rightarrow \infty$.
\end{theorem}
The convergence rates shown in the above theorem do not depend on the unknown sequence $\{\epsilon_t\}$. Moreover, the rates depend on the squared root of the condition number $\sqrt{\kappa}$.
We have hence
established a non-asymptotic accelerated linear rate of Polyak's momentum, which helps to show the advantage of Polyak's momentum over vanilla gradient descent in the finite $t$ regime.
Our result also recovers the rate
$\left(  1 - \frac{2}{\sqrt{\kappa}+1} \right)$ asymptotically under the same choices of the parameters as the previous works.
The detailed proof can be found
in Appendix~\ref{app:sec:stc}, which is actually a 
 trivial application of
Lemma~\ref{lem:stc-residual}, Theorem~\ref{thm:meta}, and Corollary~\ref{corr:1} with $C_1=C_2=C_3=0$.


\subsection{Non-asymptotic accelerated linear rate of the local convergence
for solving $f(\cdot) \in F_{\mu,\alpha}^2$}  

Here we provide a local acceleration result of the discrete-time Polyak's momentum for general smooth strongly convex and twice differentiable function $F_{\mu,\alpha}^2$.
Compared to Theorem~9 of \citep{P64},
Theorem~\ref{thm:STC} clearly indicates the required distance that ensures 
an acceleration when the iterate is in the neighborhood of the global minimizer.
Furthermore,
the rate is in the non-asymptotic sense instead of the asymptotic one.
We defer the proof of Theorem~\ref{thm:STC} to Appendix~\ref{app:thm:STC}.

\begin{theorem} \label{thm:STC}
Assume that the function $f(\cdot) \in F_{\mu,\alpha}^2$ and its Hessian is $\alpha$-Lipschitz. Denote the condition number $\kappa:= \frac{\alpha}{\mu}$.
Suppose that the initial point satisfies $\| \begin{bmatrix} w_0 - w_* \\ w_{-1} - w_* \end{bmatrix} \| \leq
\frac{1}{ 683 \kappa^{3/2}}$. Then,
Gradient descent with Polyak's momentum with the step size $\eta = \frac{1}{\alpha}$ and the momentum parameter $\beta = \left(1 - \frac{1}{2 \sqrt{\kappa}}\right)^2$ for solving $\min_{w} f(w)$ has
\begin{equation}\label{thm:qq1}
\|
\begin{bmatrix}
w_{t+1} - w_* \\
w_{t} - w_*
\end{bmatrix}
\| \leq \left(  1 - \frac{1}{4 \sqrt{\kappa}}   \right)^{t+1} 8 \sqrt{\kappa}
\|
\begin{bmatrix}
w_{0} - w_* \\
w_{-1} - w_*
\end{bmatrix}
\|,
\end{equation}
where $w_* = \arg\min_w f(w)$.
\end{theorem}


\subsection{Acceleration for training $\N_W^{\text{ReLU}}(x)$}

Before introducing our result of training the ReLU network, we need the following lemma.
\begin{lemma} \label{lem:ReLU-A} 
[Lemma~3.1 in \citet{DZPS19} and \citet{ZY19}] 
Set $m= \Omega( \lambda^{-2} n^2 \log ( n / \delta )  )$. 
Suppose that the neurons $w^{(1)}_0, \dots, w^{(m)}_0$ are i.i.d. generated by $N(0,I_d)$ initially.
Then,
with probability at least $1-\delta$,
it holds that
\[
\begin{aligned}
& \| H_0 - \bar{H}\|_F \leq \frac{\lambda_{\min}( \bar{H} )}{4}, \text{ }
\lambda_{\min}\big(H_0\big) \geq \frac{3}{4}\lambda_{\min}( \bar{H} ),
\\ &
\text{ and } \qquad
\lambda_{\max}\big(H_0\big) \leq \lambda_{\max}( \bar{H} ) + \frac{ \lambda_{\min}( \bar{H} )}{4}.
\end{aligned}
\]
\end{lemma}

Lemma~\ref{lem:ReLU-A} shows that by the random initialization, with probability $1-\delta$, the least eigenvalue of the Gram matrix $H : = H_0$ defined in Lemma~\ref{lem:ReLU-residual} is lower-bounded and the largest eigenvalue is close to $\lambda_{\max}( \bar{H} )$. 
Furthermore, Lemma~\ref{lem:ReLU-A} implies that
the condition number of the Gram matrix $H_0$ at the initialization $\hat{\kappa}:= \frac{\lambda_{\max}(H_0) }{\lambda_{\min}(H_0) }$ satisfies
$\hat{\kappa} 
\leq \frac{4}{3} \kappa + \frac{1}{3},$
where $\kappa:= \frac{\lambda_{\max}(\bar{H}) }{\lambda_{\min}(\bar{H}) }$.

\begin{theorem} \label{thm:acc}
(One-layer ReLU network $\N_W^{\text{ReLU}}(x)$)
Assume that $\lambda := \frac{3\lambda_{\min}( \bar{H})}{4} > 0$ and that $w_0^{(r)} \sim N(0,I_d)$ and $a_r$ uniformly sampled from $\{-1,1\}$.
Denote $\lambda_{\max} := \lambda_{\max} (\bar{H}) + \frac{\lambda_{\min}(\bar{H})}{4}$ and denote $\hat{\kappa} := \lambda_{max} / \lambda = (4 \kappa + 1) / 3 $. 
Set a constant step size $\eta = \frac{1}{ \lambda_{\max} }$,
fix momentum parameter $\beta  
= \big( 1 - \frac{1}{2 \hat{\kappa} }  \big)^2$, and finally set the number of network nodes $m = \Omega( \lambda^{-4} n^{4} \kappa^2 \log^3 ( n / \delta))$.
Then, with probability at least $1-\delta$ over the random initialization,
gradient descent with Polyak's momentum satisfies for any $t$,
\begin{equation} \label{eq:thm0}
\left\|
\begin{bmatrix}
\xi_{t}  \\
\xi_{t-1}  
\end{bmatrix}
\right\| \leq \left( 1 - \frac{1}{4 \sqrt{\hat{\kappa} } }  \right)^{t}
\cdot 8  \sqrt{ \hat{\kappa} }
\left\|
 \begin{bmatrix}
\xi_0 \\
\xi_{-1}
\end{bmatrix}
\right\|.
\end{equation}
\end{theorem}

We remark that $\hat{\kappa}$, which is the condition number of the Gram matrix $H_0$, is within a constant factor of the condition number of $\bar{H}$. 
Therefore, Theorem~\ref{thm:acc} essentially shows an accelerated linear rate 
$\left( 1 - \Theta ( \frac{1}{\sqrt{\kappa}}) \right)$.
The rate has an improved dependency on the condition number, i.e. $\sqrt{\kappa}$ instead of $\kappa$, which shows the advantage of Polyak's momentum over vanilla GD when the condition number is large.  
We believe this is an interesting result, as the acceleration is akin to that in convex optimization, e.g. \citet{N13,SDJS18}.

Our result also implies that
over-parametrization helps acceleration in optimization.
To our knowledge, in the literature, there is little theory of understanding why over-parametrization can help training a neural network faster.
The only exception that we are aware of is \citet{ACH18}, which shows that the dynamic of vanilla gradient descent for an over-parametrized objective function exhibits some momentum terms, although their message is very different from ours. The proof of Theorem~\ref{thm:acc} is in Appendix~\ref{app:sec:relu}.

\subsection{Acceleration for training $\N_W^{L\text{-linear}}(x)$}

\begin{theorem} \label{thm:LinearNet}
(Deep linear network $\N_W^{L\text{-linear}}(x)$)
Denote $\lambda:= \frac{L \sigma_{\min}^2(X)} {d_y}$ and $\kappa:= \frac{\sigma_{\max}^2(X)}{\sigma_{\min}^2(X)}$.
Set a constant step size $\eta = \frac{d_y}{L \sigma_{\max}^2(X)}$,
fix momentum parameter $\beta 
= \big( 1 - \frac{1}{2 \sqrt{\kappa}} \big)^2 $, and finally set a parameter $m$ that controls the width $m \geq C \frac{ \kappa^5 }{\sigma^2_{\max}(X)} \left( 
d_y( 1 + \|W^* \|^2_2) + \log ( \bar{r} / \delta)  \right)$ and $m \geq \max \{ d_x, d_y\}$ for some constant $C>0$.
Then, with probability at least $1-\delta$ over the random orthogonal initialization,
gradient descent with Polyak's momentum satisfies for any $t$,
\begin{equation} \label{eq:thm-LinearNet}
\left\|
\begin{bmatrix}
\xi_t \\
\xi_{t-1} 
\end{bmatrix}
\right\| \leq \left( 1 - \frac{1}{4 \sqrt{\kappa} } \right)^{t}
\cdot 8 \sqrt{\kappa}
\left\|
 \begin{bmatrix}
\xi_0 \\
\xi_{-1}
\end{bmatrix}
\right\|.
\end{equation}

\end{theorem}

Compared with Theorem~\ref{thm:hu} of \citet{HXP20} for vanilla GD, our result clearly shows the acceleration via Polyak's momentum.
Furthermore, the result suggests that the depth does not hurt optimization. Acceleration is achieved for any depth $L$ and the required width $m$
is independent of the depth $L$ as \citet{HXP20,ZLG20} (of vanilla GD).
The proof of Theorem~\ref{thm:LinearNet} is in Appendix~\ref{app:sec:linear}.

\section{Conclusion}

We show some non-asymptotic acceleration results of the discrete-time Polyak's momentum in this paper.
The results not only improve the previous results in convex optimization but also establish the first time that Polyak's momentum has provable acceleration for training certain neural networks.  
We analyze all the acceleration results from a modular framework.
We hope the framework can serve as a building block towards understanding Polyak's momentum in a more unified way.

\section{Acknowledgment}
The authors thank Daniel Pozo for catching a typo. The authors acknowledge support of NSF IIS Award 1910077. JW also thanks IDEaS-TRIAD Research Scholarship 03GR10000818.

\bibliography{modular_Acc}
\bibliographystyle{icml2021}


\appendix

\onecolumn

\section{Linear-rate results of the discrete-time Polyak's momentum} \label{app:history}

In the discrete-time setting,
for general smooth, strongly convex, and differentiable functions, a linear rate of the global convergence is shown by \citet{GFJ15} and \citet{SDJS18}.
However, the rate is not an accelerated rate and is not better than that of the vanilla gradient descent. 
To our knowledge, the class of the strongly convex quadratic problems is the only known example that Polyak's momentum has a provable \emph{accelerated linear rate} in terms of the \emph{global convergence} in the \emph{discrete-time} setting.

\section{Proof of Lemma~\ref{lem:SC-residual}, Lemma~\ref{lem:ReLU-residual},
and Lemma~\ref{lem:DL-residual}
 } \label{app:sec:meta}

\noindent
\textbf{Lemma}~\ref{lem:SC-residual}:
\textit{ 
Applying Algorithm~\ref{alg:HB1} or Algorithm~\ref{alg:HB2} to minimizing a function $f(w) \in F_{\mu,\alpha}^2$ induces a residual dynamics in the form of 
(\ref{eq:meta}), where 
\[
\begin{aligned}
\xi_t & = w_t - w_*
\\ H     & = \int_0^1 \nabla^2 f\big( (1-\tau) w_0 + \tau w_* \big) d \tau
\\ \varphi_t & = \eta \left( \int_0^1 \nabla^2 f\big( (1-\tau) w_0 + \tau w_* \big) d \tau -  \int_0^1 \nabla^2 f\big( (1-\tau) w_t + \tau w_* \big) d \tau \right) (w_t - w_*),
\end{aligned}
\]
where $w_*:=\arg\min_w f(w)$.
}

\begin{proof}
We have
\begin{equation} 
\begin{aligned}
\begin{bmatrix}
w_{t+1} - w_* \\
w_{t} - w_*
\end{bmatrix}
& =
\begin{bmatrix}
I_d  + \beta I_d &  - \beta I_d \\
I_d                      &  0_d
\end{bmatrix}
\cdot
 \begin{bmatrix}
w_{t} - w_* \\
w_{t-1} - w_* 
\end{bmatrix}
+ 
 \begin{bmatrix}
-\eta \nabla f(w_t) \\
 0 
\end{bmatrix}
\\ & =
\begin{bmatrix}
I_d  - \eta \int_0^1 \nabla^2 f\big( (1-\tau) w_t + \tau w_* \big) d \tau + \beta I_d &  - \beta I_d \\
I_d                      &  0_d
\end{bmatrix}
\cdot
 \begin{bmatrix}
w_{t} - w_* \\
w_{t-1} - w_* 
\end{bmatrix}
\\ & =
\begin{bmatrix}
I_d  - \eta \int_0^1 \nabla^2 f\big( (1-\tau) w_0 + \tau w_* \big) d \tau + \beta I_d &  - \beta I_d \\
I_d                      &  0_d
\end{bmatrix}
\cdot
 \begin{bmatrix}
w_{t} - w_* \\
w_{t-1} - w_* 
\end{bmatrix}
\\ & + 
\eta \left( 
\int_0^1 \nabla^2 f\big( (1-\tau) w_0 + \tau w_* \big) d \tau
- \int_0^1 \nabla^2 f\big( (1-\tau) w_t + \tau w_* \big) d \tau
 \right) (w_t - w_*),
\end{aligned}
\end{equation}
where the second equality is by the fundamental theorem of calculus.
\begin{equation}
\nabla f(w_t) - \nabla f(w_*) = \left( \int_0^1  \nabla^2 f( (1-\tau) w_t + \tau w_* ) d\tau  \right) (w_t - w_*),
\end{equation}
and that $\nabla f(w_*) = 0$.
\end{proof}

\noindent
\textbf{Lemma}~\ref{lem:ReLU-residual}:
\textit{
 (Residual dynamics of training the ReLU network $\N_W^{\text{ReLU}}(\cdot)$)
Denote 
\[
(H_t)_{i,j}:= H(W_t)_{i,j} = \frac{1}{m} \sum_{r=1}^m x_i^\top x_j \mathbbm{1}\{ \langle w^{(r)}_t, x_i \rangle \geq 0 \text{ } \&  \text{ }   \langle w^{(r)}_t, x_j \rangle \geq 0 \}.
\]
Applying Algorithm~\ref{alg:HB1} or Algorithm~\ref{alg:HB2} to (\ref{eq:obj}) for training the ReLU network 
$\N_W^{\text{ReLU}}(x)$ induces a residual dynamics in the form of 
(\ref{eq:meta}) such that
\[
\begin{aligned}
\xi_t[i] &= \N_{W_t}^{\text{ReLU}}(x_i) - y_i \quad  \text{and hence } n_0 = d
\\ H     & = H_0
\\ \varphi_t & = \phi_t + \iota_t,
\end{aligned}
\]
where 
each element $i$ of $\xi_t \in \reals^n$ is the residual error of the sample $i$,
the $i_{th}$-element of $\phi_t\in \reals^n$ satisfies \[ \textstyle |\phi_t[i]| \leq \frac{ 2 \eta \sqrt{n} |S_i^\perp|}{ m } \big(  \| u_t - y \| + \beta \sum_{s=0}^{t-1} \beta^{t-1-s}  \| u_s - y \| \big),\]
and $\iota_t = \eta  \left( H_0 - H_t \right) \xi_t \in \reals^n $.
}

\begin{proof}
For each sample $i$, we will divide the contribution to $\N(x_i)$ into two groups.
\begin{equation} \label{eq:Ndiv}
\begin{aligned}
 \N(x_i)
& 
  = \frac{1}{\sqrt{m} }  \sum_{r=1}^m a_r \sigma( \langle w^{(r)}, x_i \rangle  )
\\ &   = \frac{1}{\sqrt{m} }  \sum_{r \in S_i} 
a_r \sigma( \langle w^{(r)}, x_i \rangle  )
+ \frac{1}{\sqrt{m} }  \sum_{r \in S_i^\perp} 
a_r \sigma( \langle w^{(r)}, x_i \rangle  ).
\end{aligned}
\end{equation}
To continue, let us recall some notations;
the subgradient with respect to $w^{(r)} \in \reals^d$ is
\begin{equation}
\frac{ \partial L(W)}{ \partial w^{(r)} }
 := \frac{1}{\sqrt{m} } 
\sum_{i=1}^n \big( \N(x_i) - y_i  \big) a_r x_i \mathbbm{1}\{
 \langle w^{(r)}, x \rangle \geq 0 \},
\end{equation}
and the Gram matrix $H_t$ whose $(i,j)$ element is
\begin{equation}
H_t[i,j] := \frac{1}{m} x_i^\top x_j \sum_{r=1}^m \mathbbm{1} \{ \langle w_{t}^{(r)} , x_i \rangle \geq 0 \text{ \& } \langle w_{t}^{(r)} , x_j \rangle \geq 0 \}.
\end{equation}
Let us also denote
\begin{equation}
H_t^\perp[i,j] := \frac{1}{m} x_i^\top x_j \sum_{r \in S_i^\perp} \mathbbm{1} \{ \langle w_{t}^{(r)} , x_i \rangle \geq 0 \text{ \& } \langle w_{t}^{(r)} , x_j \rangle \geq 0 \}.
\end{equation}
We have that
\begin{equation} \label{eq:a0}
\begin{split}
\xi_{t+1}[i]  & =  \N_{t+1}(x_i) - y_i
\\ & \overset{(\ref{eq:Ndiv})}{ = }
\underbrace{
\frac{1}{\sqrt{m} }  \sum_{r \in S_i} 
a_r \sigma( \langle w_{t+1}^{(r)}, x_i \rangle  ) }_{ \text{first term} }
+ \frac{1}{\sqrt{m} }  \sum_{r \in S_i^\perp} 
a_r \sigma( \langle w_{t+1}^{(r)}, x_i \rangle  ) - y_i.
\end{split}
\end{equation}
For the first term above, we have that
\begin{equation} \label{eq:a1}
\begin{split}
\textstyle  & \underbrace{
 \frac{1}{\sqrt{m} } \sum_{r \in S_i} 
a_r \sigma( \langle w_{t+1}^{(r)}, x_i \rangle  ) }_{ \text{first term} }
=
\frac{1}{\sqrt{m} }  \sum_{r \in S_i} 
a_r \sigma( \langle w_t^{(r)} - \eta \frac{ \partial L(W_t)}{ \partial w_t^{(r)} } + \beta ( w_{t}^{(r)} - w_{t-1}^{(r)} ), x_i \rangle  )
\\ 
= &
\frac{1}{\sqrt{m} }  \sum_{r \in S_i} 
a_r \langle w_t^{(r)} - \eta \frac{ \partial L(W_t)}{ \partial w_t^{(r)} } + \beta ( w_{t}^{(r)} - w_{t-1}^{(r)} ), x_i \rangle \cdot \mathbbm{1}\{ \langle w_{t+1}^{(r)} , x_i \rangle \geq 0 \}
\\ 
\overset{(a)}{=} & 
\frac{1}{\sqrt{m} }  \sum_{r \in S_i} 
a_r \langle w_t^{(r)}, x_i \rangle \cdot \mathbbm{1}\{ \langle w_{t}^{(r)} , x_i \rangle \geq 0 \}
+ 
\frac{\beta}{\sqrt{m} }  \sum_{r \in S_i} 
a_r \langle w_t^{(r)}, x_i \rangle \cdot \mathbbm{1}\{ \langle w_{t}^{(r)} , x_i \rangle \geq 0 \}
\\ 
& 
-
\frac{\beta}{\sqrt{m} }  \sum_{r \in S_i} 
a_r \langle w_{t-1}^{(r)}, x_i \rangle \cdot \mathbbm{1}\{ \langle w_{t-1}^{(r)} , x_i \rangle \geq 0 \}
- \eta \frac{1}{\sqrt{m} }  \sum_{r \in S_i} a_r \langle \frac{ \partial L(W_t)}{ \partial w_t^{(r)} }, x_i \rangle \mathbbm{1}\{ \langle w_{t}^{(r)} , x_i \rangle \geq 0 \}
\\  
=
&
\N_t(x_i) + \beta \big(  \N_t(x_i) -  \N_{t-1}(x_i)  \big)
- \frac{1}{\sqrt{m}} \sum_{r \in S_i^\perp} 
a_r \langle w_t^{(r)}, x_i \rangle \mathbbm{1}\{ \langle w_{t}^{(r)} , x_i \rangle \geq 0 \} 
\\ &
- 
 \frac{\beta}{\sqrt{m}} \sum_{r \in S_i^\perp}
a_r \langle w_{t}^{(r)}, x_i \rangle \mathbbm{1}\{ \langle w_{t}^{(r)} , x_i \rangle \geq 0 \}  
+ \frac{\beta}{\sqrt{m}} \sum_{r \in S_i^\perp} a_r \langle w_{t-1}^{(r)}, x_i \rangle \mathbbm{1}\{ \langle w_{t-1}^{(r)} , x_i \rangle \geq 0 \}
           \big) 
\\ &
           - \eta \underbrace{  \frac{1}{\sqrt{m} }  \sum_{r \in S_i} a_r \langle \frac{ \partial L(W_t)}{ \partial w_t^{(r)} }, x_i \rangle \mathbbm{1}\{ \langle w_{t}^{(r)} , x_i \rangle \geq 0 \} }_{\text{last term}},
\end{split}
 \end{equation}
 where (a) uses that for $r \in S_i$,
$\mathbbm{1}\{ \langle w_{t+1}^{(r)} , x_i \rangle \geq 0 \}
= \mathbbm{1}\{ \langle w_{t}^{(r)} , x_i \rangle \geq 0 \}
= \mathbbm{1}\{ \langle w_{t-1}^{(r)} , x_i \rangle \geq 0 \}$ as the 
 neurons in $S_i$ do not change their activation patterns.
We can further bound (\ref{eq:a1}) as
 \begin{equation} \label{eq:a3}
\begin{split}
\overset{(b)}{=}
&
\N_t(x_i) + \beta \big(  \N_t(x_i) -  \N_{t-1}(x_i)  \big)
- \eta
\sum_{j=1}^n \big( \N_t(x_j) - y_j \big)  H(W_t)_{i,j} 
\\ &
- \frac{\eta}{ m }
\sum_{j=1}^n x_i^\top x_j ( \N_t(x_j) - y_j )
 \sum_{r \in S_i^\perp}  
 \mathbbm{1}\{ \langle w_{t}^{(r)} , x_i \rangle \geq 0 \text{ \& } \langle w_{t}^{(r)} , x_j \rangle \geq 0 \}
\\ &
- \frac{1}{\sqrt{m}} \sum_{r \in S_i^\perp} 
a_r \langle w_t^{(r)}, x_i \rangle \mathbbm{1}\{ \langle w_{t}^{(r)} , x_i \rangle \geq 0 \} - 
 \frac{\beta}{\sqrt{m}} \sum_{r \in S_i^\perp}
a_r \langle w_{t}^{(r)}, x_i \rangle \mathbbm{1}\{ \langle w_{t}^{(r)} , x_i \rangle \geq 0 \}  
\\ &
+ \frac{\beta}{\sqrt{m}} \sum_{r \in S_i^\perp} a_r \langle w_{t-1}^{(r)}, x_i \rangle \mathbbm{1}\{ \langle w_{t-1}^{(r)} , x_i \rangle \geq 0 \}
           \big), 
\end{split}
 \end{equation}
where (b) is due to that
\begin{equation}
\begin{split}
 &  \underbrace{ \frac{1}{\sqrt{m} }  \textstyle  \sum_{r \in S_i} a_r \langle \frac{ \partial L(W_t)}{ \partial w_t^{(r)} }, x_i \rangle \mathbbm{1}\{ \langle w_{t}^{(r)} , x_i \rangle \geq 0 \} }_{\text{last term}}
\\   = &  
\frac{1}{ m }
\sum_{j=1}^n x_i^\top x_j ( \N_t(x_j) - y_j )
 \sum_{r \in S_i}  
 \mathbbm{1}\{ \langle w_{t}^{(r)} , x_i \rangle \geq 0 \text{ \& } \langle w_{t}^{(r)} , x_j \rangle \geq 0 \}
\\   = &  
\sum_{j=1}^n \big( \N_t(x_j) - y_j \big) H(W_t)_{i,j}
 -  
\frac{1}{ m }
\sum_{j=1}^n x_i^\top x_j ( \N_t(x_j) - y_j )
 \sum_{r \in S_i^\perp}  
 \mathbbm{1}\{ \langle w_{t}^{(r)} , x_i \rangle \geq 0 \text{ \& } \langle w_{t}^{(r)} , x_j \rangle \geq 0 \}.
 \end{split}
\end{equation}
Combining (\ref{eq:a0}) and (\ref{eq:a3}), we have that
\begin{equation} \label{eq:a2}
\begin{split}
\xi_{t+1}[i] &   =   
\xi_t[i] + \beta \big(  \xi_t[i] -  \xi_{t-1}[i]  \big) - \eta
\sum_{j=1}^n  H_t[i,j]  
\xi_t[j] 
\\ & - \frac{\eta}{ m }
\sum_{j=1}^n x_i^\top x_j ( \N_t(x_j) - y_j )
 \sum_{r \in S_i^\perp}  
 \mathbbm{1}\{ \langle w_{t}^{(r)} , x_i \rangle \geq 0 \text{ \& } \langle w_{t}^{(r)} , x_j \rangle \geq 0 \}
\\ & 
 + \frac{1}{\sqrt{m} }  \sum_{r \in S_i^\perp} 
a_r \sigma( \langle w_{t+1}^{(r)}, x_i \rangle  )
- 
a_r \sigma( \langle w_{t}^{(r)}, x_i \rangle  )
- \beta a_r \sigma( \langle w_{t}^{(r)}, x_i \rangle  )
+\beta a_r \sigma( \langle w_{t-1}^{(r)}, x_i \rangle  ).
\end{split}
\end{equation}
So we can write the above into a matrix form.
\begin{equation} \label{k1}
\begin{split}
\xi_{t+1} & =  (I_n - \eta H_t )\xi_t + \beta ( \xi_t - \xi_{t-1} )
+ \phi_t
\\ & = (I_n - \eta H_0 )\xi_t + \beta ( \xi_t - \xi_{t-1} )
+ \phi_t + \iota_t,
\end{split}
\end{equation}
where 
the $i$ element of $\phi_t \in \reals^n$ is defined as
\begin{equation}
\begin{split}
 \phi_t[i] &  = 
- \frac{\eta}{ m }
\sum_{j=1}^n x_i^\top x_j ( \N_t(x_j) - y_j )
 \sum_{r \in S_i^\perp}  
 \mathbbm{1}\{ \langle w_{t}^{(r)} , x_i \rangle \geq 0 \text{ \& } \langle w_{t}^{(r)} , x_j \rangle \geq 0 \}
\\ &  +
\frac{1}{\sqrt{m} }   \sum_{r \in S_i^\perp} 
\big\{ 
a_r \sigma( \langle w_{t+1}^{(r)}, x_i \rangle  )
- 
a_r \sigma( \langle w_{t}^{(r)}, x_i \rangle  )
- \beta a_r \sigma( \langle w_{t}^{(r)}, x_i \rangle  )
+\beta a_r \sigma( \langle w_{t-1}^{(r)}, x_i \rangle  )
\big\}.
\end{split} 
\end{equation}
Now let us bound $\phi_t[i]$ as follows.
\begin{equation} \label{k2}
\begin{split}
  \phi_t[i] &  =  - \frac{\eta}{ m }
\sum_{j=1}^n x_i^\top x_j ( \N_t(x_j) - y_j )
 \sum_{r \in S_i^\perp}  
 \mathbbm{1}\{ \langle w_{t}^{(r)} , x_i \rangle \geq 0 \text{ \& } \langle w_{t}^{(r)} , x_j \rangle \geq 0 \}
\\ & +
 \frac{1}{\sqrt{m} }   \sum_{r \in S_i^\perp} 
\big\{ 
a_r \sigma( \langle w_{t+1}^{(r)}, x_i \rangle  )
- 
a_r \sigma( \langle w_{t}^{(r)}, x_i \rangle  )
- \beta a_r \sigma( \langle w_{t}^{(r)}, x_i \rangle  )
+\beta a_r \sigma( \langle w_{t-1}^{(r)}, x_i \rangle  )
\big\}
\\ & \overset{(a)}{ \leq } \frac{\eta \sqrt{n} |S_i^\perp | }{m} \| u_t - y \| +
\frac{ 1 }{ \sqrt{m} }  \sum_{r \in S_i^\perp}  
 \big(  \| w_{t+1}^{(r)} - w_t^{(r)}  \|
+ \beta \| w_{t}^{(r)} - w_{t-1}^{(r)} \| \big)
\\ & 
\overset{(b)}{=} \frac{\eta \sqrt{n} |S_i^\perp | }{m} \| u_t - y \| +
\frac{ \eta }{ \sqrt{m} }  \sum_{r \in S_i^\perp}   \big(  \| \sum_{s=0}^t \beta^{t-s} \frac{ \partial L(W_{s})}{ \partial w_{s}^{(r)} } \|
+ \beta \| \sum_{s=0}^{t-1} \beta^{t-1-s} \frac{ \partial L(W_{s})}{ \partial w_{s}^{(r)} } \| \big)
\\ & 
\overset{(c)}{\leq} \frac{\eta \sqrt{n} |S_i^\perp | }{m} \| u_t - y \| +
\frac{ \eta }{ \sqrt{m} }  \sum_{r \in S_i^\perp}   \big(  \sum_{s=0}^t \beta^{t-s}  \|  \frac{ \partial L(W_{s})}{ \partial w_{s}^{(r)} } \|
+ \beta \sum_{s=0}^{t-1} \beta^{t-1-s} \|  \frac{ \partial L(W_{s})}{ \partial w_{s}^{(r)} } \| \big)
\\ & 
\overset{(d)}{\leq} \frac{\eta \sqrt{n} |S_i^\perp | }{m} \| u_t - y \| +
\frac{ \eta \sqrt{n} |S_i^\perp|}{ m } \big(  \sum_{s=0}^t \beta^{t-s}  \| u_s - y \|
+ \beta \sum_{s=0}^{t-1} \beta^{t-1-s} \| u_s - y \| \big)
\\ & = 
\frac{ 2 \eta \sqrt{n} |S_i^\perp|}{ m }
\big(  \| u_t - y \| + \beta \sum_{s=0}^{t-1} \beta^{t-1-s}  \| u_s - y \| \big),
\end{split}
\end{equation}
where (a) is because
$-\frac{\eta}{ m }
\sum_{j=1}^n x_i^\top x_j ( \N_t(x_j) - y_j )
 \sum_{r \in S_i^\perp}  
 \mathbbm{1}\{ \langle w_{t}^{(r)} , x_i \rangle \geq 0 \text{ \& } \langle w_{t}^{(r)} , x_j \rangle \geq 0 \}
\leq \frac{\eta |S_i^\perp | }{m} \sum_{j=1}^n | \N_t(x_j) - y_j |
\leq \frac{\eta \sqrt{n} |S_i^\perp | }{m} \| u_t - y \|,
$
and that $\sigma(\cdot)$ is $1$-Lipschitz
so that 
\[
\begin{aligned}
& \frac{1}{ \sqrt{m} }   \sum_{r \in S_i^\perp}  
 \big ( 
a_r \sigma( \langle w_{t+1}^{(r)}, x_i \rangle  )
- 
a_r \sigma( \langle w_{t}^{(r)}, x_i \rangle ) \big) 
\leq \frac{1}{ \sqrt{m} }  \sum_{r \in S_i^\perp}  
| \langle w_{t+1}^{(r)}, x_i \rangle -  \langle w_{t}^{(r)}, x_i \rangle |
\\ & \leq \frac{1}{ \sqrt{m} }  \sum_{r \in S_i^\perp}  \| w_{t+1}^{(r)} -  w_{t}^{(r)} \| \| x_i \| \leq \frac{1}{\sqrt{m}} \sum_{r \in S_i^\perp}   \| w_{t+1}^{(r)} -  w_{t}^{(r)} \|,
\end{aligned}
\]
similarly, 
$\frac{-\beta}{ \sqrt{m} }   \sum_{r \in S_i^\perp}  
 \big ( 
a_r \sigma( \langle w_{t}^{(r)}, x_i \rangle  )
- 
a_r \sigma( \langle w_{t-1}^{(r)}, x_i \rangle ) \big) 
\leq 
\beta \frac{ 1}{ \sqrt{m} } \sum_{r \in S_i^\perp}   \| w_{t}^{(r)} - w_{t-1}^{(r)} \| 
$, (b) is by the update rule (Algorithm~\ref{alg:HB1}),
(c) is by Jensen's inequality, (d) is because
$|\frac{ \partial L(W_s)}{ \partial w_s^{(r)} }
| =| \frac{1}{\sqrt{m} } 
\sum_{i=1}^n \big( u_s[i] - y_i  \big) a_r x_i \mathbbm{1}\{
 x^\top w_t^{(r)} \geq 0 \} | \leq \frac{\sqrt{n}}{m} \| u_s - y \|$.

\end{proof}

\noindent
\textbf{Lemma: \ref{lem:DL-residual}}
\textit{
 (Residual dynamics of training $\N_W^{L\text{-linear}}(\cdot)$)
Denote 
$M_{t,l}$ the momentum term of layer $l$ at iteration $t$, which is recursively defined as
$M_{t,l} = \beta M_{t,l-1} + \frac{ \partial \ell(\W{L:1}_t)}{ \partial \W{l}_t } $. Denote
\[
\textstyle H_t \textstyle := \frac{1}{ m^{L-1} d_y } \sum_{l=1}^L [ (\W{l-1:1}_t X)^\top (\W{l-1:1}_t X ) 
\otimes
  \W{L:l+1}_t (\W{L:l+1}_t)^\top ]   \in \reals^{d_y n \times d_y n}.
\]
Applying Algorithm~\ref{alg:HB1} or Algorithm~\ref{alg:HB2} to (\ref{eq:obj}) for training the deep linear network 
$\N_W^{L\text{-linear}}(x)$ induces a residual dynamics in the form of 
(\ref{eq:meta}) such that
\[
\begin{aligned}
\xi_t &    = \text{vec}(U_t - Y) \in \reals^{d_y n} \text{, and hence } n_0 = d_y n \\
 H & =  H_0  \\
\varphi_t & = \phi_t + \psi_t  + \iota_t \in \reals^{d_y n},
\end{aligned}
\]
where
\[
\begin{aligned}
 \phi_t &  = \frac{1}{\sqrt{m^{L-1} d_y} } \v\left( \Phi_t X\right) \text{ with }
\Phi_t    = \Pi_l ( \W{l}_t - \eta M_{t,l} )
- \W{L:1}_t  +  \eta \sum_{l=1}^L \W{L:l+1}_t M_{t,l} \W{l-1:1}_t
\\   \psi_t & =\frac{1}{\sqrt{m^{L-1} d_y} } 
\v\left( (L-1) \beta \W{L:1}_{t}  X + \beta  \W{L:1}_{t-1} X 
 - \beta \sum_{l=1}^L \W{L:l+1}_t \W{l}_{t-1} \W{l-1:1}_{t} X \right)
\\ \iota_t  & = \eta (H_0 - H_t) \xi_t.
\end{aligned}
\]
}

\begin{proof}
According to the update rule of gradient descent with Polyak's momentum,
we have
\begin{equation} \label{eq:t1}
\W{L:1}_{t+1} = \Pi_l \left( \W{l}_t - \eta M_{t,l} \right)
= \W{L:1}_t - \eta \sum_{l=1}^L \W{L:l+1}_t M_{t,l} \W{l-1:1} + \Phi_t,
\end{equation}
where $M_{t,l}$ stands for the momentum term of layer $l$, which is
$M_{t,l} = \beta M_{t,l-1} + \frac{ \partial \ell(\W{L:1}_t)}{ \partial \W{l}_t } = 
\sum_{s=0}^t \beta^{t-s} \frac{ \partial \ell(\W{L:1}_s)}{ \partial \W{l}_s }$,
and 
$\Phi_t$ contains all the high-order terms (in terms of $\eta$), e.g. those with $\eta M_{t,i}$ and $ \eta M_{t,j}$, $i \neq j \in [L]$, or higher. Based on the equivalent update expression of gradient descent with Polyak's momentum
$- \eta M_{t,l} = - \eta \frac{ \partial \ell(\W{L:1}_t)}{ \partial \W{l}_t } +
\beta ( \W{l}_t - \W{l}_{t-1} )$,
we can rewrite (\ref{eq:t1}) as
\begin{equation}
\begin{aligned}
\W{L:1}_{t+1} 
& = \W{L:1}_t - \eta \sum_{l=1}^L \W{L:l+1}_t \frac{ \partial \ell(\W{L:1}_t)}{ \partial \W{l}_t } \W{l-1:1}_t 
+ \sum_{l=1}^L \W{L:l+1}_t \beta ( \W{l}_t - \W{l}_{t-1} ) \W{l-1:1}_t
+ \Phi_t
\\ & = \W{L:1}_t - \eta \sum_{l=1}^L \W{L:l+1}_t \frac{ \partial \ell(\W{L:1}_t)}{ \partial \W{l}_t } \W{l-1:1}_t 
+ \beta ( \W{L:1}_t - \W{L:1}_{t-1} ) + \Phi_t
\\ &
\quad + (L-1) \beta \W{L:1}_{t} + \beta  \W{L:1}_{t-1}
- \beta \sum_{l=1}^L \W{L:l+1}_t \W{l}_{t-1} \W{l-1:1}_{t}. 
\end{aligned}
\end{equation}
Multiplying the above equality with $\frac{1}{ \sqrt{ m^{L-1} d_y} } X$, we get
\begin{equation}
\begin{split}
U_{t+1} & = U_t - \eta \frac{1}{ m^{L-1} d_y } \sum_{l=1}^L \W{L:l+1}_t  (\W{L:l+1}_t)^\top
( U_t - Y )   (\W{l-1:1}_t X)^\top  \W{l-1:1}_t X  + \beta  (U_t - U_{t-1}) \\ & 
\quad + \frac{1}{ \sqrt{ m^{L-1} d_y} } \left( (L-1) \beta \W{L:1}_{t} + \beta  \W{L:1}_{t-1}
- \beta \sum_{l=1}^L \W{L:l+1}_t \W{l}_{t-1} \W{l-1:1}_{t} \right) X
+
\frac{1}{ \sqrt{ m^{L-1} d_y} } \Phi_t X.
\end{split}
\end{equation}
Using $\text{vec}(ACB) = (B^\top \otimes A) \text{vec}(C)$,
where $ \otimes$ stands for the Kronecker product,
 we can apply a vectorization of the above equation and obtain
\begin{equation} \label{eq:L1}
\begin{split}
\v(U_{t+1}) - \v(U_t)&  = - \eta H_t \v( U_t - Y ) +
\beta  \left( \v(U_{t}) - \v(U_{t-1})  \right)
\\ & \quad + 
\v(  \frac{1}{ \sqrt{ m^{L-1} d_y} } \left( (L-1) \beta \W{L:1}_{t} + \beta  \W{L:1}_{t-1}
- \beta \sum_{l=1}^L \W{L:l+1}_t \W{l}_{t-1} \W{l-1:1}_{t} \right) X )
\\ & \quad +
\frac{1}{ \sqrt{ m^{L-1} d_y} } \v( \Phi_t X),
\end{split}
\end{equation}
where 
\begin{equation}
H_t = \frac{1}{  m^{L-1} d_y } \sum_{l=1}^L \left[ \left( (\W{l-1:1}_t X)^\top (\W{l-1:1}_t X)
 \right)  \otimes \W{L:l+1}_t (\W{L:l+1}_t)^\top    \right],
\end{equation}
which is a positive semi-definite matrix. 

In the following, we will denote
$\xi_{t} := \v(U_t - Y) $ 
 as the vector of the residual errors. 
Also, we denote $\phi_t: = \frac{1}{ \sqrt{ m^{L-1} d_y} } \v( \Phi_t X)$
\text{ with }
$\Phi_t    = \Pi_l ( \W{l}_t - \eta M_{t,l} )
- \W{L:1}_t  +  \eta \sum_{l=1}^L \W{L:l+1}_t M_{t,l} \W{l-1:1}_t$,
 and  $\psi_t:= 
\v(  \frac{1}{ \sqrt{ m^{L-1} d_y} } \left( (L-1) \beta \W{L:1}_{t} + \beta  \W{L:1}_{t-1}
- \beta \sum_{l=1}^L \W{L:l+1}_t \W{l}_{t-1} \W{l-1:1}_{t} \right) X )$.
 Using the notations, we can rewrite (\ref{eq:L1}) as
\begin{equation} \label{eq:L2}
\begin{split}
\begin{bmatrix}
\xi_{t+1} \\
\xi_{t} 
\end{bmatrix}
& = 
\begin{bmatrix}
I_{d_y n} - \eta H_t + \beta  I_{d_y n} & - \beta  I_{d_y n}   \\
I_{d_y n} & 0_{d_y n} 
\end{bmatrix}
\begin{bmatrix}
\xi_{t} \\
\xi_{t-1} 
\end{bmatrix}
+
\begin{bmatrix}
\phi_t + \psi_t \\ 0_{d_y n}
\end{bmatrix}
\\ & = 
\begin{bmatrix}
I_{d_y n} - \eta H_0 + \beta  I_{d_y n} & - \beta  I_{d_y n}   \\
I_{d_y n} & 0_{d_y n} 
\end{bmatrix}
\begin{bmatrix}
\xi_{t} \\
\xi_{t-1} 
\end{bmatrix}
+
\begin{bmatrix}
\varphi_t \\ 0_{d_y n}
\end{bmatrix}
,
\end{split}
\end{equation}
where $\varphi_t = \phi_t + \psi_t + \iota_t \in \reals^{d_y n}$ and 
 $I_{d_y n}$ is the $d_y n \times d_y n$-dimensional identity matrix.

\end{proof}

\section{Proof of Theorem~\ref{thm:akv}} \label{app:thm:akv}

\noindent
\textbf{Theorem~\ref{thm:akv}}
\textit{
Let $A:=
\begin{bmatrix} 
(1 + \beta) I_n - \eta  H  & - \beta I_n \\
I_n & 0
\end{bmatrix}
\in \reals^{2 n \times 2 n}$. 
Suppose that $H \in \reals^{n \times n}$ is a positive semidefinite matrix. 
Fix a vector $v_0 \in \reals^n$.
If $\beta$ is chosen to satisfy 
$1 \geq \beta >  \max \{ \left( 1 - \sqrt{\eta \lambda_{\min}(H) } \right)^2, \left( 1 - \sqrt{\eta \lambda_{\max}(H) } \right)^2 \}$, 
then
\begin{equation}\label{thm:5rate}
\| A^k v_0 \| \leq \big( \sqrt{ \beta } \big)^k C_0 \| v_0 \|, 
\end{equation}
where the constant
\begin{equation} 
 C_0:=\frac{\sqrt{2} (\beta+1)}{
\sqrt{ \min\{  
h(\beta,\eta \lambda_{\min}(H)) , h(\beta,\eta \lambda_{\max}(H)) \} } } \geq 1,
\end{equation} 
and the function $h(\beta,z)$ 
is defined as
\begin{align}
     h(\beta,z) :=-\left(\beta-\left(1-\sqrt{z}\right)^2\right)\left(\beta-\left(1+\sqrt{z}\right)^2\right).  
\end{align}
}

We would first prove some lemmas for the analysis.
\begin{lemma} \label{lem:diagonal}
Under the assumption of Theorem~\ref{thm:akv}, $A$ is diagonalizable with respect to complex field $\mathbb{C}$ in $\mathbb{C}^n$, i.e., $\exists P$ such that $A = PDP^{-1}$ for some diagonal matrix $D$. Furthermore, the diagonal elements of $D$ all have magnitudes bounded by $\sqrt{\beta}$.
\end{lemma}

\begin{proof}
In the following, we will use the notation/operation ${\rm Diag}( \cdots )$
to represents a block-diagonal matrix that has the arguments on its main diagonal. 
Let $U{\rm Diag}([\lambda_1,\dots,\lambda_n]) U^*$ be the singular-value-decomposition of $H$, then 
\begin{align}
    A = 
\begin{bmatrix}
U& 0\\0& U
\end{bmatrix}
\begin{bmatrix} 
(1 + \beta) I_n - \eta {\rm Diag}([\lambda_1,\dots,\lambda_n])  & - \beta I_n \\
I_n & 0 
\end{bmatrix}
\begin{bmatrix}
U^*& 0\\0& U^*
\end{bmatrix}.
\end{align}
Let $\tilde{U} = \begin{bmatrix}
U& 0\\0& U
\end{bmatrix}$. Then, after applying some permutation matrix $\tilde{P}$, $A$ can be further simplified into 
\begin{align}\label{decompose1}
A = \tilde{U}\tilde{P} \Sigma \tilde{P}^T\tilde{U}^*,
\end{align}
where $\Sigma$ is a block diagonal matrix consisting of $n$ 2-by-2
matrices $\tilde{\Sigma}_i := \begin{bmatrix}
1+\beta-\eta \lambda_i& -\beta\\1& 0
\end{bmatrix}$. The characteristic polynomial of $\tilde{\Sigma}_i$ is $x^2 - (1+\beta -\lambda_i)x +\beta$. Hence it can be shown that when $\beta > (1-\sqrt{\eta \lambda_i})^2$ then the roots of polynomial are conjugate and have magnitude $\sqrt{\beta}$. These roots are exactly the eigenvalues of $\tilde{\Sigma}_i \in \reals^{2 \times 2}$. On the other hand, the corresponding eigenvectors $q_i,\bar{q}_i$ are also conjugate to each other as $\tilde{\Sigma}_i \in \reals^{2 \times 2}$ is a real matrix. As a result, $\Sigma \in \reals^{2n \times 2n}$ admits a block eigen-decomposition as follows,
\begin{align}\label{decompose2}
    \Sigma = & {\rm Diag}(\tilde{\Sigma}_i,\dots,\tilde{\Sigma}_n)\nonumber\\
    =&
    {\rm Diag}(Q_1,\dots,Q_n)
    {\rm Diag}\left(\begin{bmatrix}
    z_{1}&0\\0&\bar{z}_{1}
    \end{bmatrix},\dots,\begin{bmatrix}
    z_{n}&0\\0&\bar{z}_{n}
    \end{bmatrix}\right){\rm Diag}(Q^{-1}_1,\dots,Q^{-1}_n),
\end{align}
where $Q_i= [q_i,\bar{q}_i]$ and $z_{i}, \bar{z}_{i}$ are eigenvalues of $\tilde{\Sigma}_i$ (they are conjugate by the condition on $\beta$). 
Denote $Q:={\rm Diag}(Q_1,\dots,Q_n)$ and
\begin{align}
D:={\rm Diag}\left(\begin{bmatrix}
    z_{1}&0\\0&\bar{z}_{1}
    \end{bmatrix},\dots,\begin{bmatrix}
    z_{n}&0\\0&\bar{z}_{n}
    \end{bmatrix}\right).
\end{align}
By combining (\ref{decompose1}) and (\ref{decompose2}), we have
\begin{align}
    A & = P {\rm Diag}\left(\begin{bmatrix}
    z_{1}&0\\0&\bar{z}_{1}
    \end{bmatrix},\dots,\begin{bmatrix}
    z_{n}&0\\0&\bar{z}_{n}
    \end{bmatrix}\right) P^{-1}
     = P D P^{-1} ,
\end{align}
where 
\begin{align} \label{def:P}
P = \tilde{U}\tilde{P}Q,
\end{align} by the fact that $\tilde{P}^{-1} = \tilde{P}^T$ and $\tilde{U}^{-1}=\tilde{U}^{*}$.  
\end{proof}

\begin{proof} (of Theorem~\ref{thm:akv})
Now we proceed the proof of Theorem~\ref{thm:akv}.
In the following, we denote $v_k := A^k v_0$ (so $v_k = A v_{k-1} )$.
Let $P$ be the matrix in Lemma~\ref{lem:diagonal}, and $u_k:=P^{-1}v_k$, the dynamic can be rewritten as $u_k=P^{-1} A v_{k-1} = P^{-1}APu_{k-1}=Du_{k-1}$. As $D$ is diagonal, we immediately have 
\begin{align}
    &\|u_k\|\leq \max_{i \in [n]}|D_{ii}|^k \|u_0\|\nonumber\\
    \Rightarrow ~&\|P^{-1}v_k\|\leq \max_{i \in [n]}|D_{ii}|^k \|P^{-1}v_0\|\nonumber\\
    \Rightarrow~ &\sigma_{\rm min}(P^{-1})\|v_k\|\leq \sqrt{\beta}^k \sigma_{\rm max}(P^{-1})\|v_0\|~~~(\rm Lemma~\ref{lem:diagonal}.)\nonumber\\
   \Rightarrow ~& \sigma^{-1}_{\rm max}(P)\|v_k\|\leq \sqrt{\beta}^k \sigma^{-1}_{\rm min}(P)\|v_0\|\nonumber\\
   \Rightarrow ~& \|v_k\|\leq \sqrt{\beta}^k \frac{\sigma_{\rm max}(P)}{\sigma_{\rm min}(P)}\|v_0\|\nonumber\\
   \Rightarrow ~& \|v_k\|\leq \sqrt{\beta}^k \sqrt{\frac{\lambda_{\rm max}(PP^*)}{\lambda_{\rm min}(PP^*)}}\|v_0\|.
\end{align}
Hence, now it suffices to prove upper bound and lower bound of $\lambda_{\rm max}$ and $\lambda_{\rm min}$, respectively. By using Lemma~\ref{lem:eigbound} in the following, we obtain the inequality of (\ref{thm:5rate}).
We remark that as $C_0$ is an upper-bound of the squared root of the condition number $\sqrt{\frac{\lambda_{\rm max}(PP^*)}{\lambda_{\rm min}(PP^*)}}$, it is lower bounded by $1$.
\end{proof}

\begin{lemma} \label{lem:eigbound}
Let $P$ be the matrix in Lemma~\ref{lem:diagonal}, then we have $\lambda_{\rm max}(PP^*)\leq 2(\beta + 1)$ and $\lambda_{\rm min}(PP^*) \geq \min\{h(\beta,\eta\lambda_{\min}(H)),h(\beta,\eta \lambda_{\max}(H))\}/(1+\beta)$, where 
\begin{align}
    h(\beta,z)=-\left(\beta-\left(1-\sqrt{z}\right)^2\right)\left(\beta-\left(1+\sqrt{z}\right)^2\right).
\end{align}
\end{lemma}
\begin{proof}
As (\ref{def:P}) in the proof of Lemma 2, $P = \tilde{U}\tilde{P}{\rm Diag}(Q_1,\dots,Q_n)$. Since $\tilde{U}\tilde{P}$ is unitary, it does not affect the spectrum of $P$, therefore, it suffices to analyze the eigenvalues of $QQ^*$, where $Q ={\rm Diag}(Q_1,\dots,Q_n)$. Observe that $QQ^*$ is a block diagonal matrix with blocks $Q_iQ_i^*$, the eigenvalues of it are exactly that of $Q_iQ_i^*$, i.e., $\lambda_{\rm max}(QQ^*)=\max_{i\in[n]}\lambda_{\rm max}(Q_iQ_i^*)$ and likewise for the minimum. Recall $Q_i = [q_i,\bar{q_i}]$ consisting of eigenvectors of $\tilde{\Sigma}_i := \begin{bmatrix}
1+\beta-\eta \lambda_i& -\beta\\1& 0
\end{bmatrix}$ with corresponding eigenvalues $z_{i},\bar{z}_{i}$. The eigenvalues satisfy 
\begin{align}\label{val}
    z_i + \bar{z}_i = 2\Re{z_i} &= 1+\beta -\eta\lambda_i, \\
    z_i\bar{z}_i &= |z_i|^2 = \beta. 
\end{align}
On the other hand, the eigenvalue equation $\tilde{\Sigma}_iq_i = z_i q_i$ together with (\ref{val}) implies $q_i = [z_i,1]^T$.
Furthermore, $Q_iQ_i^* = q_iq_i^* + \bar{q}_i\bar{q}_i^*=2\Re{q_iq_i^*} = 2\Re{q_i}\Re{q_i}^T + 2\Im{q_i}\Im{q_i}^T$. Thus, 
\begin{align}\label{vec}
    Q_iQ_i^* &= 2\Re{q_i}\Re{q_i}^T + 2\Im{q_i}\Im{q_i}^T\nonumber\\
    &= 2\left(\begin{bmatrix}
    \Re{z_i}\\1
    \end{bmatrix}\begin{bmatrix}
    \Re{z_i}~ 1
    \end{bmatrix}
    +
    \begin{bmatrix}
    \Im{z_i}\\0
    \end{bmatrix}\begin{bmatrix}
    \Im{z_i}~ 0
    \end{bmatrix}\right)\nonumber\\
    &=2\begin{bmatrix}
    |z_i|^2 & \Re{z_i}\\
     \Re{z_i} & 1
    \end{bmatrix}.
\end{align}
Let the eigenvalues of $Q_iQ_i^*$ be $\theta_1,\theta_2$, then by (\ref{val})-(\ref{vec}) we must have 
\begin{align}\label{xi}
    \theta_1 +\theta_2 &= 2(\beta +1), \\
    \theta_1\theta_2 &= 4\left(\beta -(\frac{1+\beta -\eta\lambda_i}{2})^2\right)\nonumber\\
    &=-\left(\beta-\left(1-\sqrt{\eta\lambda_i}\right)^2\right)\left(\beta-\left(1+\sqrt{\eta\lambda_i}\right)^2\right) \geq 0.
\end{align}
From (\ref{xi}), as both eigenvalues are nonnegative, we deduce that \begin{align}2(1+\beta) \geq\max\{\theta_1,\theta_2\}\geq \beta +1.
\end{align}
On the other hand, from (\ref{xi}) we also have \begin{align}\label{minbound}
\min\{\theta_1,\theta_2\}
= & \theta_1\theta_2/\max\{\theta_1,\theta_2\} \nonumber\\
\geq& -\left(\beta-\left(1-\sqrt{\eta\lambda_i}\right)^2\right)\left(\beta-\left(1+\sqrt{\eta\lambda_i}\right)^2\right)/(1+\beta)\nonumber\\
:=&h(\beta, \eta \lambda_i)/(1+\beta).
\end{align}
Finally, as the eigenvalues of $QQ^*$ are composed of exactly that of $Q_iQ_i^*$, applying the bound of (\ref{minbound}) to each $i$ we have
\begin{align}
    \lambda_{\rm min} (PP^*)\geq&\min_{i\in[n]}h(\beta,\eta \lambda_i)/(1+\beta)\nonumber\\
    \geq& \min\{h(\beta,\eta \lambda_{\min}(H)),h(\beta,\eta \lambda_{\max}(H))\}/(1+\beta),
\end{align}
where the last inequality follows from the facts that $\lambda_{\min}(H) \leq \lambda_i \leq \lambda_{\max}(H)$ and $h$ is concave quadratic function of of $\lambda$ in which the minimum must occur at the boundary.
\end{proof}

\subsection{Proof of Corollary~\ref{corr:1}} \label{app:corr}

\noindent
\textbf{Corollary~\ref{corr:1}}
\textit{
Assume that $\lambda_{\min}(H) > 0$. Denote $\kappa:= \lambda_{\max}(H) / \lambda_{\min}(H)$.
Set $\eta = 1 / \lambda_{\max}(H)$ 
and set $\beta = \left( 1 - \frac{1}{2} \sqrt{ \eta \lambda_{\min} (H) } \right)^2 = \left( 1 - \frac{1}{2 \sqrt{\kappa}}  \right)^2$.
Then, $C_0 \leq \max \{ 4  , 2 \sqrt{\kappa} \} \leq 4 \sqrt{\kappa}$.
}


\begin{proof}
For notation brevity, in the following, we let $\mu:= \lambda_{\min}(H)$
and $\alpha := \lambda_{\max}(H)$.
Recall that
$h(\beta,z)=-\left(\beta-\left(1-\sqrt{z}\right)^2\right)\left(\beta-\left(1+\sqrt{z}\right)^2\right)$.
We have
\begin{equation} \label{eq:h1}
\begin{split}
h(\beta, \eta \mu)
& = - \left( (1 - \frac{1}{2} \sqrt{\eta \mu})^2 -\left(1-\sqrt{\eta \mu}\right)^2\right)\left( (1 - \frac{1}{2} \sqrt{\eta \mu})^2 -\left(1+\sqrt{\eta \mu}\right)^2\right)  
\\ & = 3 \left( \sqrt{\eta \mu} - \frac{3}{4} \eta \mu \right)
\left( \sqrt{\eta \mu} + \frac{1}{4} \eta \mu \right)
= 3 \left( \frac{1}{\sqrt{\kappa}} - \frac{3}{4 \kappa}  \right)
\left( \frac{1}{\sqrt{\kappa}}  + \frac{1}{4 \kappa}  \right)
\end{split}
\end{equation}
and
\begin{equation} \label{eq:h2}
\begin{split}
h(\beta, \eta \alpha)
& = - \left( (1 - \frac{1}{2} \sqrt{\eta \mu})^2 -\left(1-\sqrt{\eta \alpha}\right)^2\right)\left( (1 - \frac{1}{2} \sqrt{\eta \mu})^2 -\left(1+\sqrt{\eta \alpha}\right)^2\right)  
\\ & =  \left( 2 \sqrt{\eta \alpha} - \sqrt{\eta \mu}
-  \eta \alpha + \frac{1}{4} \eta \mu \right)
\left( \sqrt{\eta \mu} + 2 \sqrt{\eta \alpha} + \eta \alpha - \frac{1}{4} \eta \mu \right)
\\ &
= \left(1 - \frac{1}{\sqrt{\kappa}} + \frac{1}{4 \kappa} \right)
\left(3 + \frac{1}{\sqrt{\kappa}} - \frac{1}{4 \kappa} \right).
\end{split}
\end{equation}
We can simplify it to get that
$h(\beta, \eta \alpha) = 3 - \frac{2}{\sqrt{\kappa}} - \frac{1}{2 \kappa} + \frac{1}{2 \kappa^{3/2}} - \frac{1}{16 \kappa^2}  \geq 0.5 $.

Therefore,
we have
\begin{equation}
\begin{split}
 & 
 \frac{\sqrt{2} (\beta+1)}{
\sqrt{ 
h(\beta,\eta \mu) } }
= 
\frac{ \sqrt{2} (\beta+1) }{  
\sqrt{3 \eta \mu (1 - \frac{1}{2} \sqrt{\eta \mu} - \frac{3}{16} \eta \mu ) } }
= 
\frac{ \sqrt{2} (\beta+1) }{  
\sqrt{3  (1 - \frac{1}{2} \sqrt{\eta \mu} - \frac{3}{16} \eta \mu ) } } \sqrt{\kappa}
\leq 
\frac{1 }{  
\sqrt{ (1 - \frac{1}{2}  - \frac{3}{16}  ) } } \sqrt{\kappa}
\leq 2 \sqrt{\kappa},
\end{split}
\end{equation} 
where  we use $\eta \mu = \frac{1}{\kappa}$. On the other hand,
 $\frac{\sqrt{2} (\beta+1)}{
\sqrt{ 
h(\beta,\eta \alpha) } }
\leq 4 .$
We conclude that
\begin{equation} 
 C_0 =\frac{\sqrt{2} (\beta+1)}{
\sqrt{ \min\{  
h(\beta,\eta \nu , h(\beta,\eta \alpha) \} }}\leq \max \{ 4  , 2 \sqrt{\kappa} \} \leq 4 \sqrt{\kappa}.
\end{equation}

\end{proof}


\section{Proof of Theorem~\ref{thm:stcFull} } \label{app:sec:stc}

\noindent
\textbf{Theorem~\ref{thm:stcFull}}
\textit{
Assume the momentum parameter $\beta$ satisfies 
$1 \geq \beta >  \max \{ \left( 1 - \sqrt{\eta \mu } \right)^2, \left( 1 - \sqrt{\eta \alpha } \right)^2 \} $.
Gradient descent with Polyak's momentum has
\begin{equation} \label{eq:qq0}
\|
\begin{bmatrix}
w_{t} - w_* \\
w_{t-1} - w_*
\end{bmatrix}
\| \leq \left(  \sqrt{\beta}  \right)^{t} C_0
\|
\begin{bmatrix}
w_{0} - w_* \\
w_{-1} - w_*
\end{bmatrix}
\|,
\end{equation}
where the constant
\begin{equation}
 C_0:=\frac{\sqrt{2} (\beta+1)}{
\sqrt{ \min\{  
h(\beta,\eta \lambda_{\min}(\Gamma)) , h(\beta,\eta \lambda_{\max}(\Gamma)) \} } },
\end{equation} 
and
$   h(\beta,z)=-\left(\beta-\left(1-\sqrt{z}\right)^2\right)\left(\beta-\left(1+\sqrt{z}\right)^2\right).$  
Consequently, if the step size $\eta = \frac{1}{\alpha}$ and the momentum parameter $\beta = \left(1 - \sqrt{\eta \mu}\right)^2$, then it has
\begin{equation}\label{eq:qq1}
\|
\begin{bmatrix}
w_{t} - w_* \\
w_{t-1} - w_*
\end{bmatrix}
\| \leq \left(  1 - \frac{1}{2 \sqrt{\kappa}}   \right)^{t} 4 \sqrt{\kappa}
\|
\begin{bmatrix}
w_{0} - w_* \\
w_{-1} - w_*
\end{bmatrix}
\|.
\end{equation}
Furthermore, if $\eta = \frac{4}{(\sqrt{\mu}+\sqrt{\alpha})^2}$ 
and $\beta$ approaches $\beta \rightarrow \left( 1 - \frac{2}{\sqrt{\kappa}+1} \right)^2$ from above, then it has a convergence rate approximately
$ \left(  1 - \frac{2}{\sqrt{\kappa} + 1}   \right)$
as $t \rightarrow \infty$.
}

\begin{proof}
The result (\ref{eq:qq0}) and (\ref{eq:qq1}) is due to a trivial combination of
Lemma~\ref{lem:stc-residual}, Theorem~\ref{thm:meta}, and Corollary~\ref{corr:1}.

On the other hand, set
$\eta = \frac{4}{(\sqrt{\mu}+\sqrt{\alpha})^2}$, 
the lower bound on $\beta$ becomes $ \max \{\left( 1 - \sqrt{\eta \mu } \right)^2,  \left( 1 - \sqrt{\eta \alpha} \right)^2 \}  = \left( 1 - \frac{2}{\sqrt{\kappa}+1} \right)^2$.
Since the rate is $r=\lim_{t\rightarrow\infty}\frac{1}{t}\log(\sqrt{\beta}^{t+1}C_0)=\sqrt{\beta}$, setting $\beta \downarrow  \left( 1 - \frac{2}{\sqrt{\kappa}+1} \right)^2$ from above leads to the rate of $\left(  1 - \frac{2}{\sqrt{\kappa} + 1}   \right)$. 
Formally, it is straightforward to show that $C_0=\Theta\left(1/\sqrt{\beta-(1-\frac{2}{1+\sqrt{\kappa}})^2}\right)$, hence, for any $\beta$ converges to $(1-\frac{2}{\sqrt{\kappa}+1})^2$ slower than inverse exponential of $\kappa$, i.e.,
$\beta = (1-\frac{2}{\sqrt{\kappa}+1})^2+(\frac{1}{\kappa})^{o(t)}$, we have $r =  1 - \frac{2}{\sqrt{\kappa} + 1}   $. 


\end{proof}

\clearpage

\section{Proof of Theorem~\ref{thm:STC}}
\label{app:thm:STC}
\begin{proof} (of Theorem~\ref{thm:STC}) 
In the following, we denote $\xi_t:= w_t - w_*$ and
denote $\lambda := \mu >0$, which is a lower bound of $\lambda_{\min}(H)$ of the matrix $H:= \int_0^1 \nabla^2 f\big( (1-\tau) w_0 + w_* \big) d \tau$ defined in Lemma~\ref{lem:SC-residual}, i.e.
$\lambda_{\min}(H) \geq \lambda.$
Also, denote $\beta_*:=1 - \frac{1}{2} \sqrt{\eta \lambda}$
and $\theta := \beta_* + \frac{1}{4} \sqrt{\eta \lambda} = 1 - \frac{1}{4} \sqrt{\eta \lambda} $. Suppose $\eta = \frac{1}{\alpha }$, where $\alpha$ is the smoothness constant.
Denote 
$
 C_0:=\frac{\sqrt{2} (\beta+1)}{
\sqrt{ \min\{  
h(\beta,\eta \lambda_{\min}(H)) , h(\beta,\eta \lambda_{\max}(H)) \} } }
\leq 4 \sqrt{\kappa}
$ by Corollary~\ref{corr:1}. 
Let $C_1 = C_3 = C_0$ and
$C_2 = \frac{1}{4} \sqrt{\eta \lambda}$ in
Theorem~\ref{thm:meta}.
The goal is to show that
$
\left\|
\begin{bmatrix}
\xi_{t} \\
\xi_{t-1} 
\end{bmatrix}
\right\|
\leq \theta^{t} 2 C_0  \left\|
\begin{bmatrix}
\xi_{0} \\
\xi_{-1} 
\end{bmatrix}
\right\|
$ for all $t$ by induction. 
To achieve this, we will also use induction to show that for all iterations $s$,
\begin{equation} \label{induct:stc}
 \textstyle \| w_s - w_* \|  \textstyle  \leq  \textstyle  \textstyle R:= \frac{3}{64  \sqrt{\kappa} C_0 }.  
\end{equation}
A sufficient condition for the base case $s=0$ of (\ref{induct:stc}) to hold is
\begin{equation} \label{eq:h3}
\| \begin{bmatrix} w_0 - w_* \\ w_{-1} - w_* \end{bmatrix} \| \leq \frac{R}{2 C_0} = \frac{3}{128 \sqrt{\kappa} C_0^2},
\end{equation}
as $C_0 \geq 1$ by Theorem~\ref{thm:akv},
which in turn can be guaranteed if
$\| \begin{bmatrix} w_0 - w_* \\ w_{-1} - w_* \end{bmatrix} \|  \leq \frac{1}{ 683 \kappa^{3/2}}$ by using the upper bound $C_0 \leq 4 \sqrt{\kappa}$ of Corollary~\ref{corr:1}.

From Lemma~\ref{lem:SC-residual}, we have
\begin{equation} \label{eq:varphi-sc}
\begin{split}
\| \phi_s \| & \leq \eta \| \int_0^1 \nabla^2 f( (1-\tau) w_s + \tau w_* ) d\tau - \int_0^1 \nabla^2 f( (1-\tau) w_0 + \tau w_* ) d \tau \| \| \xi_s \|
\\ & \overset{(a)}{\leq } \eta  \alpha \left( \int_0^1 (1-\tau) \| w_s - w_0 \| d \tau \right)  \| \xi_s \|
\leq \eta \alpha \| w_s - w_0 \| \| \xi_s \|
\\ & \overset{(b)}{ \leq} 
\eta \alpha \left( \| w_s - w_* \| + \| w_0 - w_* \| \right) \| \xi_s \|, 
\end{split}
\end{equation}
where (a) is by $\alpha$-Lipschitzness of the Hessian and (b) is by the triangle inequality.
By \eqref{induct:stc}, \eqref{eq:varphi-sc}, Lemma~\ref{lem:SC-residual}, Theorem~\ref{thm:meta}, and Corollary~\ref{corr:1},
  it suffices to show that
given
$
\left\|
\begin{bmatrix}
\xi_{s} \\
\xi_{s-1} 
\end{bmatrix}
\right\|
\leq \theta^{s} 2 C_0  \left\|
\begin{bmatrix}
\xi_{0} \\
\xi_{-1} 
\end{bmatrix}
\right\|
$
and 
$\textstyle \| w_s - w_* \|  \textstyle  \leq  \textstyle  \textstyle R:= \frac{3}{64  \sqrt{\kappa} C_0 }$
 hold at $s=0,1,\dots,t-1$, one has
\begin{eqnarray}
\| \sum_{s=0}^{t-1} A^{t-s-1} \begin{bmatrix}
\varphi_s \\
0 
\end{bmatrix}
\|
& \leq &  \theta^{t}
C_0
\left\|
\begin{bmatrix}
\xi_{0} \\
\xi_{-1}  
\end{bmatrix}
\right\| \label{eq:Avarphi} \\
 \| w_t - w_* \| & \leq & R:= \frac{3}{64  \sqrt{\kappa} C_0 }, \label{eq:wR}
\end{eqnarray}
where $A:= \begin{bmatrix} 
(1 + \beta) I_n - \eta   \int_0^1 \nabla^2 f\big( (1-\tau) w_0 + w_* \big) d \tau & - \beta I_n \\
I_n & 0
\end{bmatrix}
$. 

We have
\begin{equation}
\begin{aligned}
\| \sum_{s=0}^{t-1} A^{t-s-1} \begin{bmatrix}
\varphi_s \\
0 
\end{bmatrix}
\|
& \leq  \sum_{s=0}^{t-1} \| A^{t-s-1} \begin{bmatrix}
\varphi_s \\
0 
\end{bmatrix}
\|
\\ & \overset{(a)}{\leq} \sum_{s=0}^{t-1} \beta_*^{t-s-1} C_0 \| \varphi_s \|
\\ & \overset{(b)}{\leq} 4 \eta \alpha R C_0^2 
 \sum_{s=0}^{t-1} \beta_*^{t-s-1} \theta^{s} 
 \|
\begin{bmatrix}
\xi_{0}  \\
\xi_{-1}  
\end{bmatrix}
\|
\\ &
\overset{(c)}{ \leq} R C_0^2  \frac{64}{3 \sqrt{\eta \lambda} } 
\theta^{t} \|
\begin{bmatrix}
\xi_{0}  \\
\xi_{-1}  
\end{bmatrix} 
\|
\\ &
\overset{(d)}{ \leq} C_0 \theta^{t} \|
\begin{bmatrix}
\xi_{0}  \\
\xi_{-1}  
\end{bmatrix} 
\|,
 \end{aligned}
\end{equation}
where (a) uses Theorem~\ref{thm:akv} with $\beta = \beta_*^2$,
(b) is by (\ref{eq:varphi-sc}), (\ref{induct:stc}), and the induction that $\| \xi_s \| \leq \theta^s 2 C_0 \| \begin{bmatrix}
\xi_{0}  \\
\xi_{-1}  
\end{bmatrix} 
\|
$,
(c) is because
 $\sum_{s=0}^{t-1} \beta_*^{t-1-s} \theta^s = \theta^{t-1} \sum_{s=0}^{t-1} \left( \frac{\beta_*}{\theta}  \right)^{t-1-s}$  $\leq \theta^{t-1} \sum_{s=0}^{t-1} \theta^{t-1-s}$ $\leq \theta^{t-1} \frac{4}{\sqrt{\eta\lambda}} \leq \theta^t \frac{16}{3 \sqrt{\eta\lambda} } $,
 and (d) is due to the definition of $R := \frac{3}{64 \sqrt{\kappa} C_0}$.
 Therefore, by Theorem~\ref{thm:meta}, we have
$ \left\|
\begin{bmatrix}
\xi_{t} \\
\xi_{t-1} 
\end{bmatrix}
\right\|
\leq \theta^{t} 2 C_0  \left\|
\begin{bmatrix}
\xi_{0} \\
\xi_{-1} 
\end{bmatrix}
\right\|.
$

Now let us switch to show (\ref{eq:wR}).
We have
\begin{equation}
\| \xi_t \| := \| w_t - w_* \| \overset{\text{induction}}{ \leq} \theta^t 2 C_0 \| \begin{bmatrix} w_0 - w_* \\ w_{-1} - w_* \end{bmatrix} \| \leq R,
\end{equation}
where the last inequality uses the constraint $\| \begin{bmatrix} w_0 - w_* \\ w_{-1} - w_* \end{bmatrix} \| \leq \frac{R}{2 C_0}$ by (\ref{eq:h3}).
\end{proof}
\clearpage

\section{Proof of Theorem~\ref{thm:acc} } \label{app:sec:relu}

We will need some supporting lemmas in the following for the proof.
In the following analysis,
we denote 
$
 C_0:=\frac{\sqrt{2} (\beta+1)}{
\sqrt{ \min\{  
h(\beta,\eta \lambda_{\min}(H)) , h(\beta,\eta \lambda_{\max}(H)) \} } }$,
where $h(\beta,\cdot)$ is defined in Theorem~\ref{thm:akv} and $H = H_0$ whose $(i,j)$ entry is
$(H_0)_{i,j}:= H(W_0)_{i,j} = \frac{1}{m} \sum_{r=1}^m x_i^\top x_j \mathbbm{1}\{ \langle w^{(r)}_0, x_i \rangle \geq 0 \text{ } \&  \text{ }   \langle w^{(r)}_0, x_j \rangle \geq 0 \}$, as defined in Lemma~\ref{lem:ReLU-residual}.
In the following, we also denote $\beta = (1- \frac{1}{2} \sqrt{\eta \lambda})^2 := \beta_*^2$. We summarize the notations in Table~\ref{table:1}.

\begin{table*}[h]
 \begin{tabular}{|c | c | c|} 
 \hline
 Notation & definition (or value)  & meaning  \\  
 \hline\hline
$ \N_{W}^{\text{ReLU}}(x)$ & $ \N_{W}^{\text{ReLU}}(x):= \frac{1}{\sqrt{m} } \sum_{r=1}^m a_r \sigma( \langle w^{(r)},  x \rangle )$ & the ReLU network's output given $x$ \\ \hline
$\bar{H}$ &
$\bar{H}_{i,j}  := \underset{ w^{(r)}}{\mathsf{E}}
[ x_i^\top x_j \mathbbm{1}\{ \langle w^{(r)}, x_i \rangle \geq 0 \text{ } \&  \text{ }   \langle w^{(r)}, x_j \rangle \geq 0 \}     ] .
$ & the expectation of the Gram matrix \\ \hline
$H_0$ & $H(W_0)_{i,j} = \frac{1}{m} \sum_{r=1}^m x_i^\top x_j \mathbbm{1}\{ \langle w^{(r)}_0, x_i \rangle \geq 0 \text{ } \&  \text{ }   \langle w^{(r)}_0, x_j \rangle \geq 0 \}$ & the Gram matrix at the initialization \\ \hline
$\lambda_{\min}(\bar{H})$ &  $\lambda_{\min}(\bar{H}) > 0 $ (by assumption)  & the least eigenvalue of $\bar{H}$.\\ \hline
$\lambda_{\max}(\bar{H})$ &       & the largest eigenvalue of $\bar{H}$\\ \hline
$\kappa$ &  $\kappa:= \lambda_{\max}(\bar{H}) / \lambda_{\min}(\bar{H})$      & the condition number of $\bar{H}$\\ \hline
$\lambda$ &  $\lambda := \frac{3}{4} \lambda_{\min}(\bar{H}) $     &
\shortstack{ (a lower bound of) \\ the least eigenvalue of $H_0$}.\\ \hline
$\lambda_{\max}$ &  $\lambda_{\max} := \lambda_{\max} (\bar{H}) + \frac{\lambda_{\min}(\bar{H})}{4}$     &
\shortstack{ (an upper bound of) \\ the largest eigenvalue of $H_0$}.\\ \hline
$\hat{\kappa}$ &  $\hat{\kappa}:=\frac{\lambda_{\max}}{\lambda} = \frac{4}{3} \kappa + \frac{1}{3} $      &
the condition number of $H_0$.\\ \hline
$\eta$  & $\eta = 1 / \lambda_{\max} $ & step size \\ \hline
$\beta$ & $\beta = (1- \frac{1}{2} \sqrt{\eta \lambda})^2 = (1 - \frac{1}{2 \sqrt{\hat{\kappa}}})^2 := \beta_*^2$ & momentum parameter \\ \hline
$\beta_*$ & $\beta_* = \sqrt{\beta} = 1- \frac{1}{2} \sqrt{\eta \lambda}$ & squared root of $\beta$ \\ \hline 
$\theta$ & $\theta = \beta_*  + \frac{1}{4} \sqrt{\eta \lambda} = 1 -\frac{1}{4} \sqrt{\eta \lambda} = 1 - \frac{1}{4 \sqrt{\hat{\kappa}}} $ & the convergence rate \\ \hline
$C_0$ &
$ C_0:=\frac{\sqrt{2} (\beta+1)}{
\sqrt{ \min\{  
h(\beta,\eta \lambda_{\min}(H_0)) , h(\beta,\eta \lambda_{\max}(H_0)) \} } }$
& the constant used in Theorem~\ref{thm:akv} 
\\ \hline \hline
\end{tabular}
\caption{Summary of the notations for proving Theorem~\ref{thm:acc}.} \label{table:1}
\end{table*}

\begin{lemma} \label{lem:ReLU-B} 
Suppose that 
the neurons $w^{(1)}_0, \dots, w^{(m)}_0$ are i.i.d. generated by $N(0,I_d)$ initially.
Then, for any set of weight vectors $W_t:=\{ w^{(1)}_t, \dots, w^{(m)}_t \}$ that satisfy for any $r\in [m]$, $\| w^{(r)}_t - w^{(r)}_0 \| \leq R^{\text{ReLU}} := \frac{\lambda}{1024 n C_0}$, 
it holds that
\[
\begin{aligned}
&
 \|  H_t - H_0 \|_F \leq 2 n R^{\text{ReLU}} = \frac{ \lambda}{512 C_0 }, 
\end{aligned}
\]
with probability at least $1 - n^2 \cdot \exp( -m R^{\text{ReLU}} / 10)$.
\end{lemma}

\begin{proof}
This is an application of Lemma~3.2 in \cite{ZY19}.
\end{proof}

Lemma~\ref{lem:ReLU-B} shows that if the distance between the current iterate $W_t$ and its initialization $W_0$ is small, then
the distance between the Gram matrix $H(W_t)$ and $H(W_0)$ should also be small. Lemma~\ref{lem:ReLU-B} allows us to obtain the following lemma, which bounds the size of $\varphi_t$ (defined in Lemma~\ref{lem:ReLU-residual}) in the residual dynamics.

\begin{lemma} \label{lem:ReLU-deviate2}
Following the setting as Theorem~\ref{thm:acc},
denote $\theta := \beta_* + \frac{1}{4} \sqrt{ \eta\lambda } = 1 - \frac{1}{4} \sqrt{ \eta\lambda }$.
Suppose that $\forall i \in [n], |S_i^\perp| \leq 4 m R^{\text{ReLU}}$ for some constant $R^{\text{ReLU}}:= \frac{\lambda}{1024n C_0} >0$.
If we have (I) for any $s \leq t$, the residual dynamics satisfies
$ \|
\begin{bmatrix}
\xi_{s} \\
\xi_{s-1} 
\end{bmatrix}
\| \leq 
\theta^{s} 
\cdot \nu C_0
\|
 \begin{bmatrix}
\xi_{0} \\
\xi_{-1} 
\end{bmatrix}
\|$, for some constant $\nu>0$,
and (II)
for any $r\in [m]$ and any $s\leq t$, $\| w^{(r)}_s - w^{(r)}_0 \| \leq R^{\text{ReLU}}$,
then $\phi_t$ and $\iota_t$ in Lemma~\ref{lem:ReLU-residual} satisfies
\[
\begin{aligned}
\| \phi_t \| & \leq 
\frac{ \sqrt{\eta \lambda} }{16} \theta^t
\nu 
 \| \begin{bmatrix} \xi_0 \\ \xi_{-1} \end{bmatrix} \|
, \text{ and } \| \iota_t \|  \leq \frac{\eta \lambda}{512} 
\theta^t \nu 
\| \begin{bmatrix} \xi_0 \\ \xi_{-1} \end{bmatrix} \|.
\end{aligned}
\]
Consequently, $\varphi_t$ in Lemma~\ref{lem:ReLU-residual} satisfies
\[
\| \varphi_t \| \leq \left( \frac{ \sqrt{\eta \lambda} }{16} 
 +  \frac{\eta \lambda }{512 } \right)
\theta^t \nu 
\| \begin{bmatrix} \xi_0 \\ \xi_{-1} \end{bmatrix} \|.
\]
\end{lemma}

\begin{proof}
Denote $\beta_* := 1 - \frac{1}{2} \sqrt{ \eta\lambda }$
and $\theta := \beta_* + \frac{1}{4} \sqrt{ \eta\lambda } = 1 - \frac{1}{4} \sqrt{ \eta\lambda }$.
We have by Lemma~\ref{lem:ReLU-residual}
\begin{equation} \label{eq:main4}
\begin{split}
 \| \phi_t \|
&  = \sqrt{ \sum_{i=1}^n \phi_t[i]^2 } \leq \sqrt{  \sum_{i=1}^n  \big( \frac{ 2 \eta \sqrt{n} |S_i^\perp|}{ m } \big( \| \xi_t \| +  \beta \sum_{\tau=0}^{t-1} \beta^{t-1-\tau}  \| \xi_{\tau}  \| \big) \big)^2 }
\\ & 
\overset{(a)}{\leq} 8 \eta n R^{\text{ReLU}} \big( \| \xi_t  \| + \beta \sum_{\tau=0}^{t-1} \beta^{t-1-\tau}  \| \xi_{\tau}  \| \big)
\\ & 
\overset{(b)}{\leq} 
8 \eta n R^{\text{ReLU}}
\left( \theta^{t} 
 \nu C_0 \| \begin{bmatrix} \xi_0 \\ \xi_{-1} \end{bmatrix} \| 
+ 
\beta \sum_{\tau=0}^{t-1} \beta^{t-1-\tau}  
\theta^{\tau} 
 \nu C_0 
 \| \begin{bmatrix} \xi_0 \\ \xi_{-1} \end{bmatrix} \| \right)
\\ & 
\overset{(c)}{=} 
8 \eta n R^{\text{ReLU}}
\left( 
\theta^{t} 
 \nu C_0 \| \begin{bmatrix} \xi_0 \\ \xi_{-1} \end{bmatrix} \|
 + \beta_*^{2} \nu C_0
  \sum_{\tau=0}^{t-1} \beta_*^{2(t-1-\tau)}
  \theta^{\tau} 
 \| \begin{bmatrix} \xi_0 \\ \xi_{-1} \end{bmatrix} \| \right)
\\ & 
\overset{(d)}{\leq}
8 \eta n R^{\text{ReLU}}
\left( 
\theta^{t} 
 \nu C_0 \| \begin{bmatrix} \xi_0 \\ \xi_{-1} \end{bmatrix} \|
 + \beta_*^{2} \nu C_0
   \theta^{t-1} 
  \sum_{\tau=0}^{t-1} 
  \theta^{t-1-\tau} 
 \| \begin{bmatrix} \xi_0 \\ \xi_{-1} \end{bmatrix} \| \right)
\\ & \leq
8 \eta n R^{\text{ReLU}} \theta^{t} ( 1 + \beta_* \sum_{\tau=0}^{t-1} \theta^{\tau} ) 
\nu C_0
\| \begin{bmatrix} \xi_0 \\ \xi_{-1} \end{bmatrix} \|
\\ & \leq
8 \eta n R^{\text{ReLU}} \theta^{t} ( 1 + \frac{\beta_*}{1-\theta} ) 
\nu C_0
\| \begin{bmatrix} \xi_0 \\ \xi_{-1} \end{bmatrix} \|
\\ & \overset{(e)}{\leq} 
\frac{\sqrt{\eta \lambda}}{16} \theta^t \nu 
 \| \begin{bmatrix} \xi_0 \\ \xi_{-1} \end{bmatrix} \|,
\end{split}
\end{equation}
where (a) is by
$|S_i^\perp| \leq 4 m R^{\text{ReLU}}$,
(b) is by induction that $\| \xi_t \| \leq \theta^t 
 \nu C_0 
 \| \begin{bmatrix} \xi_0 \\ \xi_{-1} \end{bmatrix} \|
$ as $u_0 = u_{-1}$, (c) uses that $\beta= \beta_*^2 $,
(d) uses $\beta_* \leq \theta$,
(e) uses $1+\frac{\beta_*}{1-\theta} \leq \frac{2}{1-\theta} \leq \frac{8}{\sqrt{\eta \lambda}}$ and $R^{\text{ReLU}}:= \frac{\lambda}{1024 n C_0}$.

Now let us switch to bound $\| \iota_t \|$.
\begin{equation} 
\begin{split}
\| \iota_t \| & \leq \eta \|  H_0 - H_t \|_2 \| \xi_t \|
\leq \frac{ \eta \lambda }{512 C_0}
 \theta^{t} \nu C_0 
  \| \begin{bmatrix} \xi_0 \\ \xi_{-1} \end{bmatrix} \|, 
\end{split}
\end{equation}
where we uses Lemma~\ref{lem:ReLU-B} that 
 $\|  H_0 - H_t \|_2 \leq \frac{ \lambda }{512 C_0}$
 and the induction that
$
\| \begin{bmatrix} \xi_t \\ \xi_{t-1} \end{bmatrix} \| \leq \theta^t  \nu C_0 
 \| \begin{bmatrix} \xi_0 \\ \xi_{-1} \end{bmatrix} \| $.

\end{proof}

The assumption of Lemma~\ref{lem:ReLU-deviate2}, $\forall i \in [n], |S_i^\perp| \leq 4 m R^{\text{ReLU}}$ only depends on the initialization. 
Lemma~\ref{lem:3.12} 
shows that it holds
with probability at least $1- n \cdot \exp( - m R^{\text{ReLU}})$.

\begin{lemma} \label{lem:ReLU-deviate1} 
Following the setting as Theorem~\ref{thm:acc},
denote $\theta := \beta_* + \frac{1}{4} \sqrt{ \eta\lambda } = 1 - \frac{1}{4} \sqrt{ \eta\lambda }$.
Suppose that the initial error satisfies $\| \xi_0 \|^2 = O(  n \log ( m / \delta) \log^2 (n / \delta) )$.
If for any $s < t$, the residual dynamics satisfies
$ \|
\begin{bmatrix}
\xi_{s} \\
\xi_{s-1} 
\end{bmatrix}
\| \leq \theta^{s} \cdot \nu C_0 
\|
 \begin{bmatrix}
\xi_{0} \\
\xi_{-1} 
\end{bmatrix}
\|$, for some constant $\nu>0$,
then \[ 
 \| w_{t}^{(r)} - w_0^{(r)} \|  \leq R^{\text{ReLU}} := \frac{\lambda}{1024n C_0} .\]
 \end{lemma}

\begin{proof}
We have 
\begin{equation}
\begin{split}
 \| w_{t+1}^{(r)} - w_0^{(r)} \| 
& \overset{(a)}{\leq} 
\eta \sum_{s=0}^t \|  M_s^{(r)} \|  
\overset{(b)}{=} 
\eta \sum_{s=0}^t
\| \sum_{\tau=0}^s \beta^{s-\tau}   \frac{ \partial L(W_{\tau})}{ \partial w_{\tau}^{(r)} }  \|
\leq 
\eta \sum_{s=0}^t \sum_{\tau=0}^s \beta^{s-\tau} 
\|   \frac{ \partial L(W_{\tau})}{ \partial w_{\tau}^{(r)} }  \|
\\ &
\overset{(c)}{\leq} \eta \sum_{s=0}^t \sum_{\tau=0}^s \beta^{s-\tau} 
\frac{\sqrt{n}}{\sqrt{m}} \| y - u_{\tau} \|
\\ &
\overset{(d)}{\leq} \eta \sum_{s=0}^t \sum_{\tau=0}^s \beta^{s-\tau} 
\frac{\sqrt{2n}}{\sqrt{m}}  \theta^{\tau} 
\nu C_0
 \| y - u_0 \|  
\\ &
\overset{(e)}{\leq} \frac{\eta \sqrt{2 n} }{ \sqrt{m} } 
\sum_{s=0}^t \frac{ \theta^s}{1 - \theta}
\nu C_0  \| y - u_0 \| 
\leq 
 \frac{\eta \sqrt{2 n}  }{ \sqrt{m} } \left(   
\frac{\nu C_0}{(1 - \theta)^2}  \right) \| y - u_0 \|
\\ &
\overset{(f)}{=} 
\frac{\eta \sqrt{2 n} }{ \sqrt{m} } 
\left(   \frac{16 \nu C_0}{\eta \lambda}  \right)
\| y - u_0 \|
\\ &
\overset{(g)}{=}
\frac{ \eta \sqrt{2n} }{ \sqrt{m} }
\left(   \frac{16 \nu C_0}{\eta \lambda}  \right)
O(\sqrt{ n \log ( m / \delta) \log^2 (n / \delta)} )
\\ & \textstyle
\overset{(h)}{\leq} \frac{\lambda}{1024 n C_0},
\end{split}
\end{equation}
where (a), (b) is by the update rule of momentum, which is
$ w_{t+1}^{(r)} - w_t^{(r)} = - \eta M_t^{(r)}$, where $M_t^{(r)}:= \sum_{s=0}^t \beta^{t-s} \frac{ \partial L(W_{s})}{ \partial w_{s}^{(r)} }$, (c) is because $\|\frac{ \partial L(W_{s})}{ \partial w_{s}^{(r)} }\| = \| \sum_{i=1}^n (y_i - u_s[i]) \frac{1}{\sqrt{m}} a_r x_i \cdot \mathbbm{1}\{ \langle w_s^{(r)}, x \rangle \geq 0 \} \| \leq \frac{1}{\sqrt{m}} \sum_{i=1}^n | y_i - u_s[i] | \leq \frac{\sqrt{n}}{\sqrt{m}} \| y - u_s \|$,
(d) is by
$ \|
\begin{bmatrix}
\xi_{s} \\
\xi_{s-1} 
\end{bmatrix}
\| \leq \theta^{s}  \nu C_0
\|
 \begin{bmatrix}
\xi_{0} \\
\xi_{-1} 
\end{bmatrix}
\|
$
(e) is because that $\beta=\beta_*^2 \leq \theta^2$,
(f) we use $\theta := (1 - \frac{1}{4} \sqrt{ \eta\lambda } )$, 
so that $\frac{1}{(1-\theta)^2} = \frac{16}{\eta \lambda}$,
(g) 
is by that the initial error satisfies $\| y - u_0 \|^2 = O(  n \log ( m / \delta) \log^2 (n / \delta) ),$
and (h) is by the choice of the number of neurons $m = \Omega( \lambda^{-4} n^{4 } C_0^4 \log^3 ( n / \delta)   ) = \Omega( \lambda^{-4} n^{4 } \kappa^2 \log^3 ( n / \delta)   )$, as $C_0 = \Theta( \sqrt{\kappa} )$ by Corollary~\ref{corr:1}.

The proof is completed.

\end{proof}

Lemma~\ref{lem:ReLU-deviate1} basically says that if the size of the residual errors is bounded and decays over iterations, then the distance between the current iterate $W_t$ and its initialization $W_0$ is well-controlled. The lemma will allows us to invoke Lemma~\ref{lem:ReLU-B} and Lemma~\ref{lem:ReLU-deviate2}
when proving Theorem~\ref{thm:acc}. 
The proof of Lemma~\ref{lem:ReLU-deviate1} is in Appendix~\ref{app:sec:relu}.
The assumption of Lemma~\ref{lem:ReLU-deviate1},
$\|  \xi_0 \|^2 = O(  n \log ( m / \delta)$  $\log^2 (n / \delta) )$,
  is satisfied by the random initialization with probability at least $1-\delta / 3$ according to   
Lemma~\ref{lem:3.10} .

\begin{lemma} (Claim 3.12 of \cite{ZY19}) \label{lem:3.12}
Fix a number $R_1 \in (0,1)$.
Recall that $S_i^\perp$ is a random set defined in Subsection~\ref{inst:ReLU}.
With probability at least $1 - n \cdot \exp(-mR_1)$, we have that
for all $i \in [n]$,
\[
\displaystyle
| S_i^\perp | \leq 4 m R_1.
\]
\end{lemma}
A similar lemma also appears in \citep{DZPS19}.
Lemma~\ref{lem:3.12} says that the number of neurons whose activation patterns for a sample $i$ could change during the execution is only a small faction of $m$ if $R_1$ is a small number,
i.e. $| S_i^\perp | \leq 4 m R_1 \ll m$. 

\begin{lemma} \label{lem:3.10} (Claim 3.10 in \cite{ZY19})
Assume that $w_0^{(r)} \sim N(0,I_d)$ and $a_r$ uniformly sampled from $\{-1,1\}$. For $0 < \delta < 1$, we have that 
\[
\| y - u_0 \|^2 = O( n \log ( m / \delta) \log^2 (n / \delta) ),
\]
with probability at least $1-\delta$.
\end{lemma}

\subsection{Proof of Theorem~\ref{thm:acc}}

\begin{proof} 
 (of Theorem~\ref{thm:acc})
Denote $\lambda:= \frac{3}{4} \lambda_{\min}(\bar{H})>0$.
Lemma~\ref{lem:ReLU-A} shows that $\lambda$
is a lower bound of $\lambda_{\min}(H)$ of the matrix $H$ defined in Lemma~\ref{lem:ReLU-residual}.
Also, denote $\beta_*:=1 - \frac{1}{2} \sqrt{\eta \lambda}$ (note that $\beta=\beta_*^2$)
and $\theta := \beta_* + \frac{1}{4} \sqrt{\eta \lambda} = 1 - \frac{1}{4} \sqrt{\eta \lambda} $. In the following,
we let $\nu = 2 $ in Lemma~\ref{lem:ReLU-deviate2},
~\ref{lem:ReLU-deviate1},
and let $C_1 = C_3 = C_0$ and
$C_2 = \frac{1}{4} \sqrt{\eta \lambda}$ in
Theorem~\ref{thm:meta}. The goal is to show that
$
\left\|
\begin{bmatrix}
\xi_{t} \\
\xi_{t-1} 
\end{bmatrix}
\right\|
\leq \theta^{t} 2 C_0  \left\|
\begin{bmatrix}
\xi_{0} \\
\xi_{-1} 
\end{bmatrix}
\right\|
$ for all $t$ by induction. 
To achieve this, we will also use induction to show that for all iterations $s$,
\begin{equation} \label{induct:relu}
 \textstyle \forall r \in [m], \| w^{(r)}_s - w^{(r)}_0 \|  \textstyle  \leq  \textstyle  \textstyle R^{\text{ReLU}} :=\frac{\lambda}{1024 n C_0},  
\end{equation}
which is clear true in the base case $s=0$.

By Lemma~\ref{lem:ReLU-residual},~\ref{lem:ReLU-A},~\ref{lem:ReLU-B},~\ref{lem:ReLU-deviate2}, Theorem~\ref{thm:meta}, and Corollary~\ref{corr:1},
 it suffices to show that
given
$
\left\|
\begin{bmatrix}
\xi_{s} \\
\xi_{s-1} 
\end{bmatrix}
\right\|
\leq \theta^{s} 2 C_0  \left\|
\begin{bmatrix}
\xi_{0} \\
\xi_{-1} 
\end{bmatrix}
\right\|
$ and (\ref{induct:relu}) hold at $s=0,1,\dots,t-1$, one has
\begin{eqnarray} 
\textstyle
\| \sum_{s=0}^{t-1} A^{t-s-1} \begin{bmatrix}
\varphi_s \\
0 
\end{bmatrix}
\|
& \textstyle \leq & \textstyle 
\theta^{t}
C_0
\left\|
\begin{bmatrix}
\xi_{0} \\
\xi_{-1} 
\end{bmatrix}
\right\|,
\label{eq:thm-ReLU-1}
\\ \textstyle \forall r \in [m], \| w^{(r)}_t - w^{(r)}_0 \|  &\textstyle  \leq  \textstyle & \textstyle R^{\text{ReLU}} :=\frac{\lambda}{1024 n C_0},  
\label{eq:thm-ReLU-2}
\end{eqnarray}
where the matrix $A$ and the vector $\varphi_t$ are defined in Lemma~\ref{lem:ReLU-residual}. The inequality (\ref{eq:thm-ReLU-1}) is the required  condition for using the result of Theorem~\ref{thm:meta}, while the inequality (\ref{eq:thm-ReLU-2}) helps us to show (\ref{eq:thm-ReLU-1}) through invoking Lemma~\ref{lem:ReLU-deviate2} to bound the terms $\{\varphi_s\}$ as shown in the following.

We have
\begin{equation} \label{eq:main5}
\begin{split}
& 
\| \sum_{s=0}^{t-1} A^{t-s-1} \begin{bmatrix}
\varphi_s \\
0 
\end{bmatrix}
\|
\overset{(a)}{\leq} 
 \sum_{s=0}^{t-1} \beta_*^{t-s-1} C_0 \| \varphi_s \| 
\\ &
\overset{(b)}{ \leq }
\left( \frac{ \sqrt{\eta \lambda}}{16}  
+ \frac{\eta \lambda}{512}  
\right)  2 C_0 \|
\begin{bmatrix}
\xi_{0} \\
\xi_{-1} 
\end{bmatrix}
\| \left(
 \sum_{s=0}^{t-1}
\beta_*^{t-1-s}  \theta^{s}  \right)
\\ & 
\overset{(c)}{ \leq }
\left( 
\frac{1 }{2} 
+ \frac{1}{64} \sqrt{\eta \lambda} 
\right) \theta^{t-1}
 C_0 \|
\begin{bmatrix}
\xi_{0} \\
\xi_{-1} 
\end{bmatrix}
\| 
\overset{(d)}{ \leq }  \theta^t C_0  \|
\begin{bmatrix}
\xi_{0} \\
\xi_{-1} 
\end{bmatrix}
\|, 
\end{split}
\end{equation}

where (a) uses Theorem~\ref{thm:akv}, (b) is due to Lemma~\ref{lem:ReLU-deviate2}, Lemma~\ref{lem:3.12},
(c) is because $\sum_{s=0}^{t-1} \beta_*^{t-1-s} \theta^s = \theta^{t-1} \sum_{s=0}^{t-1} \left( \frac{\beta_*}{\theta}  \right)^{t-1-s} \leq \theta^{t-1} \sum_{s=0}^{t-1} \theta^{t-1-s}$ $\leq \theta^{t-1} \frac{4}{\sqrt{\eta\lambda}}$, (d) uses that $\theta \geq \frac{3}{4}$ and $\eta \lambda \leq 1$. Hence, we have shown (\ref{eq:thm-ReLU-1}). 
 Therefore, by Theorem~\ref{thm:meta}, we have
$ \left\|
\begin{bmatrix}
\xi_{t} \\
\xi_{t-1} 
\end{bmatrix}
\right\|
\leq \theta^{t} 2 C_0  \left\|
\begin{bmatrix}
\xi_{0} \\
\xi_{-1} 
\end{bmatrix}
\right\|.
$

By Lemma~\ref{lem:ReLU-deviate1} and Lemma~\ref{lem:3.10}, 
we have (\ref{eq:thm-ReLU-2}). 
Furthermore, with the choice of $m$, we have $3 n^2 \exp( - m R^{\text{ReLU}} /10 ) \leq \delta$.  Thus, we have completed the proof.

\end{proof}

\clearpage

\section{Proof of Theorem~\ref{thm:LinearNet}} \label{app:sec:linear}

We will need some supporting lemmas in the following for the proof.
In the following analysis,
we denote 
$
 C_0:=\frac{\sqrt{2} (\beta+1)}{
\sqrt{ \min\{  
h(\beta,\eta \lambda_{\min}(H)) , h(\beta,\eta \lambda_{\max}(H)) \} } }$,
where $h(\beta,\cdot)$ is the constant defined in Theorem~\ref{thm:akv} and $\textstyle H = H_0 \textstyle := \frac{1}{ m^{L-1} d_y } \sum_{l=1}^L [ (\W{l-1:1}_0 X)^\top (\W{l-1:1}_0 X ) \otimes
  \W{L:l+1}_0 (\W{L:l+1}_0)^\top ]   \in \reals^{d_y n \times d_y n},
$ as defined in Lemma~\ref{lem:DL-residual}.
We also denote $\beta = (1- \frac{1}{2} \sqrt{\eta \lambda})^2 := \beta_*^2$.
As mentioned in the main text, following \citet{DH19,HXP20},
we will further assume that (A1) there exists a $W^*$ such that $Y = W^* X$, $X \in \reals^{d \times \bar{r}}$, and $\bar{r}=rank(X)$, which is actually without loss of generality (see e.g. the discussion in Appendix B of \citet{DH19}). We summarize the notions in Table~\ref{table:2}.

\begin{table*}[h]
\centering
 \begin{tabular}{|c | c | c|} 
 \hline
 Notation & definition (or value)  & meaning  \\  
 \hline\hline
$ \N_{W}^{L\text{-linear}}(x)$ & $
\N_W^{L\text{-linear}}(x) := \frac{1}{\sqrt{m^{L-1} d_{y}}} \W{L} \W{L-1} \cdots \W{1} x,$ &output of the deep linear network \\ \hline
$H_0$ & \shortstack{
$H_0 \textstyle := \frac{1}{ m^{L-1} d_y } \sum_{l=1}^L [ (\W{l-1:1}_0 X)^\top (\W{l-1:1}_0 X ) $ \\ \qquad \qquad $ \otimes
  \W{L:l+1}_0 (\W{L:l+1}_0)^\top ] \in \reals^{d_y n \times d_y n}$ } 
 & $H$ in (\ref{eq:meta}) is $H=H_0$ (Lemma~\ref{lem:DL-residual}) \\ \hline
$\lambda_{\max}(H_0)$ & $\lambda_{\max}(H_0)\leq L \sigma^2_{\max}(X) / d_y$  (Lemma~\ref{lem:DL-A})    & the largest eigenvalue of $H_0$\\ \hline
$\lambda_{\min}(H_0)$ & $\lambda_{\min}(H_0)\geq L \sigma^2_{\min}(X) / d_y$ (Lemma~\ref{lem:DL-A})     & the least eigenvalue of $H_0$\\ \hline
$\lambda$ & $\lambda:= L \sigma^2_{\min}(X) / d_y$     & \shortstack{ (a lower bound of) \\ the least eigenvalue of $H_0$}\\ \hline
$\kappa$ &  $\kappa:= \frac{\lambda_{1}(X^\top X) }{ \lambda_{\bar{r}}(X^\top X)} = \frac{\sigma^2_{\max}(X)}{ \sigma^2_{\min}(X) } $ (A1)      & the condition number of the data matrix $X$ \\ \hline
$\hat{\kappa}$ & $\hat{\kappa} := \frac{\lambda_{\max}(H_0)}{ \lambda_{\min}(H_0)} \leq \frac{\sigma^2_{\max}(X) }{ \sigma^2_{\min}(X)} = \kappa$ (Lemma~\ref{lem:DL-A})   & the condition number of $H_0$ \\ \hline
$\eta$  & $\eta = \frac{d_y}{L \sigma^2_{\max}(X)} $ & step size \\ \hline
$\beta$ & $\beta = (1- \frac{1}{2} \sqrt{\eta \lambda})^2 = (1 - \frac{1}{2 \sqrt{\kappa}})^2 := \beta_*^2$ & momentum parameter \\ \hline
$\beta_*$ & $\beta_* = \sqrt{\beta} = 1- \frac{1}{2} \sqrt{\eta \lambda}$ & squared root  of $\beta$ \\ \hline 
$\theta$ & $\theta = \beta_*  + \frac{1}{4} \sqrt{\eta \lambda} = 1 -\frac{1}{4} \sqrt{\eta \lambda} = 1 - \frac{1}{4 \sqrt{\kappa}} $ & the convergence rate \\ \hline
$C_0$ &
$ C_0:=\frac{\sqrt{2} (\beta+1)}{
\sqrt{ \min\{  
h(\beta,\eta \lambda_{\min}(H_0)) , h(\beta,\eta \lambda_{\max}(H_0)) \} } }$
& the constant used in Theorem~\ref{thm:akv} 
\\ \hline \hline
\end{tabular}
\caption{Summary of the notations for proving Theorem~\ref{thm:LinearNet}.
We will simply use $\kappa$ to represent the condition number of the matrix $H_0$ in the analysis since we have $\hat{\kappa} \leq \kappa$. 
} \label{table:2}
\end{table*}

\begin{lemma} \label{lem:DL-A}
 [Lemma 4.2 in \citep{HXP20}] 
 By the orthogonal initialization,
we have
\[
\begin{aligned}
& \lambda_{\min}(H_0)  \geq L \sigma^2_{\min}(X) / d_y, \quad
   \lambda_{\max}(H_0)  \leq L \sigma^2_{\max}(X) / d_y .
\\ &   \sigma_{\max}( \W{j:i}_0 ) = m^{ \frac{j-i+1}{2} }, \quad
\sigma_{\min}( \W{j:i}_0 ) = m^{ \frac{j-i+1}{2} }
\end{aligned}
\]
Furthermore, with probability $1-\delta$,
\[
\begin{aligned}
 \ell(W_0)  \leq B_0^2 = O\left( 1 + \frac{\log(\bar{r}/\delta)}{d_y} + \| W_* \|^2_2  \right),
\end{aligned}
\]
for some constant $B_0 > 0$.
\end{lemma}

We remark that Lemma~\ref{lem:DL-A} implies that the condition number of $H_0$ satisfies
\begin{equation}
\hat{\kappa} := \frac{\lambda_{\max}(H_0) }{\lambda_{\min}(H_0)} 
\leq \frac{\sigma^2_{\max}(X) }{ \sigma^2_{\min}(X)} = \kappa.
\end{equation}
\begin{lemma} \label{lem:linear-deviate1}
Following the setting as Theorem~\ref{thm:LinearNet},
denote $\theta := \beta_* + \frac{1}{4} \sqrt{ \eta\lambda } = 1 - \frac{1}{4} \sqrt{ \eta\lambda }$.
If we have (I) for any $s \leq t$, the residual dynamics satisfies
$ \|
\begin{bmatrix}
\xi_{s} \\
\xi_{s-1} 
\end{bmatrix}
\| \leq \theta^{s} 
\cdot \nu C_0
\|
 \begin{bmatrix}
\xi_{0} \\
\xi_{-1} 
\end{bmatrix}
\|, $ for some constant $\nu > 0$,
and (II) for all $l \in [L]$ and for any $s \leq t$,
$\| \W{l}_s - \W{l}_0 \|_F \leq R^{L\text{-linear}} := 
\frac{64 \| X \|_2 \sqrt{d_y}}{ L \sigma_{\min}^2(X) } \nu C_0 B_0$,
then 
\[ 
\| \phi_t \| 
\leq 
\frac{ 43  \sqrt{d_y} }{\sqrt{m} \| X \|_2}   
 \theta^{2t} \nu^2 C_0^2
\left( 
 \frac{ \| \xi_0 \|  }{1 -\theta}  \right)^2, \quad
 \| \psi_t \| 
\leq 
\frac{ 43  \sqrt{d_y} }{\sqrt{m} \| X \|_2}   
 \theta^{2(t-1)} \nu^2 C_0^2
\left( 
 \frac{ \| \xi_0 \|  }{1 -\theta}  \right)^2,
 \quad
 \| \iota_t \|  \leq  \frac{ \eta \lambda }{80} \theta^{t} \nu C_0 
  \| \begin{bmatrix} \xi_0 \\ \xi_{-1} \end{bmatrix} \| .
\]
Consequently,  $\varphi_t$ in Lemma~\ref{lem:DL-residual} satisfies
\[
\| \varphi_t \| \leq
\frac{ 1920 \sqrt{d_y} }{\sqrt{m} \| X \|_2} \frac{1}{\eta \lambda}   
 \theta^{2t}  \nu^2 C_0^2  \| \begin{bmatrix} \xi_0 \\ \xi_{-1} \end{bmatrix} \|^2
 +
\frac{\eta \lambda}{80} \theta^{t} \nu C_0 
  \| \begin{bmatrix} \xi_0 \\ \xi_{-1} \end{bmatrix} \| .
\]
\end{lemma}

\begin{proof}
By Lemma~\ref{lem:DL-residual},
$\varphi_t = \phi_t + \psi_t  + \iota_t \in \reals^{d_y n}$,
we have
\begin{equation}
\begin{aligned}
 \phi_t & := \frac{1}{\sqrt{m^{L-1} d_y}} \v( \Phi_t X)
\text{ , with } 
\Phi_t   := \Pi_l \left( \W{l}_t - \eta M_{t,l} \right)
- \W{L:1}_t +  \eta \sum_{l=1}^L \W{L:l+1}_t M_{t,l} \W{l-1:1}_t,
\end{aligned}
\end{equation}
and
\begin{equation}
\begin{aligned}
& \psi_t:= \frac{1}{\sqrt{m^{L-1} d_y}} 
\v\left(    (L-1) \beta \W{L:1}_{t}  X + \beta  \W{L:1}_{t-1} X
- \beta \sum_{l=1}^L \W{L:l+1}_t \W{l}_{t-1} \W{l-1:1}_{t} X \right).
\end{aligned}
\end{equation}
and
\begin{equation}
\begin{aligned}
& \iota_t:= \eta (H_0 - H_t) \xi_t.
\end{aligned}
\end{equation}

So if we can bound $\| \phi_t \|$, $\| \psi_t \|$, and $\| \iota_t\|$ respectively, then
we can bound $\| \varphi_t \| $
 by the triangle inequality.
\begin{equation} \label{eq:var}
\| \varphi_t \| \leq \| \phi_t \| + \|\psi_t \| + \| \iota_t \| .
\end{equation}

Let us first upper-bound $\| \phi_t \|$.
Note that $\Phi_t$ is the sum of all the high-order (of $\eta$'s) term in the product,
\begin{equation}
\W{L:1}_{t+1} = \Pi_l \left( \W{l}_t - \eta M_{t,l} \right)
= \W{L:1}_t - \eta \sum_{l=1}^L \W{L:l+1}_t M_{t,l} \W{l-1:1} + \Phi_t.
\end{equation}
By induction, we can bound the gradient norm of each layer as
\begin{equation} \label{eq:gnorm-linear}
\begin{split}
 \| \frac{ \partial \ell(\W{L:1}_s)}{ \partial \W{l}_s } \|_F
 &
\leq \frac{1}{\sqrt{m^{L-1} d_y} } \| \W{L:l+1}_s \|_2 \| U_s - Y \|_F \| \W{l-1:1}_s \|_2 \| X \|_2
\\ &
\leq \frac{1}{\sqrt{m^{L-1} d_y} } 1.1 m^{\frac{L-l}{2}}  
\theta^s \nu C_0 2 \sqrt{2} \| U_0 - Y \|_F
 1.1 m^{\frac{l-1}{2}} \| X \|_2
\\ &
\leq
\frac{4 \| X \|_2}{\sqrt{d_y}} \theta^s \nu C_0 \| U_0 - Y \|_F,
\end{split}
\end{equation}
where the second inequality we use Lemma~\ref{lem:linear-eigen} and that 
$\| \begin{bmatrix} \xi_s \\ \xi_{s-1} \end{bmatrix} \| \leq \theta^s \nu C_0 \| \begin{bmatrix} \xi_0 \\ \xi_{-1} \end{bmatrix} \|$ and $\| \xi_s \| = \| U_s - Y \|_F$.

So the momentum term of each layer can be bounded as
\begin{equation} \label{eq:tmpM2}
\begin{split}
\| M_{t,l} \|_F & = \| \sum_{s=0}^t \beta^{t-s}  \frac{ \partial \ell(\W{L:1}_s)}{ \partial \W{l}_s }  \|_F
\leq \sum_{s=0}^t \beta^{t-s} \| \frac{ \partial \ell(\W{L:1}_s)}{ \partial \W{l}_s }  \|_F
\\ & \leq 
\frac{4 \| X \|_2}{\sqrt{d_y}} 
\sum_{s=0}^t \beta^{t-s} \theta^s \nu C_0  \| U_0 - Y \|_F.
\\ & \leq 
\frac{4 \| X \|_2}{\sqrt{d_y}} 
\sum_{s=0}^t \theta^{2(t-s)} \theta^s \nu C_0  \| U_0 - Y \|_F.
\\ &  \leq
\frac{4 \| X \|_2}{\sqrt{d_y}}  \frac{ \theta^t }{1 -\theta} \nu C_0 \| U_0 - Y \|_F, 
\end{split}
\end{equation}
where in the second to last inequality we use $\beta= \beta_*^2 \leq \theta^2$.

Combining all the pieces together, we can bound 
$\| \frac{1}{\sqrt{m^{L-1} d_y} } \Phi_t X \|_F$ as
\begin{equation}  
\begin{split}
& \| \frac{1}{\sqrt{m^{L-1} d_y} } \Phi_t X \|_F
\\ &
\overset{(a)}{\leq} \frac{1}{ \sqrt{ m^{L-1} d_y}  } \sum_{j=2}^L {L \choose j} 
\left( \eta 
\frac{4 \| X \|_2}{\sqrt{d_y}}  \frac{ \theta^{t} }{1 -\theta}  \nu C_0    \| U_0 - Y \|_F  \right)^j (1.1)^{j+1}  m^{\frac{L-j}{2}} \| X \|_2
\\ &
\overset{(b)}{\leq} 1.1 \frac{1}{ \sqrt{ m^{L-1} d_y}  } \sum_{j=2}^L L^j 
\left( \eta
\frac{4.4 \| X \|_2}{\sqrt{d_y}}  \frac{ \theta^{t} }{1 -\theta} \nu C_0 \| U_0 - Y \|_F\right)^j  m^{\frac{L-j}{2}} \| X \|_2
\\ &
\leq 1.1  \sqrt{ \frac{ m}{ d_y}  } \| X \|_2 \sum_{j=2}^L  
\left( \eta
\frac{ 4.4  L \| X \|_2}{\sqrt{m d_y}}  \frac{ \theta^{t} }{1 -\theta} \nu C_0   \| U_0 - Y \|_F\right)^j,   
\end{split}
\end{equation}
where (a) uses (\ref{eq:tmpM2}) and Lemma~\ref{lem:linear-eigen}
for bounding a $j \geq 2$ higher-order terms like
$\frac{1}{ \sqrt{ m^{L-1} d_y} }\beta \W{L:k_j+1}_{t} \cdot (-\eta M_{t,k_j}) \W{k_j-1:k_{j-1}+1}_{t} 
\cdot (-\eta M_{t,k_{j-1}}) \cdots  (-\eta M_{t,k_{1}}) \cdot \W{k_1-1:1}_{t} $, where $1 \leq k_1 < \cdots < k_j \leq L$
 and (b) uses that ${L \choose j  } \leq \frac{L^j}{j!} $

To proceed, 
let us bound 
$\eta
\frac{ 4.4 L \| X \|_2}{\sqrt{m d_y}}  \frac{ \theta^{t} }{1 -\theta} \nu C_0  \| U_0 - Y \|_F$ in the sum above. We have
\begin{equation}
\begin{aligned}
\eta \frac{ 4.4  L \| X \|_2}{\sqrt{m d_y}}  \frac{ \theta^{t} }{1 -\theta} \nu C_0 \| U_0 - Y \|_F
 &
\leq 4.4
\sqrt{ \frac{  d_y}{ m }  } \frac{1}{ \| X \|_2 }  \frac{ \theta^{t} }{1 -\theta} \nu C_0  \| U_0 - Y \|_F
\\ & \leq 0.5,  
\end{aligned}
\end{equation}
where the last inequality uses that $\tilde{C}_1 \frac{d_y B_0^2 C_0^2  }{ \| X \|_2^2 } \frac{1}{ \left( 1 - \theta \right)^2}  \leq \tilde{C}_1
\frac{d_y B_0^2 C_0^2}{ \| X \|_2^2 } \frac{1}{ \eta \lambda}
\leq \tilde{C}_2
\frac{d_y B_0^2 \kappa^{2}  }{  \| X \|_2^2 } 
   \leq m $, for some sufficiently large constant $\tilde{C}_1, \tilde{C}_2 >0$.
Combining the above results, we have 
\begin{equation} \label{eq:phi}
\begin{split}
\| \phi_t \| &  = \| \frac{1}{ \sqrt{ m^{L-1} d_y}  } \Phi_t X \|_F
\\ & \leq 1.1  \sqrt{ \frac{ m}{ d_y}  } \| X \|_2 
\left( \eta
\frac{4.4  L \| X \|_2}{\sqrt{m d_y}}  \frac{ \theta^{t} }{1 -\theta} \nu C_0  \| U_0 - Y \|_F \right)^2  
\sum_{j=2}^{L-2} 
\left(  0.5 \right)^{j-2}   
\\ & \leq 2.2  \sqrt{ \frac{ m}{ d_y}  } \| X \|_2
\left( \eta
\frac{4.4  L \| X \|_2}{\sqrt{m d_y}}  \frac{ \theta^{t} }{1 -\theta} \nu C_0    \| U_0 - Y \|_F \right)^2   
\\ & \leq  \frac{ 43  \sqrt{d_y} }{ \sqrt{m} \| X \|_2}   
\left( 
 \frac{ \theta^{t} }{1 -\theta} \nu C_0  \| U_0 - Y \|_F \right)^2.   
\end{split}
\end{equation}

Now let us switch to upper-bound $\| \psi_t \|$.
It is equivalent to upper-bounding the Frobenius norm of
$\frac{1}{ \sqrt{ m^{L-1} d_y} }\beta (L-1)  \W{L:1}_{t} X + \frac{1}{ \sqrt{ m^{L-1} d_y} }\beta  \W{L:1}_{t-1} X
- \frac{1}{ \sqrt{ m^{L-1} d_y} } \beta \sum_{l=1}^L \W{L:l+1}_t \W{l}_{t-1} \W{l-1:1}_{t} X$,
which can be rewritten as
\begin{equation} \label{eq:import}
\begin{aligned}
& \underbrace{\frac{1}{ \sqrt{ m^{L-1} d_y} }   \beta (L-1)  \cdot \Pi_{l=1}^L \left( \W{l}_{t-1} - \eta M_{t-1,l} \right) X }_{\text{first term} } + \underbrace{ \frac{1}{ \sqrt{ m^{L-1} d_y} }\beta  \W{L:1}_{t-1} X }_{\text{second term}}
\\ & 
\underbrace{
- \frac{1}{ \sqrt{ m^{L-1} d_y} } \beta \sum_{l=1}^L \Pi_{i=l+1}^L \left( \W{i}_{t-1} - \eta M_{t-1,i} \right)  \W{l}_{t-1} \Pi_{j=1}^{l-1} \left( \W{j}_{t-1} - \eta M_{t-1,j} \right)  X }_{\text{third term}}.
\end{aligned}
\end{equation}
The above can be written as $B_0 + \eta B_1 + \eta^2 B_2 + \dots + \eta^L B_L$ for some matrices $B_0,\dots, B_L \in \reals^{d_y \times n}$.
Specifically, we have
\begin{equation}
\begin{split}
B_0 & = \underbrace{ \frac{1}{ \sqrt{ m^{L-1} d_y} }(L-1) \beta \W{L:1}_{t-1} X }_{ \text{due to the first term} }+ 
 \underbrace{  \frac{1}{ \sqrt{ m^{L-1} d_y} }\beta  \W{L:1}_{t-1} X }_{ \text{due to the second term} }
 \underbrace{
- \frac{1}{ \sqrt{ m^{L-1} d_y} }\beta L \W{L:1}_{t-1} X }_{ \text{due to the third term} } = 0
\\
B_1 & = \underbrace{ - \frac{1}{ \sqrt{ m^{L-1} d_y} }(L-1) \beta \sum_{l=1}^L \W{L:l+1}_{t-1}  M_{t-1,l} \W{l-1:1}_{t-1} }_{ \text{due to the first term} }
+ \underbrace{ \frac{1}{ \sqrt{ m^{L-1} d_y} }\beta \sum_{l=1}^L \sum_{k \neq l}
\W{L:k+1}_{t-1}  M_{t-1,k} \W{k-1:1}_{t-1} }_{ \text{due to the third term} }
= 0.
\end{split}
\end{equation}
So what remains on (\ref{eq:import}) are all the higher-order terms (in terms of the power of $\eta$), i.e. those with $\eta M_{t-1,i}$ and  $\eta M_{t-1,j}$, $\forall i \neq j$ or higher.

To continue, observe that for a fixed $(i,j)$, $i < j$, 
the second-order term that involves $\eta M_{t-1,i}$ and $\eta M_{t-1,j}$ on 
(\ref{eq:import}) is with coefficient
$\frac{1}{ \sqrt{ m^{L-1} d_y} }\beta$,
because the first term on (\ref{eq:import})
contributes to
$\frac{1}{ \sqrt{ m^{L-1} d_y} }(L-1) \beta $, while the third term on (\ref{eq:import})
contributes to
$- \frac{1}{ \sqrt{ m^{L-1} d_y} }(L-2) \beta$. 
Furthermore,
for a fixed $(i,j,k)$, $i < j < k$, 
the third-order term that involves $\eta M_{t-1,i}$, $\eta M_{t-1,j}$, and 
$\eta M_{t-1,k}$
on (\ref{eq:import}) is with coefficient $-2 \frac{1}{ \sqrt{ m^{L-1} d_y} }\beta$,
as the first term on (\ref{eq:import})
contributes to
$- \frac{1}{ \sqrt{ m^{L-1} d_y} }(L-1) \beta $, while the third term on (\ref{eq:import})
contributes to
$\frac{1}{ \sqrt{ m^{L-1} d_y} }(L-3) \beta$. 
Similarly, for a $p$-order term $\eta \underbrace{ M_{t-1,*}, \cdots, \eta M_{t-1,**} }_{ \text{p terms} }$,
the coefficient is 
$(p-1) \frac{1}{ \sqrt{ m^{L-1} d_y} }\beta (-1)^{p}$.

By induction (see (\ref{eq:tmpM2})), we can bound the norm of the momentum
at layer $l$ as
\begin{equation} \label{eq:tmpM}
\| M_{t-1,l} \|_F \leq \frac{4 \| X \|_2}{\sqrt{d_y}}  \frac{ \theta^{t-1} }{1 -\theta} \nu C_0 \| U_0 - Y \|_F .
\end{equation}
Combining all the pieces together, we have
\begin{equation} \label{eq:qqq1}
\begin{split}
& \| \frac{1}{ \sqrt{ m^{L-1} d_y} }\beta (L-1)  \W{L:1}_{t} X + \frac{1}{ \sqrt{ m^{L-1} d_y} }\beta  \W{L:1}_{t-1} X
- \frac{1}{ \sqrt{ m^{L-1} d_y} } \beta \sum_{l=1}^L \W{L:l+1}_t \W{l}_{t-1} \W{l-1:1}_{t} X \|_F
\\ &
\overset{(a)}{\leq} \frac{\beta}{ \sqrt{ m^{L-1} d_y}  } \sum_{j=2}^L \left(j-1\right) {L \choose j} 
\left( \eta 
\frac{4 \| X \|_2}{\sqrt{d_y}}  \frac{ \theta^{t-1} }{1 -\theta}  \nu C_0   \| U_0 - Y \|_F  \right)^j (1.1)^{j+1}  m^{\frac{L-j}{2}} \| X \|_2
\\ &
\overset{(b)}{\leq} 1.1 \frac{\beta}{ \sqrt{ m^{L-1} d_y}  } \sum_{j=2}^L L^j 
\left( \eta
\frac{4.4 \| X \|_2}{\sqrt{d_y}}  \frac{ \theta^{t-1} }{1 -\theta} \nu C_0    \| U_0 - Y \|_F\right)^j  m^{\frac{L-j}{2}} \| X \|_2
\\ &
\leq 1.1 \beta \sqrt{ \frac{ m}{ d_y}  } \| X \|_2 \sum_{j=2}^L  
\left( \eta
\frac{ 4.4  L \| X \|_2}{\sqrt{m d_y}}  \frac{ \theta^{t-1} }{1 -\theta} \nu C_0    \| U_0 - Y \|_F\right)^j,   
\end{split}
\end{equation}
where (a) uses (\ref{eq:tmpM}), the above analysis of the coefficients of the higher-order terms
 and Lemma~\ref{lem:linear-eigen}
for bounding a $j \geq 2$ higher-order terms like
$\frac{1}{ \sqrt{ m^{L-1} d_y} }\beta (j-1) (-1)^{j} \W{L:k_j+1}_{t-1} \cdot (-\eta M_{t-1,k_j}) \W{k_j-1:k_{j-1}+1}_{t-1} 
\cdot (-\eta M_{t-1,k_{j-1}}) \cdots  (-\eta M_{t-1,k_{1}}) \cdot \W{k_1-1:1}_{t-1} $, where $1 \leq k_1 < \cdots < k_j \leq L$
 and (b) uses that ${L \choose j  } \leq \frac{L^j}{j!} $

Let us bound 
$\eta
\frac{ 4.4  L \| X \|_2}{\sqrt{m d_y}}  \frac{ \theta^{t-1} }{1 -\theta} \nu C_0   \| U_0 - Y \|_F$ in the sum above. We have
\begin{equation} \label{eq:qqq2}
\begin{aligned}
\eta \frac{ 4.4 L \| X \|_2}{\sqrt{m d_y}}  \frac{ \theta^{t-1} }{1 -\theta} \nu C_0  \| U_0 - Y \|_F
 &
\leq 4.4
\sqrt{ \frac{  d_y}{ m }  } \frac{1}{ \| X \|_2 }  \frac{ \theta^{t-1} }{1 -\theta} \nu C_0  \| U_0 - Y \|_F
\\ & \leq 0.5,  
\end{aligned}
\end{equation}
where the last inequality uses that $\tilde{C}_1 \frac{d_y B_0^2 C_0^2  }{ \| X \|_2^2 } \frac{1}{ \left( 1 - \theta \right)^2}  \leq \tilde{C}_1
\frac{d_y B_0^2 C_0^2 }{ \| X \|_2^2 } \frac{1}{ \eta \lambda}
\leq \tilde{C}_2
\frac{d_y B_0^2 \kappa^{2}  }{  \| X \|_2^2 } 
   \leq m $, for some sufficiently large constant $\tilde{C}_1, \tilde{C}_2 >0$.
Combining the above results, i.e. (\ref{eq:qqq1}) and (\ref{eq:qqq2}), we have 
\begin{equation} \label{eq:psi}
\begin{split}
\| \psi_t \|& \leq \| \frac{1}{ \sqrt{ m^{L-1} d_y} }\beta (L-1)  \W{L:1}_{t} X + \frac{1}{ \sqrt{ m^{L-1} d_y} }\beta  \W{L:1}_{t-1} X
- \frac{1}{ \sqrt{ m^{L-1} d_y} } \beta \sum_{l=1}^L \W{L:l+1}_t \W{l}_{t-1} \W{l-1:1}_{t} X \|_F
\\ & \leq 1.1 \beta \sqrt{ \frac{ m}{ d_y}  } \| X \|_2 
\left( \eta
\frac{4.4  L \| X \|_2}{\sqrt{m d_y}}  \frac{ \theta^{t-1} }{1 -\theta} \nu C_0   \| U_0 - Y \|_F \right)^2  
\sum_{j=2}^{L-2} 
\left(  0.5 \right)^{j-2}   
\\ & \leq 2.2 \beta \sqrt{ \frac{ m}{ d_y}  } \| X \|_2
\left( \eta
\frac{4.4  L \| X \|_2}{\sqrt{m d_y}}  \frac{ \theta^{t-1} }{1 -\theta} \nu C_0   \| U_0 - Y \|_F \right)^2   
\\ & \leq  \frac{ 43  \sqrt{d_y} }{ \sqrt{m} \| X \|_2}   
\left( 
 \frac{ \theta^{t-1} }{1 -\theta} \nu C_0   \| U_0 - Y \|_F \right)^2,
\end{split}
\end{equation}
where the last inequality uses $\eta \leq \frac{d_y}{ L \| X \|_2^2}$.

Now let us switch to bound $\| \iota_t \|$.
We have
\begin{equation} 
\begin{aligned} \label{eq:jj1}
& \| \iota_t \| = \| \eta (H_t - H_0) \xi_t \|
\\ & =
\frac{\eta}{m^{L-1} d_y} 
\| \sum_{l=1}^L \W{L:l+1}_t (\W{L:l+1}_t)^\top ( U_t - Y) (\W{l-1:1}_t X)^\top \W{l-1:1}_t X
\\ & \qquad \qquad \qquad \qquad -
\sum_{l=1}^L \W{L:l+1}_0 (\W{L:l+1}_0)^\top ( U_t - Y) (\W{l-1:1}_0 X)^\top \W{l-1:1}_0 X\|_F
\\ & \leq 
\frac{\eta}{m^{L-1} d_y} 
\sum_{l=1}^L  
\| \W{L:l+1}_t (\W{L:l+1}_t)^\top ( U_t - Y) (\W{l-1:1}_t X)^\top \W{l-1:1}_t  X 
\\ & \qquad \qquad \qquad \qquad -
\W{L:l+1}_0  (\W{L:l+1}_0 )^\top ( U_t - Y) (\W{l-1:1}_0 X)^\top \W{l-1:1}_0 X\|_F
\\ & \leq 
\frac{\eta}{m^{L-1} d_y} 
\sum_{l=1}^L  \big( 
\underbrace{ 
\|\left( \W{L:l+1}_t (\W{L:l+1}_t)^\top - \W{L:l+1}_0 (\W{L:l+1}_0)^\top \right)  ( U_t - Y) (\W{l-1:1}_t X)^\top \W{l-1:1}_t  X \|_F }_{\text{ first term} }
\\ & \qquad \qquad \qquad \qquad + 
\underbrace{ 
\| \W{L:l+1}_0 (\W{L:l+1}_0)^\top  ( U_t - Y) \left( \W{l-1:1}_t X)^\top \W{l-1:1}_t X  - (\W{l-1:1}_0 X)^\top \W{l-1:1}_0 X   \right)  \|_F \big)
}_{\text{ second term} }.
\end{aligned}
\end{equation}

Now let us bound the first term. We have
\begin{equation} \label{eq:j0}
\begin{aligned}
&\underbrace{ \|\left( \W{L:l+1}_t (\W{L:l+1}_t)^\top - \W{L:l+1}_0 (\W{L:l+1}_0)^\top \right)  ( U_t - Y) (\W{l-1:1}_t X)^\top \W{l-1:1}_t X  \|_F }_{\text{ first term} }
\\ & \leq 
\| \W{L:l+1}_t (\W{L:l+1}_t)^\top - \W{L:l+1}_0 (\W{L:l+1}_0)^\top \|_2
\| U_t - Y \|_F \|  (\W{l-1:1}_t X)^\top \W{l-1:1}_t X  \|_2.
\end{aligned}
\end{equation}
For $\|  (\W{l-1:1}_t X)^\top \W{l-1:1}_t X  \|_2$, by using 
Lemma~\ref{lem:linear-deviate2} and Lemma~\ref{lem:linear-eigen},
we have
\begin{equation} \label{eq:j00}
\| (\W{l-1:1}_t X)^\top \W{l-1:1}_t X \|_2 
\leq \left( \sigma_{\max}( \W{l-1:1}_t X) \right)^2 \leq
\left( 1.1 m^{\frac{l-1}{2}} \sigma_{\max}(X) \right)^2.
\end{equation}

For $\| \W{L:l+1}_t (\W{L:l+1}_t)^\top - \W{L:l+1}_0 (\W{L:l+1}_0)^\top \|_2
$, denote $\W{L:l+1}_t = \W{L:l+1}_0 + \Delta^{(L:l+1)}_t$, we have
\begin{equation} \label{eq:j1}
\begin{split}
& \| \W{L:l+1}_t (\W{L:l+1}_t)^\top - \W{L:l+1}_0 (\W{L:l+1}_0)^\top \|_2
\\ & \leq \| \Delta_t^{(L:l+1)}  (\W{L:l+1}_t)^\top  +  \W{L:l+1}_t (\Delta_t^{(L:l+1)})^\top + \Delta_t^{(L:l+1)} ( \Delta_t^{(L:l+1)})^\top \|_2
\\ & \leq 2 \| \Delta_t^{(L:l+1)} \|_2 \cdot \sigma_{\max} (\W{L:l+1}_t)  +
\| \Delta_t^{(L:l+1)} \|_2^2 
\\ & \leq 2 \| \Delta_t^{(L:l+1)} \|_2 \cdot 
\left( 1.1 m^{\frac{L-l}{2}}\right)  +
\| \Delta_t^{(L:l+1)} \|_2^2.
\end{split}
\end{equation}

Therefore, we have to bound $\| \Delta_t^{(L:l+1)} \|_2$.
We have for any $1\leq i \leq j \leq L$.
\begin{equation} \label{eq:tw}
\W{j:i}_t = \left( \W{j}_0  + \Delta_j \right)
 \cdots \left( \W{i}_0 + \Delta_i  \right),
\end{equation}
where $\| \Delta_i \|_2 \leq \|\W{i}_t - \W{i}_0 \|_F \leq D:
= \frac{64 \| X \|_2 \sqrt{d_y}}{ L \sigma_{\min}^2(X) } \nu  C_0 B_0$
by Lemma~\ref{lem:linear-deviate2}.
The product (\ref{eq:tw}) above minus $\W{j:i}_0$ can be written as a finite sum of some terms of the form
\begin{equation}
\W{j:k_l+1}_0 \Delta_{k_l} \W{k_l-1: k_{l-1}+1}_0 \Delta_{k_{l-1}} \cdots \Delta_{k_1} \W{k_1-1:i}_0,
\end{equation}
where $i \leq k_1 < \cdots < k_l \leq j$. Recall that 
$\| \W{j':i'}_0 \|_2 = m^{\frac{j'-i'+1}{2} }$ by Lemma~\ref{lem:DL-A}.
Thus, we can bound 
\begin{equation} \label{eq:delta}
\begin{split}
\| \Delta_t^{(j:i)} \|_2 \leq
\| \W{j:i}_t - \W{j:i}_0 \|_F & \leq 
\sum_{l=1}^{j-i+1} { j-i+1 \choose l } (D)^l m^{\frac{j-i+1-l}{2}}
= (\sqrt{m} + D)^{j-i+1} - (\sqrt{m})^{j-i+1}
\\ & = (\sqrt{m})^{j-i+1} \left(  (1+D/\sqrt{m})^{j-i+1} -1  \right)
\leq (\sqrt{m})^{j-i+1} \left(  (1+D/\sqrt{m})^{L} -1  \right)
\\ & 
\overset{(a)}{=} \left(  (1+\frac{1}{\sqrt{C'}L\kappa})^{L} -1  \right)( \sqrt{m} )^{j-i+1} \overset{(b)}{\leq} \left(\exp\left(\frac{1}{\sqrt{C'}\kappa}\right)-1\right)( \sqrt{m} )^{j-i+1} 
\\ & \overset{(c)}{\leq} \left(1 + (e-1)\frac{1}{\sqrt{C'}\kappa} - 1\right)( \sqrt{m} )^{j-i+1}  \overset{(d)}{\leq} 
 \frac{1}{480 \kappa } ( \sqrt{m} )^{j-i+1} ,
\end{split}
\end{equation}
where (a) 
uses $\frac{D}{\sqrt{m}} \leq \frac{1}{\sqrt{C'}L\kappa}$, 
for some constant $C'>0$, since $C' \frac{d_y C_0^2 B_0^2 \kappa^4}{ \| X \|^2_2  } \leq C \frac{d_y  B_0^2 \kappa^5}{ \| X \|^2_2  } \leq m$, (b) follows by the inequality $(1+x/n)^n\leq e^x, \forall x \geq 0, n >0$, (c) from Bernoulli's inequality $e^r\leq 1 + (e-1)r,\forall 0 \leq r\leq 1$, and (d) by choosing any sufficiently larger $C'$. 

From (\ref{eq:delta}), we have 
$\| \Delta_t^{(L:l+1)} \|_2 \leq  \frac{1}{480 \kappa } ( \sqrt{m} )^{L-l}.$
Combining this with (\ref{eq:j0}), (\ref{eq:j00}), and (\ref{eq:j1}), we have
\begin{equation} \label{eq:jj2}
\begin{split}
&\underbrace{ \|\left( \W{L:l+1}_t (\W{L:l+1}_t)^\top - \W{L:l+1}_0 (\W{L:l+1}_0)^\top \right)  ( U_t - Y) (\W{l-1:1}_t X)^\top \W{l-1:1}_t X  \|_F }_{\text{ first term} }
\\ & 
\leq \left( 2 \| \Delta_t^{(L:l+1)} \|_2 \cdot 
\left( 1.1 m^{\frac{L-l}{2}}\right)  +
\| \Delta_t^{(L:l+1)} \|_2^2 \right)
\left( 1.1 m^{\frac{l-1}{2}} \sigma_{\max}(X) \right)^2 \| U_t - Y \|_F
\\ & 
\leq \left( 2 \frac{1}{480 \kappa } ( \sqrt{m} )^{L-l} \cdot 
\left( 1.1 m^{\frac{L-l}{2}} \right)  +
\big( \frac{1}{480 \kappa } ( \sqrt{m} )^{L-l} \big)^2 \right)
\left( 1.1 m^{\frac{l-1}{2}} \sigma_{\max}(X) \right)^2 \| U_t - Y \|_F
\\ & 
\leq \frac{\sigma_{\min}^2(X)}{160} m^{L-1}  \| U_t - Y \|_F,
\end{split}
\end{equation}
where in the last inequality we use $\kappa := \frac{\sigma_{\max}^2(X)}{\sigma_{\min}^2(X) }$.

Now let us switch to bound the second term, we have
\begin{equation} \label{eq:j33}
\begin{split}
& \underbrace{ 
\| (\W{L:l+1}_0 (\W{L:l+1}_0)^\top  ( U_t - Y) \left( \W{l-1:1}_t X)^\top \W{l-1:1}_t X  - (\W{l-1:1}_0 X)^\top \W{l-1:1}_0 X   \right)  \|_F \big)
}_{\text{ second term} }
\\ & 
\leq 
\| (\W{L:l+1}_0 (\W{L:l+1}_0)^\top \|_2 \| U_t - Y \|_F 
\| (\W{l-1:1}_t X)^\top \W{l-1:1}_t X  - (\W{l-1:1}_0 X)^\top \W{l-1:1}_0 X   \|_2.
\end{split}
\end{equation}
For $\| \W{L:l+1}_0 (\W{L:l+1}_0)^\top \|_2 $, based on Lemma~\ref{lem:DL-A}, we have
\begin{equation} \label{eq:j30}
\| \W{L:l+1}_0 (\W{L:l+1}_0)^\top \|_2 \leq m^{L-l}.
\end{equation}
To bound 
$\| (\W{l-1:1}_t X)^\top \W{l-1:1}_t X  - (\W{l-1:1}_0 X)^\top \W{l-1:1}_0 X   \|_2 $, we proceed as follows.
Denote $\W{l-1:1}_t = \W{l-1:1}_0 + \Delta^{(l-1:1)}_t$, we have
\begin{equation} \label{eq:j3}
\begin{split}
& \| (\W{l-1:1}_t X)^\top \W{l-1:1}_t X  - (\W{l-1:1}_0 X)^\top \W{l-1:1}_0 X   \|_2 
\\ & \leq 2 \| (\Delta^{(l-1:1)}_t X)^\top \W{l-1:1}_t X \|_2 + \| \Delta^{(l-1:1)}_t X \|^2_2
\\ &
\leq \left( 2 \ \| \Delta^{(l-1:1)}_t \| \| \W{l-1:1}_t \|_2
+  \| \Delta^{(l-1:1)}_t \|^2_2
\right) \| X \|^2_2
\\ &
\leq \left(2  \frac{1}{480 \kappa} m^{\frac{l-1}{2}}  1.1 m^{\frac{l-1}{2}}
+  \left( \frac{1}{480 \kappa} m^{\frac{l-1}{2}} \right)^2
\right) \| X \|^2_2
\\ &
\leq \frac{\sigma_{\min}^2(X)}{160} m^{l-1}  ,
\end{split}
\end{equation}
where the second to last inequality uses (\ref{eq:delta}),
Lemma~\ref{lem:linear-deviate2}, and
Lemma~\ref{lem:linear-eigen}, while the last inequality uses
$\kappa := \frac{\sigma_{\max}^2(X)}{\sigma_{\min}^2(X) }$.
Combining (\ref{eq:j33}), (\ref{eq:j30}), (\ref{eq:j3}), we have
\begin{equation} \label{eq:jj3}
\begin{split}
& \underbrace{ 
\| (\W{L:l+1}_0 (\W{L:l+1}_0)^\top  ( U_t - Y) \left( \W{l-1:1}_t X)^\top \W{l-1:1}_t X  - (\W{l-1:1}_0 X)^\top \W{l-1:1}_0 X   \right)  \|_F \big)
}_{\text{ second term} }
\\ & \leq \frac{\sigma_{\min}^2(X)}{160} m^{L-1} \| U_t - Y \|_F .
\end{split}
\end{equation}

Now combing (\ref{eq:jj1}), (\ref{eq:jj2}), and (\ref{eq:jj3}), we have
\begin{equation} \label{eq:iota}
\begin{split}
\| \iota_t \| \leq \frac{\eta}{m^{L-1} d_y} L \frac{\sigma_{\min}^2(X)}{80} m^{L-1} \| U_t - Y \|_F = \frac{\eta \lambda}{80} \| \xi_t \|,
\end{split}
\end{equation} 
where we use $\lambda:= \frac{L \sigma^2_{\min}(X)}{ d_y }$.

Now we have (\ref{eq:phi}), (\ref{eq:psi}), and  (\ref{eq:iota}),  
which leads to
\begin{equation}
\begin{split}
\| \varphi_t \| & \leq 
\| \phi_t \| + \|\psi_t \| + \| \iota_t \| 
\\ & \leq 
\frac{ 43  \sqrt{d_y} }{\sqrt{m} \| X \|_2}   
( \theta^{2t} +  \theta^{2(t-1)}) \nu^2 C_0^2
\left( 
 \frac{ \| \xi_0 \|  }{1 -\theta}  \right)^2
 + \frac{\eta \lambda}{80} \nu C_0  \| \begin{bmatrix} \xi_0 \\ \xi_{-1} \end{bmatrix} \| .
\\ & \leq
\frac{ 1920 \sqrt{d_y} }{\sqrt{m} \| X \|_2} \frac{1}{\eta \lambda}   
 \theta^{2t} \nu^2 C_0^2  \| \begin{bmatrix} \xi_0 \\ \xi_{-1} \end{bmatrix}\|^2
 + \frac{\eta \lambda}{80} \nu C_0  \| \begin{bmatrix} \xi_0 \\ \xi_{-1} \end{bmatrix} \| .
\end{split}
\end{equation}
where the last inequality uses that $1 \leq \frac{16}{9} \theta^2$ as $\eta \lambda \leq 1$ so that $\theta \geq \frac{3}{4}$.
\end{proof}

\begin{lemma}~\label{lem:linear-deviate2}
Following the setting as Theorem~\ref{thm:LinearNet},
denote $\theta := \beta_* + \frac{1}{4} \sqrt{ \eta\lambda } = 1 - \frac{1}{4} \sqrt{ \eta\lambda }$.
If for any $s \leq t$, the residual dynamics satisfies
$\textstyle \|
\begin{bmatrix}
\xi_{s} \\
\xi_{s-1} 
\end{bmatrix}
\| \leq \theta^{s} 
\cdot \nu C_0
\|
 \begin{bmatrix}
\xi_{0} \\
\xi_{-1} 
\end{bmatrix}
\|,
$ 
for some constant $\nu > 0$,
then 
\[
\| \W{l}_t - \W{l}_0 \|_F \leq R^{L\text{-linear}} := 
\frac{64 \| X \|_2 \sqrt{d_y}}{ L \sigma_{\min}^2(X) } \nu C_0 B_0.
\]
\end{lemma}

\begin{proof}
We have 
\begin{equation}
\begin{split}
 \| \W{l}_{t+1} - \W{l}_0 \|_F 
& \overset{(a)}{\leq} 
\eta \sum_{s=0}^t \|  M_{s,l} \|_F  
\overset{(b)}{=} 
\eta \sum_{s=0}^t
\| \sum_{\tau=0}^s \beta^{s-\tau}   \frac{ \partial \ell(\W{L:1}_{\tau})}{ \partial \W{L}_{\tau} }  \|_F
\leq 
\eta \sum_{s=0}^t \sum_{\tau=0}^s \beta^{s-\tau} 
\| \frac{ \partial \ell(\W{L:1}_{\tau})}{ \partial \W{L}_{\tau} }   \|_F
\\ &
\overset{(c)}{\leq} \eta \sum_{s=0}^t \sum_{\tau=0}^s \beta_*^{2(s-\tau)} 
\frac{4 \| X \|_2}{\sqrt{d_y}} \theta^{\tau} \nu C_0 \| U_0 - Y \|_F.
\\ &
\overset{(d)}{\leq} \eta \sum_{s=0}^t \frac{ \theta^{s} }{ 1 - \theta}
\frac{4 \| X \|_2}{\sqrt{d_y}} \nu C_0 \| U_0 - Y \|_F.
\\ &
\leq 
\frac{4 \eta \| X \|_2}{\sqrt{d_y}} \frac{1}{(1 - \theta) (1 - \theta)} \nu C_0 \| U_0 - Y \|_F
\\ &
\overset{(e)}{\leq} \frac{64 \| X \|_2}{ \lambda \sqrt{d_y}} \nu C_0 \| U_0 - Y \|_F
\\ &
\overset{(f)}{\leq} \frac{64 \| X \|_2 \sqrt{d_y}}{ L \sigma_{\min}^2(X) } \nu C_0 B_0,
\end{split}
\end{equation}
where (a), (b) is by the update rule of momentum, which is
$ \W{l}_{t+1} - \W{l}_t = - \eta M_{t,l}$, where $M_{t,l}:= \sum_{s=0}^t \beta^{t-s} \frac{ \partial \ell(W_{L:1})}{ \partial \W{l}_s }$, (c) is because $\|\frac{ \partial \ell(W_{L:1})}{ \partial \W{l}_s }\|_F = 
\frac{4 \| X \|_2}{\sqrt{d_y}} \theta^{s} \nu C_0 \| U_0 - Y \|_F$ (see (\ref{eq:gnorm-linear})),
(d) is because that $\beta= \beta_*^2 \leq \theta^2$,
(e) is because that $\frac{1}{(1-\theta)^2} = \frac{16}{\eta \lambda}$,
and (f) uses the upper-bound $B_0 \geq \| U_0 - Y \|$ defined in Lemma~\ref{lem:DL-A} and $\lambda:= \frac{L \sigma_{\min}^2(X)}{d_y}$.
The proof is completed.

\end{proof}

\begin{lemma}  \citep{HXP20} \label{lem:linear-eigen}
Let $R^{L\text{-linear}}$ be an upper bound that satisfies $\| \W{l}_t - \W{l}_t \|_F \leq R^{L\text{-linear}}$ for all $l$ and $t$.
Suppose the width $m$ satisfies $m > C (LR^{L\text{-linear}})^2$, where $C$ is any sufficiently large constant.
Then, 
\[
\begin{aligned}
\textstyle \sigma_{\max}( \W{j:i}_t ) \leq 1.1 m^{ \frac{j-i+1}{2} },  &\textstyle \text{ } 
\sigma_{\min}( \W{j:i}_t ) \geq 0.9 m^{ \frac{j-i+1}{2} }.
\end{aligned}
\]
\end{lemma}
\begin{proof}
The lemma has been proved in proof of Claim 4.4 and Claim 4.5 in \cite{HXP20}. For completeness, let us replicate the proof here.

We have for any $1\leq i \leq j \leq L$.
\begin{equation}
\W{j:i}_t = \left( \W{j}_0  + \Delta_j \right)
 \cdots \left( \W{i}_0 + \Delta_i  \right),
\end{equation}
where $\Delta_i = \W{i}_t - \W{i}_0$.
The product above minus $\W{j:i}_0$ can be written as a finite sum of some terms of the form
\begin{equation}
\W{j:k_l+1}_0 \Delta_{k_l} \W{k_l-1: k_{l-1}+1}_0 \Delta_{k_{l-1}} \cdots \Delta_{k_1} \W{k_1-1:i}_0,
\end{equation}
where $i \leq k_1 < \cdots < k_l \leq j$. Recall that 
$\| \W{j':i'}_0 \|_2 = m^{\frac{j'-i'+1}{2} }$.
Thus, we can bound 
\begin{equation} \label{eq:wdist}
\begin{aligned}
\| \W{j:i}_t - \W{j:i}_0 \|_F & \leq 
\sum_{l=1}^{j-i+1} { j-i+1 \choose l } (R^{L\text{-linear}})^l m^{\frac{j-i+1-l}{2}}
= (\sqrt{m} + R^{L\text{-linear}})^{j-i+1} - (\sqrt{m})^{j-i+1}
\\ & = (\sqrt{m})^{j-i+1} \left(  (1+R^{L\text{-linear}}/\sqrt{m})^{j-i+1} -1  \right)
\leq (\sqrt{m})^{j-i+1} \left(  (1+R^{L\text{-linear}}/\sqrt{m})^{L} -1  \right)
\\ & \leq 0.1 ( \sqrt{m} )^{j-i+1},
\end{aligned}
\end{equation}
where the last step uses $m > C (LR^{L\text{-linear}})^2$.
By combining this with Lemma~\ref{lem:DL-A}, one can obtain the result.

\end{proof}
\noindent
\textbf{Remark:}
In the proof of Lemma~\ref{lem:linear-deviate1}, we obtain a tighter bound of the distance $\| \W{j:i}_t - \W{j:i}_0 \|_F \leq O(\frac{1}{\kappa} ( \sqrt{m} )^{j-i+1} )$. 
However,
to get the upper-bound $\sigma_{\max}( \W{j:i}_t )$ shown in Lemma~\ref{lem:linear-eigen}, (\ref{eq:wdist}) is sufficient for the purpose.

\subsection{Proof of Theorem~\ref{thm:LinearNet}} \label{sec:linear}

\begin{proof} (of Theorem~\ref{thm:LinearNet})
Denote $\lambda:= L \sigma_{\min}^2(X) / d_y$.
By Lemma~\ref{lem:DL-A}, $\lambda_{\min}(H) \geq \lambda$.
Also, denote $\beta_*:=1 - \frac{1}{2} \sqrt{\eta \lambda}$
and $\theta := \beta_* + \frac{1}{4} \sqrt{\eta \lambda} = 1 - \frac{1}{4} \sqrt{\eta \lambda} $.
Let $\nu = 2 $ in Lemma~\ref{lem:linear-deviate1},
~\ref{lem:linear-deviate2},
and let $C_1 = C_3 = C_0$ and
$C_2 = \frac{1}{4} \sqrt{\eta \lambda}$ in
Theorem~\ref{thm:meta}.
The goal is to show that
$
\left\|
\begin{bmatrix}
\xi_{t} \\
\xi_{t-1} 
\end{bmatrix}
\right\|
\leq \theta^{t} 2 C_0  \left\|
\begin{bmatrix}
\xi_{0} \\
\xi_{-1} 
\end{bmatrix}
\right\|
$ for all $t$ by induction.
To achieve this, we will also use induction to show that for all iterations $s$,
\begin{equation} \label{induct:linear}
\forall l \in [L], \| \W{l}_t - \W{l}_0 \|
 \leq   R^{L\text{-linear}}:= \frac{64 \| X \|_2 \sqrt{d_y}}{ L \sigma_{\min}^2(X) }  C_0 B_0,  
\end{equation}
which is clearly true in the base case $s=0$.

By Lemma~\ref{lem:DL-residual},~\ref{lem:DL-A}, \ref{lem:linear-deviate1},~\ref{lem:linear-deviate2}, Theorem~\ref{thm:meta} and Corollary~\ref{corr:1}, it suffices to show that
$
\left\|
\begin{bmatrix}
\xi_{s} \\
\xi_{s-1} 
\end{bmatrix}
\right\|
\leq
\theta^{s}  \cdot 2C_0 \left\|
\begin{bmatrix}
\xi_{0} \\
\xi_{-1} 
\end{bmatrix}
\right\|
$
and 
$
\forall l \in [L], \| \W{l}_s - \W{l}_0 \|
\leq   R^{L\text{-linear}}$
 hold at $s=0,1,\dots,t-1$, one has 
\begin{eqnarray}
\| \sum_{s=0}^{t-1} A^{t-s-1} \begin{bmatrix}
\varphi_s \\
0 
\end{bmatrix}
\|
& \leq & 
\theta^{t}
 C_0
\left\|
\begin{bmatrix}
\xi_{0} \\
\xi_{-1} 
\end{bmatrix}
\right\|,
\label{eq:thm-DL-1}
\\ \forall l \in [L], \| \W{l}_t - \W{l}_0 \| & \leq & R^{L\text{-linear}}:= \frac{64 \| X \|_2 \sqrt{d_y}}{ L \sigma_{\min}^2(X) }  C_0 B_0,
 \label{eq:thm-DL-2}
\end{eqnarray}
where the matrix $A$ and the vector $\varphi_t$ are defined in Lemma~\ref{lem:DL-residual}, and $B_0$ is a constant such that $B_0 \geq \| Y - U_0 \|_F$ with probability $1-\delta$ by Lemma~\ref{lem:DL-A}.
 The inequality (\ref{eq:thm-DL-1}) is the required  condition for using the result of Theorem~\ref{thm:meta}, while the inequality (\ref{eq:thm-DL-2}) helps us to show (\ref{eq:thm-DL-1}) through invoking Lemma~\ref{lem:linear-deviate1} to bound the terms $\{\varphi_s\}$ as shown in the following.

Let us show (\ref{eq:thm-DL-1}) first.
We have
\begin{equation} \label{eq:6-linear}
\begin{aligned}
 \| \sum_{s=0}^{t-1} A^{t-1-s} \| \begin{bmatrix}
\varphi_{s} \\ 0 \end{bmatrix} \|
& \overset{(a)}{ \leq} 
\sum_{s=0}^{t-1} \beta_*^{t-1-s} C_0 \| \varphi_s \|
\\ & \overset{(b)}{ \leq} 
\frac{ 1920 \sqrt{d_y} }{\sqrt{m} \| X \|_2} \frac{1}{\eta \lambda}   
\sum_{s=0}^{t-1} \beta_*^{t-1-s}
 \theta^{2s}  4 C_0^3  \| \begin{bmatrix} \xi_0 \\ \xi_{-1} \end{bmatrix} \|^2
 +  \sum_{s=0}^{t-1} \beta_*^{t-1-s}
\frac{\eta \lambda}{80}  \theta^{s} 2 C_0^2 
  \| \begin{bmatrix} \xi_0 \\ \xi_{-1} \end{bmatrix} \| 
\\ & \overset{(c)}{ \leq} 
\frac{ 1920 \sqrt{d_y} }{\sqrt{m} \| X \|_2} \frac{1}{\eta \lambda}   
\sum_{s=0}^{t-1} \beta_*^{t-1-s}
 \theta^{2s}  4 C_0^3  \| \begin{bmatrix} \xi_0 \\ \xi_{-1} \end{bmatrix} \|^2
+ \frac{2 \sqrt{\eta \lambda}}{15}  \theta^{t} C_0^2 
  \| \begin{bmatrix} \xi_0 \\ \xi_{-1} \end{bmatrix} \| 
 \\ & \overset{(d)}{ \leq} 
\frac{ 1920 \sqrt{d_y} }{\sqrt{m} \| X \|_2} \frac{16}{3(\eta \lambda)^{3/2}}   
\theta^{t} 4 C_0^3
\| \begin{bmatrix} \xi_0 \\ \xi_{-1} \end{bmatrix} \|^2 +  
\frac{2 \sqrt{\eta \lambda}}{15}  \theta^{t} C_0^2 
  \| \begin{bmatrix} \xi_0 \\ \xi_{-1} \end{bmatrix} \| 
\\ & \overset{(e)}{ \leq} 
\frac{1}{3} \theta^{t} C_0
\| \begin{bmatrix} \xi_0 \\ \xi_{-1} \end{bmatrix} \|
 + 
\frac{2 \sqrt{\eta \lambda}}{15}  \theta^{t} C_0^2 
  \| \begin{bmatrix} \xi_0 \\ \xi_{-1} \end{bmatrix} \| 
\\ & \overset{(f)}{ \leq}
\theta^{t}   C_0  \| \begin{bmatrix} \xi_0 \\ \xi_{-1} \end{bmatrix} \|,
\end{aligned}
\end{equation}
where (a) uses Theorem~\ref{thm:akv} with $\beta = \beta_*^2$,
(b)
is by Lemma~\ref{lem:linear-deviate1}, (c)
uses
$\sum_{s=0}^{t-1} \beta_*^{t-1-s} \theta^s = \theta^{t-1} \sum_{s=0}^{t-1} \left( \frac{\beta_*}{\theta}  \right)^{t-1-s} \leq \theta^{t-1} \sum_{s=0}^{t-1} \theta^{t-1-s}$ $\leq \theta^{t-1} \frac{4}{\sqrt{\eta\lambda}} \leq \theta^{t} \frac{16}{3\sqrt{\eta\lambda}} $, $\beta_* = 1 - \frac{1}{2} \sqrt{\eta \lambda} \geq \frac{1}{2}$
, and $\theta = 1 - \frac{1}{4} \sqrt{\eta \lambda} \geq \frac{3}{4}$,
(d) uses $\sum_{s=0}^{t-1} \beta_*^{t-1-s} \theta^{2s} \leq \sum_{s=0}^{t-1} \theta^{t-1+s} \leq \frac{\theta^{t-1}}{1-\theta}  \leq \theta^{t} \frac{16}{3\sqrt{\eta\lambda}} $,
(e) is because
$C' \frac{d_y  C_0^4 B_0^2 }{  \| X \|_2^2 } \frac{1}{ (\eta \lambda)^3 } 
\leq C \frac{d_y  B_0^2 }{ \| X \|_2^2 } \kappa^5  \leq m
$ for some sufficiently large constants $C', C > 0$, and (f) uses that $\eta \lambda = \frac{1}{\kappa}$ and $C_0 \leq 4 \sqrt{\kappa}$ by Corollary~\ref{corr:1}.
Hence, we have shown (\ref{eq:thm-DL-1}). 
 Therefore, by Theorem~\ref{thm:meta}, we have
$ \left\|
\begin{bmatrix}
\xi_{t} \\
\xi_{t-1} 
\end{bmatrix}
\right\|
\leq \theta^{t} 2 C_0  \left\|
\begin{bmatrix}
\xi_{0} \\
\xi_{-1} 
\end{bmatrix}
\right\|.
$

By Lemma~\ref{lem:linear-deviate2}, we have (\ref{eq:thm-DL-2}). 
Thus, we have completed the proof.

\end{proof}
\clearpage

\section{Experiment} \label{app:exp}

\subsection{ReLU network}
We report a proof-of-concept experiment for training the ReLU network. We sample $n=5$ points from the normal distribution, and then scale the size to the unit norm. We generate the labels uniformly random from $\{1,-1\}$. We let $m=1000$ and $d=10$. We compare vanilla GD and gradient descent with Polyak's momentum. 
Denote $\hat{\lambda}_{\max}:= \lambda_{\max}(H_0)$, $\hat{\lambda}_{\min} := \lambda_{\min}(H_0)$, and $\hat{\kappa}:= \hat{\lambda}_{\max} /\hat{\lambda}_{\min} $. Then, for gradient descent with Polyak's momentum, we set the step size $\eta =  1 / \left(  \hat{\lambda}_{\max} \right)$ and set the momentum parameter $\beta = (1 - \frac{1}{2} \frac{1}{\sqrt{ \hat{\kappa} }}) ^2$.  
For gradient descent, we set the same step size. The result is shown on Figure~\ref{fig:exp}.

We also report the percentiles of pattern changes over iterations. Specifically, we report the quantity
\[
\frac{ \sum_{i=1}^n \sum_{r=1}^m \mathbbm{1} \{ \text{sign}( x_i^\top w_t^{(r)} ) \neq  \text{sign}( x_i^\top w_0^{(r)} ) \} }{  m n  },
\]
as there are $m n$ patterns. For gradient descent with Polyak's momentum, the percentiles of pattern changes is approximately $0.76\%$; while for vanilla gradient descent, the percentiles of pattern changes is $0.55\%$. 

\subsection{Deep linear network}

We let the input and output dimension $d = d_y = 20$, the width of the intermediate layers $m=50$, the depth $L=100$. We sampled a $X \in \reals^{20 \times 5}$ from the normal distribution. We let  $W^* = I_{20} + 0.1\bar{W}$, where $\bar{W} \in \reals^{20 \times 20}$ is sampled from the normal distribution.
Then, we have $Y = W^* X$,
$\eta = \frac{d_y}{L \sigma_{\max}^2(X)}$ and $\beta = (1-\frac{1}{2} \sqrt{\eta \lambda})^2$, where $\lambda= \frac{L \sigma_{\min}^2(X)}{ d_y}$.
Vanilla GD also uses the same step size. 
The network is initialized by the orthogonal initialization and both algorithms start from the same initialization.
The result is shown on Figure~\ref{fig:exp2}. 

\begin{figure}[h]
  \centering
    \includegraphics[width=0.4\textwidth]{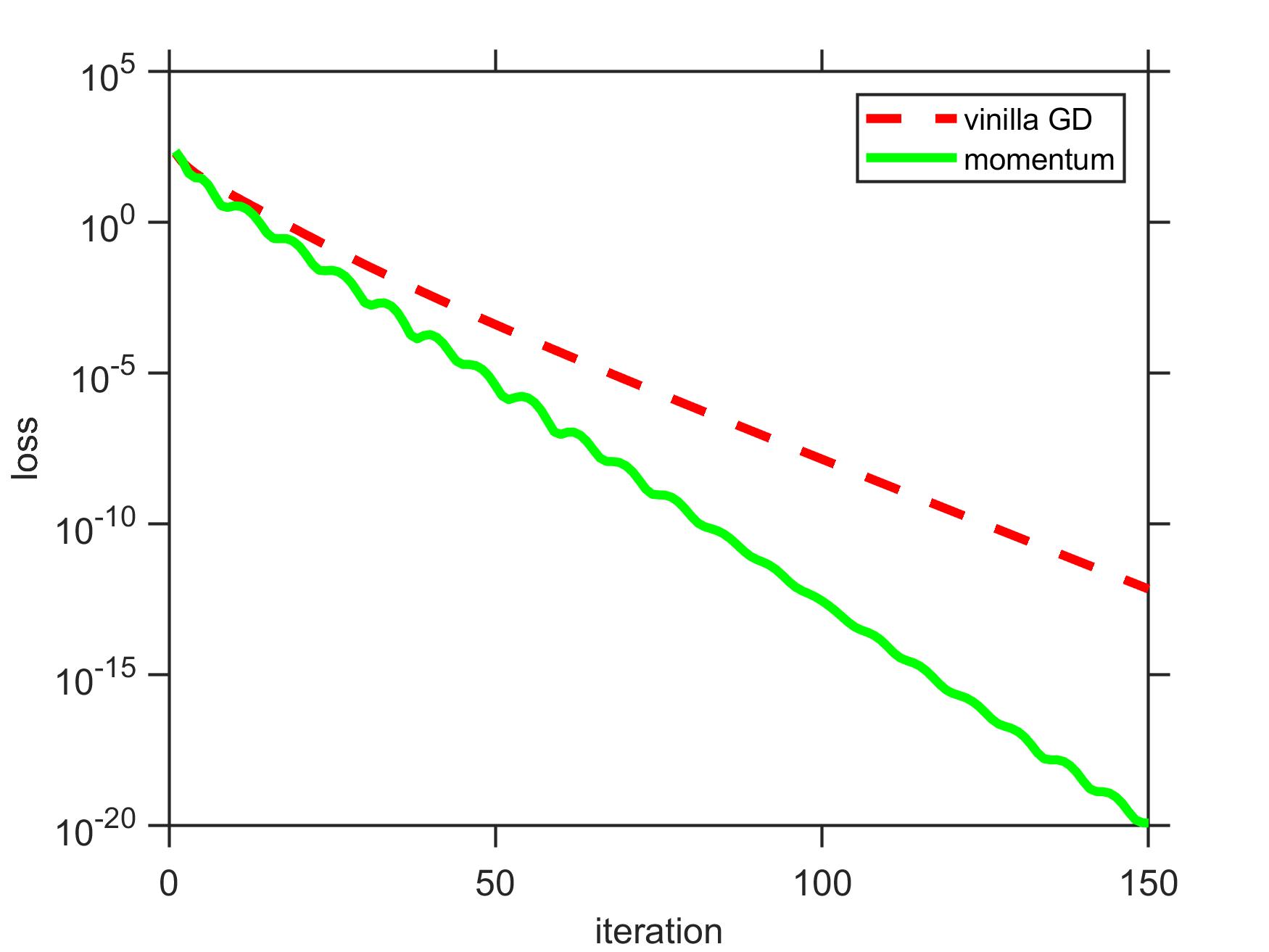}
    \caption{\footnotesize Training a $100$-layer deep linear network. Here ``momentum'' stands for gradient descent with Polyak's momentum. 
    } 
        \label{fig:exp2} 
\end{figure}

\end{document}